\documentclass[twoside,11pt]{article}

\usepackage{amsthm}

\usepackage[preprint]{jmlr2e}

\usepackage{mathtools}
\usepackage{bm}
\usepackage{bbm}
\usepackage{enumitem}
\usepackage{subfigure}
\usepackage{booktabs}
\usepackage{makecell}
\usepackage{multicol}
\usepackage{multirow}
\usepackage{subcaption}
\usepackage{tikz}

\newtheorem{thm}{Theorem}
\newtheorem{lem}{Lemma}

\newtheorem{cor}{Corollary}

\theoremstyle{definition}

\newtheorem{defn}{Definition}

\newtheorem{rem}{Remark}

\newcommand{\cD}{\mathcal{D}}
\newcommand{\cH}{\mathcal{H}}

\newcommand{\cT}{\mathcal{T}}

\newcommand{\cX}{\mathcal{X}}

\newcommand{\bb}{\bm{b}}

\newcommand{\bx}{x}

\newcommand{\bz}{\bm{z}}
\newcommand{\bw}{\bm{w}}

\newcommand{\bA}{\bm{A}}

\newcommand{\bX}{X}

\newcommand{\bZ}{\bm{Z}}

\newcommand{\Reals}{\mathbb{R}}
\newcommand{\defined}{\triangleq}

\newcommand{\dif}{\textrm{d}}

\newcommand{\EEE}[2]{\mathsf{E}_{#1}\big[ #2 \big]}

\newcommand{\Varr}[2]{\mathsf{Var}_{#1}\left(#2\right)}

\newcommand{\Var}[1]{\mathsf{Var}\left(#1\right)}
\newcommand{\Unc}[1]{\mathsf{Unc}\left(#1\right)}

\newcommand{\Sim}[1]{\mathsf{Sim}\left(#1\right)}
\newcommand{\Perf}[1]{\mathsf{Perf}_t\left(#1\right)}
\newcommand{\myPerf}[1]{\mathsf{Perf}_t \left(Y_{#1}(\bx^*),y_t^*\right)}  %

\usepackage{color}

\newcommand{\rev}{}  %

\newcommand{\rdim}{d}
\newcommand{\amk}{\mathsf{Dist}_{k}}
\newcommand{\amone}{\mathsf{Dist}_{1}}

\newcommand{\kendall}{Kendall's $\tau$ coefficient }

\usepackage{lastpage}

\firstpageno{1}

\begin{document}

\title{Training-Free Uncertainty Estimation for Embedding Models}

\author{\name Young-Jin Park\thanks{Equal contribution.} \email youngp@mit.edu \\
    \addr Massachusetts Institute of Technology
    \AND
    \name Addison Kristanto Julistiono$^*$ \email aaron25@mit.edu \\
    \addr Massachusetts Institute of Technology
    \AND
    \name Hao Wang \email hao-wang@redhat.com \\
    \addr MIT-IBM Lab \& Red Hat AI
    \AND
    \name Shervin Ardeshir \email shervin.ardeshir@gmail.com\!\!
    \AND
    \name Navid Azizan \email azizan@mit.edu \\
    \addr Massachusetts Institute of Technology
}

\editor{...}

\maketitle

\begin{abstract}
Embedding models, often obtained via self-supervised learning, extract general-purpose representations from data. Quantifying the reliability of these representations is crucial, as many downstream models rely on them as input for their own tasks.
To this end, we introduce a formal definition of \emph{representation reliability}: the representation for a given test point is considered to be reliable if the downstream models built on top of that representation can{, on average,} consistently generate accurate predictions for that test point {across various downstream tasks}. However, accessing the downstream data to quantify the representation reliability is often limited or restricted for various reasons.
We propose training-free methods for estimating the representation reliability without access to the downstream data.
Our method is based on the concept of \emph{neighborhood consistency} (NC) across distinct pre-trained representation spaces. The key insight is to find shared neighboring points as anchors to align these representation spaces before comparing them.
We {provide theoretical justifications for NC and} develop two practical approaches: (1) directly computing NC when multiple pre-trained models are available, and (2) a \emph{perturbation-based} NC (PNC), which creates synthetic ensembles from a single model through isotropic Gaussian noise, avoiding the computational cost of training deep ensembles.
We further propose PNC-spread tuning, which systematically determines the perturbation magnitude by maximizing the spread of the PNC scores on a reference set.
We demonstrate through comprehensive numerical experiments that our methods effectively capture the representation reliability with a high degree of correlation, achieving robust and favorable performance compared with baseline methods.
\end{abstract}


\section{Introduction}
\label{sec::intro}

Self-supervised learning (SSL) has opened the door to the development of general-purpose embedding functions, often referred to as \emph{foundation models}, that can be used or fine-tuned for various downstream tasks \citep{jaiswal2020survey, jing2020self}.
These embedding functions are pre-trained on large corpora of different data modalities, spanning visual \citep{chen2020simple, oquab2023dinov2}, textual \citep{brown2020language, grattafiori2024llama}, audio \citep{al2021clar, radford2023robust}, and their combinations \citep{radford2021learning, morgado2021audio, liu2023visual, openai2023gpt4, team2023gemini}, aimed at being general-purpose and agnostic to the downstream tasks they may be utilized for. For instance, the recent surge in large pre-trained models such as GPT-5 \citep{singh2025openai} has resulted in the development of many prompt-based, dialogue-based, and agentic downstream use cases, none of which are known a priori when the pre-trained model is being deployed.

Embedding functions learned through SSL may produce unreliable outputs. For example, large language models can generate factually inaccurate information with a high level of confidence \citep{bommasani2021opportunities, tran2022plex}. With the increasing use of SSL to generate textual, visual, and audio content, unreliable embedding functions could have significant implications. Furthermore, given that these embedding functions are frequently employed as frozen backbones for various downstream use cases, %
adding more labeled downstream data may not improve the performance if the initial representation is unreliable. %
Therefore, having notions of \emph{reliability} or \emph{uncertainty} for such pre-trained models, alongside their abstract representations, would be a key enabler for their reliable deployment, especially in high-stakes and safety-critical settings.

In this paper, we introduce a formal definition of representation reliability based on its impact on downstream tasks.
Our definition pertains to a representation of a given test point produced by an embedding function, or to a collection of representations produced by an ensemble of embedding functions.
If a variety of downstream tasks that build upon this (or these) representation(s) consistently yield accurate results for the test point, we consider it reliable. 
Existing uncertainty quantification (UQ) frameworks mostly focus on the supervised learning setting, where they rely on the consistency of predictions across various predictive models. We provide a counterexample showing that they cannot be directly applied to this setting, as representations lack a ground truth for comparison. In other words, inconsistent representations do not necessarily imply that they are unreliable because their downstream predictions could still be consistent. Hence, it is critical to align representation spaces in such a way that corresponding regions have similar semantic meanings before comparing them.

To begin, we investigate a method for estimating representation reliability without prior knowledge of the downstream tasks, when an ensemble of embedding functions is available. 
We prove that a test point has reliable representations if it has a reliable neighbor that remains consistently close to the test point, across the representation spaces generated by different embedding functions.
Based on this theoretical insight, we propose a \emph{neighborhood consistency} (NC) framework: we first draw a set of reference points from pretraining data, then compute the number of consistent neighboring points among the reference points to estimate the representation reliability.
The underlying reasoning is that a test point with more consistent neighbors is more likely to have a reliable and consistent neighbor. This reliable and consistent neighbor can be used to align different representation spaces that are generated by different embedding functions.

While the neighborhood consistency approach described above provides strong reliability estimates, it relies on an ensemble of embedding functions.
In many practical scenarios, a practitioner may only possess a single pre-trained model.
One possible approach is to train additional models and apply neighborhood consistency on the new ensemble. However, the necessary computational resources to train those models may not always be available.
To address this limitation, we present an extension of the neighborhood consistency method that replaces the deep ensemble with an efficiently generated \emph{synthetic ensemble}, constructed by adding controlled noise to the parameters of a single pre-trained model. We call this method \emph{perturbation-based} neighborhood consistency (PNC). This strategy preserves the theoretical foundation of our neighborhood consistency framework while significantly reducing the computational burden. 

We conduct extensive numerical experiments to validate our approach and compare it with baselines, including state-of-the-art out-of-distribution (OOD) detection and UQ metrics in supervised learning. Our approach consistently captures representation reliability across different models and predicts how these embedding functions perform on both in-distribution and out-of-distribution downstream tasks. Importantly, we show that while the performance of the baselines fluctuates significantly across different settings and can even become negative quite often, our metrics consistently demonstrate a positive correlation with downstream performance. 

In summary, our main contributions are:
\begin{enumerate}[leftmargin=1.0em]
    
    \item We present a formal definition of representation reliability, which quantifies how well an abstract representation can be used across various downstream tasks. 
    To the best of our knowledge, this is the first comprehensive study to investigate uncertainty in representation space.

    \item We provide a counterexample, showing that standard UQ tools in supervised learning cannot be directly applied to estimate the representation reliability. 
        
    \item We prove that a reliable and consistent neighbor of a test point can serve as an anchor point for aligning different representation spaces and assuring the representation reliability.
    
    \item Based on our theoretical findings, we introduce \emph{neighborhood consistency}, a downstream-task-agnostic method for estimating the reliability of representations produced by an ensemble of embedding functions.

    \item We extend the neighborhood consistency method for an individual embedding function by generating a synthetic ensemble via adding controlled parameter noise.
    We prove that a reliable and consistent neighbor can still serve as an anchor point in this synthetic ensemble.

    \item We conduct comprehensive numerical experiments, showing that our approaches consistently capture the representation reliability and outperform baselines. 

\end{enumerate}

\paragraph{Implications and Broader Impact.}
Our work introduces a way to quantify the reliability of pre-trained models prior to their deployment in specific downstream tasks, which has several implications. Imagine a practitioner has access to multiple types of pre-trained models that have been trained using distinct learning paradigms, data, or architectures. Our method helps compare and rank these (ensembles of) models based on their reliability scores, enabling the practitioner to select the model with the highest reliability score. This is particularly valuable when transmitting downstream training data is challenging due to privacy concerns or when the downstream tasks shift over time. 
Similarly, in cases where a pre-trained model yields incorrect (or harmful) decisions for specific individuals, our method can help explain whether the issue stems from the abstract representation or the projection heads added to the pre-trained model. 
We hope our effort can push the frontiers of self-supervised learning towards more responsible and reliable deployment, encouraging further research to ensure the transferability of knowledge acquired during pre-training across diverse tasks and domains.

\subsection{Related Work}

\paragraph{Uncertainty Quantification in Supervised Learning.} Existing work on UQ mostly focused on supervised learning settings.
For example, Bayesian inference quantifies uncertainty by placing a prior distribution over model parameters, updating this prior distribution with observed data to obtain a posterior distribution, and examining the inconsistency of predictions derived from the posterior distribution \citep{neal1996bayesian, mackay1992practical,kendall2017uncertainties,depeweg2018decomposition}. Since the posterior distribution may not have an analytical form, many approximating approaches have been introduced, including Monte Carlo dropout \citep{gal2016dropout}, deep ensembles \citep{osband2018randomized, lakshminarayanan2017simple, wen2020batchensemble}, and Laplace approximation \citep{daxberger2021laplace,sharma2021sketching}. In this paper, we focus on quantifying the uncertainty of representations and prove that standard supervised-learning frameworks cannot be directly applied to investigate representation uncertainty (see Section~\ref{sec:consistency} for more details).

\paragraph{Novelty Detection and Representation Reliability.}  Self-supervised learning is increasingly used for novelty/OOD detection. These approaches train self-supervised models and then compute an OOD score for a new test point based on its distance from the training data in the representation space \citep{lee2018simple, van2020uncertainty, tack2020csi, mirzae2022fake}.
However, OOD detection and our representation reliability are different concepts. 
\rev{The former identifies whether a test point belongs to the same distribution as the (pre-)training data, while the latter evaluates the possibility that a test point can receive accurate predictions when the self-supervised learning model is adapted to various downstream tasks (see Section~\ref{sec:defin} for more details).}

To compare with this line of work, we conduct comprehensive numerical experiments (Section~\ref{sec::exp}). 
The results suggest that our approach more robustly captures the representation reliability compared with state-of-the-art OOD detection measures and the empirical metrics proposed by \citet{ardeshir2022uncertainty}.
Finally, our representation reliability extends the notion of probe proposed by \citet{haochen2021provable} to multiple downstream tasks. We introduce an algorithm for estimating the representation reliability without prior knowledge of the specific downstream tasks.

\paragraph{Uncertainty-Aware Representation Learning.} 
There is a growing body of research aimed at training robust self-supervised models that map input points to a distribution in the representation space, rather than to a single point \citep{vilnis2014word,neelakantan2015efficient,karaletsos2015bayesian,bojchevski2017deep,oh2018modeling,chen2020simple,wu2020simple,zhang2021temperature,almecija2022uncertaintyaware}. 
They rely on special neural network architectures and/or introduce alternative training schemes. 
For example, the approach by \citet{zhang2021temperature} requires an additional output (i.e., a temperature parameter) and the approach by \citet{oh2018modeling} requires the network to output means and variances of a mixture of Gaussian distributions.
In contrast, we avoid making any assumptions about the training process of the embedding functions, while only needing black-box access to the function for computing NC and only access to the model's parameters and architecture for computing PNC. 
Furthermore, we provide a theoretical analysis of our method and explore the impact of the representation reliability on the performance of downstream tasks.
We provide a more in-depth discussion about related works in Appendix~\ref{app:summary}.

{
\paragraph{Extensions over the Conference Version.}
The conference version of this work \citep{park2024quantifying} assumes access to an ensemble of embedding functions, which is often impractical due to the high cost of pre-training multiple models. We remove this requirement by proposing perturbation-based neighborhood consistency (PNC), which constructs a synthetic ensemble by injecting isotropic Gaussian noise into the parameters of a single pre-trained model (Section~\ref{sec:pnc}). We also provide a label-free criterion for choosing the perturbation magnitude. 
On the theory side, we prove a new converse result (Theorem~\ref{thm:nb_necessary}) of the neighborhood-consistency guarantee: if a test point has no consistent neighbors, then it is either out-of-distribution or unreliable. We further extend the guarantee to the perturbation setting (Theorem~\ref{thm::perturbed_consistency}), providing the theoretical basis for PNC.
On the empirical side, while the conference version evaluated NC only in the ensemble setting using downstream binary classification tasks, this work adds a comprehensive evaluation of PNC in the single-model setting (Section~\ref{sec:individual_exp}). We also extend the downstream evaluation of both NC and PNC to multi-class classification tasks alongside the binary one-vs-one tasks (Appendices~\ref{app:multi} and \ref{app:individual_exp_multi})
}

\section{Background}

\subsection{Uncertainty Quantification in Supervised Learning} \label{sec:bg}
We recall a Bayesian-inference view of epistemic uncertainty in supervised learning. Consider predictive models, where each model (e.g., a neural network) $f$ outputs a predictive probability $p(y|\bx, f)$ for an input variable $\bx$. In the Bayesian framework, a prior probability distribution $p(f)$ is first introduced and a posterior distribution is learned given a training dataset $\mathcal{D}$: $p(f | \cD) \propto p(\cD | f) \cdot p(f)$. For a new test point $\bx^*$, its  posterior predictive distribution is obtained by averaging the predictive probabilities over models: 
\begin{equation}
\label{eq::pred_post}
    p(\hat{y}^* \mid \bx^*, \cD) = \int p(\hat{y}^* \mid \bx^*, f) p(f \mid \cD) \dif f.
\end{equation}

Since the posterior distribution does not have an analytical expression for complex neural networks, Monte Carlo approaches could be used to approximate the integral in Equation~\eqref{eq::pred_post}.
One of the most prominent approaches is deep ensembles \citep{lakshminarayanan2017simple}, in which a collection of neural networks $\mathcal{F}$ are trained with $M$ different random initialization of the learnable parameters $\{\theta_i\}_{i=1}^M$.
The posterior predictive distribution is then approximated by
\begin{equation}
p(\hat{y}^* \mid \bx^*, \cD) \approx \frac{1}{M} \sum_{i=1}^{M}  p(\hat{y}^* \mid \bx^*, f_{\theta_i}),
\end{equation}
where $p(\hat{y}^* \mid \bx^*, f_{\theta_i})$ denotes the predictive distribution of the $i$-th model.
Finally, the uncertainty can be assessed by the \emph{inconsistency} of the output across sampled functions, which can be quantified by metrics such as variance\footnote{For a random vector, its variance is defined as the trace of its covariance matrix.} (for regression tasks) or entropy (for classification tasks).
For example, in regression tasks, the uncertainty can be estimated by
\begin{align}
     \Unc{\bx^* \mid \cD} & = \Var{\hat{y}^* \mid \bx^*, \cD} \label{eq:unc_sup} 
    \approx \Varr{i\sim\{1,\ldots,M\}}{f_{\theta_i}(\bx^*)} 
    = \frac{1}{M^2} \sum_{i < j} \lVert f_{\theta_i}(\bx^*) - f_{\theta_j}(\bx^*) \rVert_2^2,
\end{align}
where $\Varr{i\sim\{1,\ldots,M\}}{\cdot}$ is the population variance over indices $i \in \{1, \cdots, M\}$.
When different models disagree on their predictions, it suggests a high level of uncertainty. Conversely, if multiple predictive models give similar outputs, the prediction can be considered certain. 
However, this principle does not hold in self-supervised learning, as representations do not have a ground truth and different pre-trained models can carry different semantic meanings.

\subsection{Novelty Detection in Representation Space}
\label{sec:bg_novelty}

Several studies have proposed ways to detect OOD samples by assessing the novelty of data representations using a pre-trained embedding function, $h$ \citep{lee2018simple, van2020uncertainty, tack2020csi, mirzae2022fake}.
Although the specific details of each technique vary, we observe that these methods commonly rely on the relative distance of the query data point's representation to other reference points:
\begin{equation}
    D_{\mathsf{dist}}(\bx^*, \bX_{\text{ref}}, h) := \big\{ \mathsf{dist} \big( h(\bx^*), h(\bx) \big) \mid \bx \in \bX_{\text{ref}} \big\},
\end{equation}
where $\bx^*$ denotes the query, $\bX_{\text{ref}}$ denotes the set of reference points (usually being the pretraining dataset), and $\mathsf{dist}$ denotes the distance metric function.

For instance, \citet{tack2020csi} measures the minimum distance between the query instance and all training instances in the representation space. If we denote $\amk$ to be the average of the $k$ smallest values of a set, this equals $\mathsf{Dist}_1(D_{\mathsf{dist}}(\bx^*, \bX_{\text{ref}}, h))$.
\citet{mirzae2022fake} averages the distance to the $k$-nearest instances rather than taking only the closest one to improve robustness, which equals $\amk(D_{\mathsf{dist}}(\bx^*, \bX_{\text{ref}}, h))$.

\begin{figure*}[t!]
  \centering
  \includegraphics[width=0.95\textwidth]{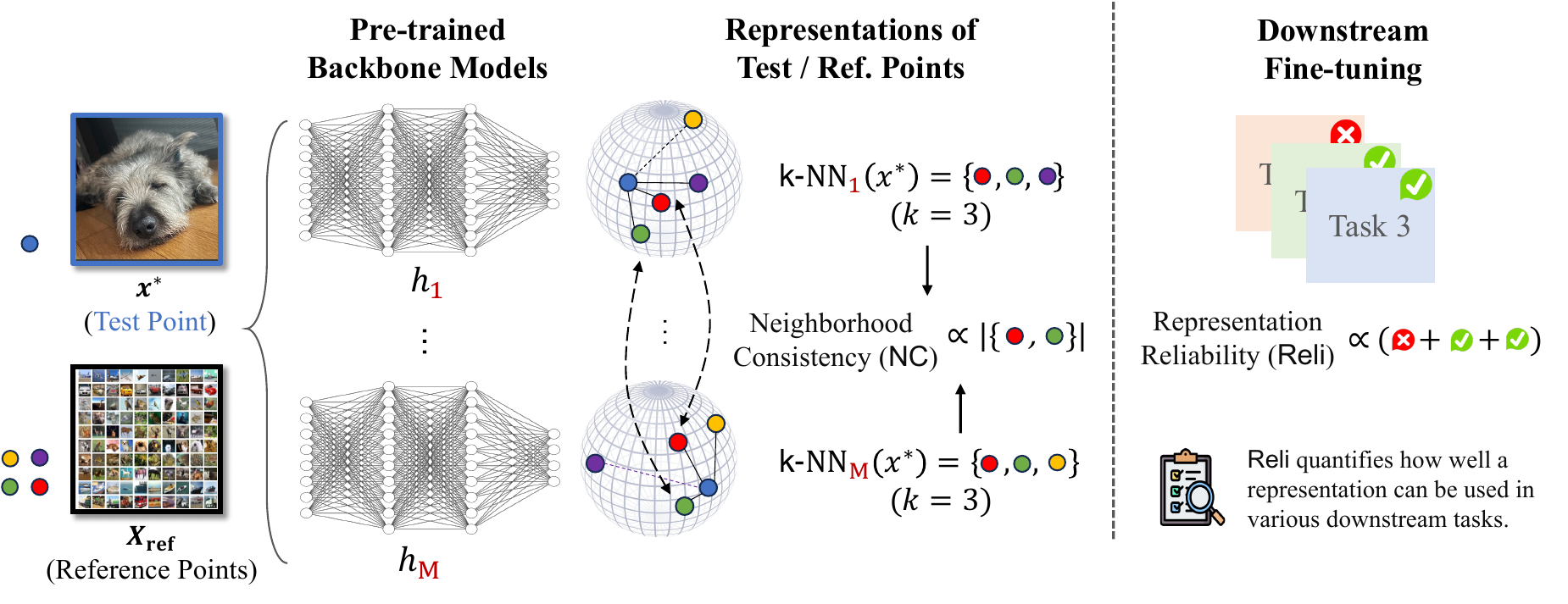}
  \caption{
  Illustration of \emph{representation reliability} ($\mathsf{Reli}$) and \emph{neighborhood consistency} ($\mathsf{NC}$). For a test point $\bx^*$ and a collection of pre-trained backbone models $\mathcal{H} = \{h_1, \cdots, h_M\}$, the representation reliability is defined as the average performance of downstream models when using the representations of $\bx^*$ provided by the backbones in $\mathcal{H}$.
  \textbf{Our $\mathsf{NC}$ estimates $\mathsf{Reli}$ without requiring any prior knowledge of the downstream tasks.} It operates by measuring the number of consistent neighbors of $\bx^*$ among reference points across different representation spaces.}
  \label{fig:repr_reli}
\end{figure*}

\section{Quantifying Representation Reliability}

In this section, we introduce a framework for assessing the reliability of representations assigned by pre-trained models. We discuss limitations of existing UQ frameworks in supervised learning when applied to representation spaces.
Then we present a method that examines the consistency of neighboring points in the representation space.
Our method effectively captures the representation reliability without the need for a priori knowledge of downstream tasks.

\subsection{Representation Reliability} \label{sec:defin}

We introduce a formal definition of \emph{representation reliability} by examining its impact on various downstream tasks. Intuitively, a reliable representation assigned by a pre-trained model should consistently yield accurate outcomes when the model is adapted to these downstream tasks.

We introduce some notations that will be used in our definition. We define $\mathcal{H}$ as the collection of embedding functions trained by a self-supervised learning algorithm (e.g., SimCLR \citep{chen2020simple}) using a pre-trained dataset (e.g., ImageNet \citep{deng2009imagenet}). Each embedding function $h:\mathcal{X} \to \mathcal{Z}$ maps a data point (e.g., an image) to an abstract representation. Here the representation space $\mathcal{Z}$ can be either a real $d$-space $\Reals^d$ or a unit hyper-sphere $\mathcal{S}^{d-1}$.

We consider a collection of downstream tasks (e.g., classifications), denoted as $\mathcal{T}$. Each task is associated with a set of downstream heads (i.e., additional layers added on top of the embedding function) and a population risk function that assesses the performance of each downstream head. For each task $t$, we optimize the risk function to obtain an optimal head $g_{h,t}$. This way, we can eliminate the impact of downstream training processes on our definition. The representation reliability of a new test point $\bx^*$ is measured by the (average) performance of these downstream models on $\bx^*$. A formal definition is provided below.

\begin{defn}[Ensemble Representation Reliability]
\label{defn::rep^rel}
Let $\mathcal{H}=\{h_1,h_2,...,h_M\}$ be a finite set of embedding functions. For each $h_i\in\mathcal{H}$ and each task $t \in \mathcal{T}$, denote the optimal downstream prediction head by $g_{i,t}.$\footnote{For ease of notation, we use $g_{i,t}$ to refer to $g_{h_i,t}$.} 
{Further, denote the multiset of predictions induced by $\mathcal H$ and their respective prediction heads $g_{i,t}$ for task $t$ for a test point $\bx^* \in \cX$ with label $y_t^*$ as}
\begin{equation}
    {Y_t(\bx^*) := \{g_{i,t} \circ h_i(\bx^*):h_i\in\mathcal{H}\} .}
\end{equation}
We define the \emph{representation reliability} for a test point $\bx^* \in \cX$ with task-dependent labels $y_t^*$ for $t \in \mathcal{T}$ as:
\begin{equation} \label{eq:rep_reli}
    \mathsf{Reli}(\bx^*; \cH, \mathcal{T}) \defined \frac{1}{|\cT|} \sum_{t \in \cT} \Perf{Y_t(\bx^*),y_t^*},
\end{equation}
{where $\Perf{\cdot,\cdot}$ measures the performance of a multiset of predictions for task $t$.} 
\end{defn}

\subsection{Computing Representation Reliability}
In this section, we provide some examples to compute representation reliability using standard uncertainty or accuracy measures using an ensemble of embedding functions ${\mathcal{H}}=\{h_1, \cdots, h_M\}$ and the corresponding downstream predictors $\{g_{1,t}, \cdots, g_{M,t}\}$.

In regression tasks, the negative of the (population) variance of the predictions can be used as a performance metric:
\begin{align} 
\myPerf{t} := & -\Varr{i\sim\{1,\ldots,M\}}{g_{i,t} \circ h_i(\bx^*)} \nonumber \\
 = & -\frac{1}{M^2} \sum_{i < j} \Big( g_{i,t} \circ h_i(\bx^*) - g_{j,t} \circ h_j(\bx^*) \Big)^2 . \label{eq:rep_reli_var}
\end{align}
In the context of multi-output tasks, the trace of the covariance matrix can be utilized, which is essentially the sum of variances across each output dimension. This approach aggregates the individual uncertainties of each output, providing a comprehensive measure of overall uncertainty in the multi-output setting.

While our theoretical analysis has focused on regression tasks using variance, for classification tasks, metrics like the Brier score or entropy are more natural choices. For a given downstream classification task $t$ with $C$ classes, the ground truth label of $\bx^*$ is represented by a one-hot vector: 
$y_{t}^* = ( [y_{t}^*]_1, \cdots, [y_{t}^*]_C)^T$ where $[y_{t}^*]_c \in \{0, 1\}$.
The softmax output for class $c$ via the predictive function $g_{i,t} \circ h_i (\cdot)$ is denoted as $[g_{i,t} \circ h_i (\cdot)]_{c}$. %

The negative Brier score for this setup is calculated as
\begin{gather}
\myPerf{t} := -\mathsf{Brier} (\EEE{i \sim \{1,\ldots,M\}}{g_{i,t} \circ h(\bx^*)} ; y_{t}^*) \nonumber \\
\hspace{4.3cm}  = - \sum_{c=1}^{C} \Big( \frac{1}{M} \sum_{i=1}^{M} [g_{i,t} \circ h_i(\bx^*)]_{c} - [y_{t}^*]_c \Big)^2,
\end{gather}
and the negative entropy is given by
\begin{align}
\myPerf{t} := & -\mathsf{Entropy} ( \EEE{i \sim \{1,\ldots,M\}}{g_{i,t} \circ h(\bx^*)} ) \nonumber \\
= & \sum_{c=1}^{C} \Big(\frac{1}{M} \sum_{i=1}^{M} [g_{i,t} \circ h_i(\bx^*)]_{c} \Big) \log \Big(\frac{1}{M} \sum_{i=1}^{M} [g_{i,t} \circ h_i(\bx^*)]_{c} \Big),
\end{align}
where $\EEE{i\sim\{1,\ldots,M\}}\cdot$ denotes the average.
As illustrated, computing the representation reliability requires the set of downstream tasks (and ground-truth labels). 
In practice, this may not always be the case (see Broad Impact and Implication in Section~\ref{sec::intro} for examples). Next, we discuss how to estimate the representation reliability based on the properties of the representation itself, without prior knowledge of the downstream tasks.

\subsection{First Attempt: Representation Consistency} 
\label{sec:consistency}

Our first attempt is directly applying standard supervised-learning techniques (see Section~\ref{sec:bg}) to estimate the representation reliability. Recall that if multiple predictive models give different predictions for the same test point, then it is likely that their predictions are uncertain. One may wonder whether the same idea can be applied to estimate the representation reliability. We present a negative result, showing that even if different embedding functions produce completely different representations, their downstream predictions can still be consistent. Below, we provide an (informal) theorem and defer a more rigorous statement along with its proof to Appendix~\ref{append::sl_count_ex}.

\begin{thm}
\label{thm::sl_count_exam_inf}
For any constant $A$ and a test point $\bx^*$, there exist embedding functions $h_1,\cdots,h_{M} \in \mathcal{H}$ such that $\Varr{i\sim \{1,\ldots,M\}}{h_i(\bx^*)} \geq A$ but $\Varr{{i\sim\{1,\ldots,M\}}}{g_{i,t} \circ h_i(\bx^*)} = 0$ for any downstream task $t$, where $g_{i,t}$ is an optimal downstream head for $h_i$ under task $t$. 
\end{thm}
The key insight behind our proof is that an input point's representation is not unique (e.g., rotated spaces are in fact equivalent). In other words, even if different embedding functions assign distinct representations to the same test point, downstream heads built on these embedding functions can also vary, ultimately leading to similar predictions.

\subsection{Theoretical Foundation: The Impact of Reliable Neighboring Point} \label{sec:nb_consistency}

To address the aforementioned issue, we propose the idea of using an ``anchor'' point to align different representation spaces. The anchor point serves as a bridge that transforms different representation spaces into the same space. The graphical visualization of this idea is illustrated in Figure~\ref{fig:repr_reli} and Figure~\ref{fig:sketch_proof}. We formalize this intuition more rigorously in the following theorem. It states that if a test point has a (reliable) consistent neighboring point across all representation spaces, then the uncertainty of its downstream predictions is bounded above.

\begin{thm}
\label{thm::nb_consistency}
For a test point $\bx^*$, suppose that there exists a \textbf{consistent} neighbor $\bx^r$ across all embedding functions $\mathcal{H} = \{h_1, \cdots, h_M\}$, i.e., satisfying
\begin{equation} \label{eq:nb}
    \lVert h_i(\bx^r) - h_i(\bx^*) \rVert_2 \le \epsilon_{nb} ~, ~ \forall i \in \{1,\ldots,M\} . 
\end{equation}
Suppose the downstream heads $g_{i,t}$ are Lipschitz continuous. Then, for any downstream task $t$, the variance of downstream prediction at $\bx^*$ is bounded above by:
\begin{equation} \label{eq:ens_var}
\Varr{i\sim\{1,\ldots,M\}}{g_{i,t} \circ h_i(\bx^*)} \le (\sqrt{2} L_t \epsilon_{nb} + \sigma_{r,t})^2.
\end{equation}
where $\sigma_{r,t}^2=\Varr{i\sim\{1,\ldots,M\}}{g_{i,t} \circ h_i(\bx^r)}$ is the uncertainty of $\bx^r$ as measured by variance, the function $g_{i,t}$ is the optimal downstream head for task $t$ built upon $h_i$, whose Lipschitz constant is $L_{i, t}$, and $L_t = \max_i L_{i, t}$.\footnote{When spectral normalized weights are employed with 1-Lipschitz continuous activation functions (e.g., identity, ReLU, sigmoid, or softmax), the constant $L_t = 1$.}
\end{thm}
See Appendix~\ref{pf::nb_consistency} for proof and extension to the multi-output case. The key takeaway is that if we can find a reference point consistently close to the test point, it can serve as an anchor point that helps align representation spaces with distinct semantic meanings. 
Furthermore, the uncertainty of the reference point along with its relative distance to the test point ensures an upper bound on the uncertainty of the test point, which implies the lower bound on its representation reliability.
A natural follow-up question is whether the converse holds: does the absence of consistent neighbors imply unreliability? The following theorem provides an affirmative answer under additional assumptions.

\begin{thm} \label{thm:nb_necessary}
Let $t$ be a downstream task. Suppose the reference set $\bX_{\text{ref}}$ is reliable, such that for all $\bx^r \in \bX_{\text{ref}}$, the downstream predictions are consistent across all embedding functions $\mathcal{H} = \{h_1, \dots, h_M\}$:
\begin{equation} \label{eq:ref_uniform_reliable}
\max_{i,j\in\{1,\ldots,M\}} \big\| g_{i,t}\circ h_i(\bx^r) - g_{j,t}\circ h_j(\bx^r) \big\|_2 \le \sigma_{\text{ref}}.
\end{equation}
Suppose the downstream heads $g_{i,t}$ are bi-Lipschitz continuous\footnote{Bi-Lipschitz continuity ensures that both the function and its inverse mapping are Lipschitz continuous, a property satisfied by e.g., linear heads.}.
If a test point $\bx^*$ \textbf{does not} have any consistent neighboring point in $\bX_{\text{ref}}$ (i.e., no reference point is close to $\bx^*$ across all representation spaces), then $\bx^*$ must satisfy at least one of the following conditions:
\begin{enumerate}
    \item[(i)] \textbf{OOD:} $\bx^*$ is far from the reference set in every representation space. Specifically, for every $i \in \{1,\ldots,M\}$, $\bx^*$ has no $\epsilon$-neighbor in $\bX_{\text{ref}}$.
    \item[(ii)] \textbf{Unreliable:} The ensemble predictions for $\bx^*$ are inconsistent. Specifically, there exist $i,j \in \{1,\ldots,M\}$ such that:
    \begin{equation}
        \big\| g_{i,t}\circ h_i(\bx^*) - g_{j,t}\circ h_j(\bx^*) \big\|_2 > \sigma_{\text{ref}}.
    \end{equation}
\end{enumerate}
\end{thm}

Theorem~\ref{thm:nb_necessary} is the contrapositive of the observation that reliable, in-distribution points naturally induce consistent neighborhoods due to the bi-Lipschitz continuity of the heads. See Appendix~\ref{app:necessary} for a more rigorous statement along with the proof.

\subsection{Proposed Algorithm: Neighborhood Consistency} \label{sec:nc_algorithm}

In practice, identifying a reliable reference point is challenging without prior knowledge of the downstream tasks. Instead of searching for a single point as the anchor, we draw a set of reference points, denoted as $\bX_{\text{ref}} = \{\bx^{(l)}\}_{l=1}^n$. We then compute the number of consistent neighboring points within $\bX_{\text{ref}}$ and use it to estimate the representation reliability. The rationale behind this is that a test point with more consistent neighbors is more likely to have a reliable and consistent neighbor.

\paragraph{Our Algorithm.} Given an ensemble of embedding functions $h_1,\cdots, h_M$ and $\bX_{\text{ref}}$, we define the \textbf{Neighborhood Consistency (NC)} of a test point $\bx^*$ as:
\begin{equation} \label{defn::NC}
    \mathsf{NC}_{k}(\bx^*) = \frac{1}{M^2} \sum_{i < j} \Sim{k\text{-NN}_i \big( \bx^* \big), ~ k\text{-NN}_j \big( \bx^*\big) }
\end{equation}
where $k\text{-NN}_i(\bx^*)$ is the index set of $k$-nearest neighbors of $h_i(\bx^*)$ among $\{h_i(\bx) \mid \bx \in \bX_{\text{ref}}\}$ 
and $\Sim{\cdot, \cdot}$ is a measure of similarity between sets (e.g., \emph{Jaccard Similarity}). A further discussion on NC is described in Appendix~\ref{app:alternative}.

\begin{rem}
The parameter $k$ involves a trade-off between two factors: choosing a neighborhood closer to the test point and increasing the chance of incorporating a more reliable neighbor. As shown in Theorem~\ref{thm::nb_consistency}, the reliability of a test point's representation hinges on whether it has a nearby reliable neighbor. Hence, a large $k$ may result in selecting a reliable neighbor, but it could be far from the test point.
On the other hand, since the upper bound in Theorem~\ref{thm::nb_consistency} holds using any one of the neighbors as the anchor point, the reliability of the test point is bounded using the neighbor with the smallest $\sigma_{r,t}$.
Consequently, a small $k$ may compromise the reliability of the selected neighboring point.
We present an empirical study about this trade-off in Section~\ref{sec:ablation}.
\end{rem}

\section{Ensemble-Free Extension of Neighborhood Consistency} \label{sec:pnc}

In the previous section, we assumed access to an ensemble of independently trained embedding functions. However, in practice, only a single pre-trained model may be available. 
We first extend our Definition~\ref{defn::rep^rel} as follows:
\begin{defn}[Individual Representation Reliability]
\label{defn::rep^rel_single}
Consider a single embedding function $h$.
For each task $t \in \mathcal{T}$, denote the optimal downstream prediction head by $g_{h,t}$.
We extend the definition of \emph{representation reliability} for a test point $\bx^* \in \cX$ with task-dependent labels $y_t^*$ for $t \in \mathcal{T}$ as:
\begin{equation}
    \mathsf{Reli}(\bx^*; h, \mathcal{T}) \defined \frac{1}{|\cT|} \sum_{t \in \cT} \Perf{\{g_{h,t} \circ h(\bx^*)\},y_t^*}.
\end{equation}
This corresponds to the degenerate case of $\mathcal{H} = \{h\}$ in Definition~\ref{defn::rep^rel}. In this setting, standard performance metrics such as the Brier score or entropy can be directly applied, while variance-based measures are not applicable.
\end{defn}

Similarly, we aim to extend our NC algorithm to an ensemble-free setting.
Specifically, we propose an efficient approach inspired by the connection between flat minima and generalization \citep{hochreiter1997flat, keskar2016genflat}: a model whose output is insensitive to small perturbations of its weights is more likely to generalize well. Building on this insight, we construct a \emph{synthetic ensemble} by perturbing the parameters of the given model, and then apply $\mathsf{NC}_k$ to this ensemble. Given the original embedding model $h(\cdot;\theta^*)$ with parameters $\theta^*$, we sample perturbed parameters $\theta_1, \ldots, \theta_{M-1}$ from a distribution $\mathcal{P}$ centered at $\theta^*$ and form the synthetic ensemble $\{h(\cdot;\theta^*), h(\cdot;\theta_1), \ldots, h(\cdot;\theta_{M-1})\}$. We refer to this method as \textbf{perturbation-based neighborhood consistency (PNC)}.

Analogous to Theorem~\ref{thm::nb_consistency}, we show that a consistent neighbor that is robust to parameter perturbations guarantees the reliability of the test point under the same perturbations.

\begin{thm}\label{thm::perturbed_consistency}
    For a test point $\bx^*$, suppose there exists a consistent neighbor $\bx^r$ across most embedding functions $h(\cdot;\theta)$, for $\theta$ sampled from some distribution $\mathcal{P}$, satisfying
    \[\EEE{\theta\sim\mathcal{P}}{\|h(\bx^*;\theta)-h(\bx^r;\theta)\|_2^4}\leq\epsilon_{nb}^4.\]
    Suppose the downstream heads $g_{\theta,t}$\footnote{For ease of notation, we use $g_{\theta,t}$ to refer to $g_{h(\cdot;\theta),t}$.} are $L_{\theta,t}$-Lipschitz continuous. Then for any downstream task $t$, the variance of downstream prediction at $\bx^*$ is bounded above by:
    \[\mathsf{Var}_{\theta\sim\mathcal{P}}[g_{\theta,t}\circ h(\bx^*;\theta)]\leq(\sqrt{2}L_t\epsilon_{nb}+\sigma_{r,t})^2,\]
    where $\sigma_{r,t}^2=\mathsf{Var}_{\theta\sim\mathcal{P}}[g_{\theta,t}\circ h(\bx^r;\theta)]$ is the reliability of $\bx^r$ as measured by variance, the function $g_{\theta,t}$ is the optimal downstream head for task $t$ built upon $h(\cdot;\theta)$, and $L_t=\EEE{\theta\sim\mathcal{P}}{L_{\theta,t}^4}^{1/4}$.
\end{thm}

Theorem~\ref{thm::perturbed_consistency} generalizes Theorem~\ref{thm::nb_consistency}: the latter is recovered as a special case by taking $\mathcal{P}$ to be the uniform distribution over the finite set $\mathcal{H}$. Because Theorem~\ref{thm::perturbed_consistency} bounds the quartic mean rather than the maximum, it yields a strictly less restrictive condition.

\subsection{Implementation of Perturbation-Based Neighborhood Consistency}
\label{sec:pnc:desc}

Motivated by Theorem~\ref{thm::perturbed_consistency}, we define $\mathsf{PNC}_k$ metric as follows. Given the pre-trained model $h(\cdot; \theta^*)$ and a test point $\bx^*$, we parameterize the perturbation as $\theta_i = \theta^* + \Delta\theta_i$, where each $\Delta\theta_i$ is drawn independently from a noise distribution. We then compute the $\mathsf{NC}_k$ score on the resulting synthetic ensemble. For simplicity, we adopt an isotropic Gaussian perturbation:
\[\Delta\theta\sim\mathcal{N}(0;\sigma_{pert}^2I),\]
where $\sigma_{pert}$ controls the perturbation magnitude.

\paragraph{Tuning the perturbation magnitude.}
The choice of $\sigma_{pert}$ is critical, and for the discussion of our method for tuning it, we make an assumption that the metric $\mathsf{Sim}$ of any two sets with size $k$ is bounded. Furthermore, without loss of generality, we assume that its infimum is zero while its supremum is one (this is true for Jaccard similarity). If $\sigma_{pert}$ is too small, the perturbed models are nearly identical to the original, and all $\mathsf{PNC}_k$ scores collapse to one regardless of the test point. If $\sigma_{pert}$ is too large, the perturbed representations lose semantic structure, and neighborhoods become effectively random, driving all scores toward zero. 

The rate at which the $\mathsf{PNC}_k$ scores collapse differs between each point because each point have their own level of sensitivity towards parameter noise.
Because we want our PNC score to rate uncertainty based on this sensitivity, we would want to choose the $\sigma_{pert}$ so that certain points are not sensitive towards that level of noise, while others are.
We therefore decided to select $\sigma_{pert}$ to maximize the sample variance of the set of $\mathsf{PNC}_k$ scores over the pre-training set,
which would avoid the two degenerate cases of having $\sigma_{pert}$ being too small or too large and allowing us to separate out less sensitive points to more sensitive points.
We refer to this selection criterion as $\mathsf{PNC}$-spread tuning, and denote the resulting method by $\mathsf{PNC}_k^s$.

\section{Numerical Experiments}
\label{sec::exp}

\subsection{Experiment Setup} 
\label{sec:setup}

\paragraph{Embedding Functions.} We use SimCLR \citep{chen2020simple}, Bootstrap Your Own Latent (BYOL) \citep{grill2020bootstrap}, and Momentum Contrast (MoCo) \citep{he2020momentum} to pre-train ResNet-18 and ResNet-50 models \citep{he2016deep} as the embedding functions, using a single NVIDIA V100 GPU. During the pre-training stage, we do not use any class label information. 
To construct the ensemble, we train a total of $M=10$ embedding functions using each pre-training algorithm and dataset with different random seeds, following \cite{lakshminarayanan2017simple}.
The impact of $M$ on our experiments is studied in Section~\ref{sec:ablation}.

\paragraph{Baselines.} We compare our proposed method---$\mathsf{NC}_k$ in Equation~\eqref{defn::NC}---against the following widely-recognized baseline methods:
\begin{itemize}[leftmargin=2em]
    \item $\amk$ \citep{tack2020csi, mirzae2022fake}: the average of the $k$ minimum distances from the test point to the reference data in the representation space. A lower value of $\amk$ indicates higher reliability.
    
    \item $\mathsf{Norm}$ \citep{tack2020csi}: the $L_2$ norm of the representation $\lVert h(\bx^*) \rVert_2$. A higher value of $\mathsf{Norm}$ indicates higher reliability.

    \item Feature Variance ($\mathsf{FV}$): representation consistency measured by $\Varr{i\sim \{1,\ldots,M\}}{h_i(\bx^*)}$, as described in Theorem~\ref{thm::sl_count_exam_inf}.
    $\mathsf{FV}$ is extended from UQ in supervised learning. Specifically, the variance of neural networks' predicted scores is often used to measure epistemic uncertainty \citep{kendall2017uncertainties, lakshminarayanan2017simple, ritter2018scalable}. Here we apply this measure to examine latent representation spaces.
    A lower value of $\mathsf{FV}$ indicates higher reliability.
    
\end{itemize}

All baselines, except $\mathsf{FV}$, are based on a single embedding function. For a fair comparison, we consider the (point-wise) ensemble average of each score over different embedding functions for $\amk$, and $\mathsf{Norm}$.  

\paragraph{Hyperparameters.} To reduce the computational cost, we randomly select $n = 5{,}000$ pre-training data as the reference dataset $\bX_{\text{ref}}$.
We repeat our experiments $5$ times with distinct random seeds for choosing reference data and report the average evaluation scores. Since the standard deviations of our scores are on average below 1\%, we defer the experimental results with error bars to the Appendix.
We choose $k=100$ for $\mathsf{NC}_k$ and $k=1$ for $\amk$ (see Section~\ref{sec:ablation} for an ablation study about the choice of $k$). 

For our method and $\amk$, we test both cosine distance and Euclidean distance as options for distance metric. In a similar manner, $\mathsf{FV}$ is evaluated using both unnormalized $h(\bx) \in \Reals^d$ and normalized representations $h(\bx) / \lVert h(\bx) \rVert_2 \in \mathcal{S}^{d-1}$. 
A discussion on the implications of these selections is provided in Section~\ref{sec:ablation}. Given the space limits, the results using normalized representations are reported in the main manuscript, while those using unnormalized representations are detailed in Appendix~\ref{app:unnormalized}.

\paragraph{Evaluation Protocol.} 
To evaluate the effectiveness of our method (i.e., $\mathsf{NC}_{100}$) in capturing the representation reliability (Definition~\ref{defn::rep^rel}) and to compare it against baselines\footnote{We use the negative scores of $\amk$ and $\mathsf{FV}$ since their lower values indicate higher reliability.}, we use \kendall to measure the correlation between each method and the representation reliability ($\mathsf{Reli}$).
We compute the downstream performance ($\mathsf{Perf}$) on each test point by using either 1) the negative predictive entropy, capturing downstream uncertainty; or 2) the negative Brier score, reflecting predictive accuracy.

We focus on downstream classification tasks. For each task, we freeze the pre-trained model and fine-tune the linear heads.
To leverage the multi-class labels and minimize the influence from the downstream training processes, we break down each $C$-class classification into a set of one-vs-one (OVO) binary classification tasks. 
As a result, the total number of downstream tasks is $|\cT| = C (C - 1) / 2$, with each data point being evaluated in $(C - 1)$ tasks.
Finally, we average the performance across all OVO tasks to compute the representation reliability.
We defer the results for the multi-class downstream tasks to Appendix~\ref{app:multi}.


\begin{table*}[t!] 
\centering
\fontsize{9pt}{9pt}\selectfont
\caption{Comparison of our neighborhood consistency ($\mathsf{NC}_{100}$) with baseline methods in terms of their correlation with the representation reliability.
We use \kendall to measure this correlation.
Downstream tasks are in-distribution tasks where the model is pre-trained and fine-tuned on the same dataset. 
Performance on these tasks is evaluated using either negative predictive entropy or negative Brier score. The highest and second-highest scores are highlighted in bold and underlined, respectively. As shown, our method consistently receives a favorable score compared with baselines.}

\resizebox{\textwidth}{!}{%
\renewcommand{\arraystretch}{1.1}
\begin{tabular}{cc|cccc|cccc}
\toprule
\multirow{3}{*}{\makecell[c]{Pretraining \\ Algorithms}} & \multirow{3}{*}{Method} & \multicolumn{4}{c|}{ResNet-18} & \multicolumn{4}{c}{ResNet-50}  \\
\cline{3-10} 
 & & \multicolumn{2}{c}{\texttt{CIFAR-10}}  & \multicolumn{2}{c|}{\texttt{CIFAR-100}} & \multicolumn{2}{c}{\texttt{CIFAR-10}}  & \multicolumn{2}{c}{\texttt{CIFAR-100}} \\
 & & Entropy  & Brier & Entropy & Brier & Entropy  & Brier & Entropy & Brier \\
\midrule
\multirow{5}{*}{SimCLR} & $\mathsf{NC}_{100}$ & \bf 0.3865 & \bf 0.3262 & \bf 0.3130 & \bf 0.2590 & \bf 0.3786 & \bf 0.3206 & \bf 0.3070 & \bf 0.2516 \\
& $\amone$ & \underline{0.3021} & \underline{0.2544} & 0.0757 & 0.0635 & 0.2750 & 0.2376 & 0.1100 & \underline{0.0919} \\
& $\mathsf{Norm}$ & 0.2907 & 0.2429 & \underline{0.1198} & \underline{0.0649} & \underline{0.2789} & \underline{0.2385} & \underline{0.1311} & 0.0710 \\
& $\mathsf{FV}$ & -0.0476 & -0.0458 & -0.1891 & -0.1710 & -0.0402 & -0.0393 & -0.1984 & -0.1870 \\
\midrule
\multirow{5}{*}{BYOL} & $\mathsf{NC}_{100}$ & \bf 0.2749 & \bf 0.1826 & \bf 0.2354 & \bf 0.1753 & \bf 0.3524 & \bf 0.2131 & \bf 0.3233 & \bf 0.1733 \\
& $\amone$ & 0.0322 & 0.0385 & -0.1477 & -0.0709 & -0.1858 & -0.1442 & -0.3124 & -0.1598 \\
& $\mathsf{Norm}$ & 0.0725 & 0.0235 & \underline{0.1659} & \underline{0.1231} & \underline{0.1990} & \underline{0.1184} & \underline{0.1990} & \underline{0.1476} \\
& $\mathsf{FV}$ & \underline{0.1439} & \underline{0.1003} & -0.1727 & -0.1284 & 0.0676 & 0.0524 & -0.0213 & -0.0625 \\
\midrule
\multirow{5}{*}{MoCo} & $\mathsf{NC}_{100}$ & \bf 0.4157 & \bf 0.3653 & \bf 0.3632 & \bf 0.3128 & 0.3757 & 0.3236 & \bf 0.2667 & \bf 0.2225 \\
& $\amone$ & \underline{0.3523} & \underline{0.3093} & \underline{0.2246} & \underline{0.1925} & \bf 0.3831 & \bf 0.3397 & \underline{0.2453} & \underline{0.2183} \\
& $\mathsf{Norm}$ & 0.3238 & 0.2708 & 0.2186 & 0.1432 & \underline{0.3798} & \underline{0.3294} & 0.2172 & 0.1659 \\
& $\mathsf{FV}$ & 0.0519 & 0.0321 & -0.1420 & -0.1578 & 0.1228 & 0.1122 & -0.1081 & -0.1100 \\
\bottomrule
\end{tabular}
}
\label{tab:comparison}
\end{table*}

\subsection{Correlation between Neighborhood Consistency and Representation Reliability} \label{sec:class_level}

We first evaluate $\mathsf{NC}_{100}$ in capturing the representation reliability of a collection of embedding functions (Definition~\ref{defn::rep^rel}), where the ensemble $\mathcal{H} = \{h_1, \ldots, h_M\}$ is available.

\subsubsection{Correlation to ID Downstream Performance}

Recall that the representation reliability relies on downstream tasks. 
Here we focus on in-distribution (ID) tasks where embedding functions are pre-trained on \texttt{CIFAR-10} or \texttt{CIFAR-100} datasets \citep{krizhevsky2009learning} without using any labeling information. Subsequently, they are fine-tuned with a linear head on the same dataset.

Tables~\ref{tab:comparison} and \ref{tab:comparison_full} shows that our $\mathsf{NC}_{100}$ demonstrates a higher correlation with the representation reliability compared with baselines. Notably, regardless of the choice of pre-training configurations or downstream tasks, our method always has a positive correlation whereas baselines occasionally demonstrate very low or even negative correlation. 
Moreover, $\mathsf{FV}$ presents low correlation and this observation aligns with the theoretical results in Section~\ref{sec:consistency}.

We offer insight into why $\mathsf{NC}_k$ has a more favorable performance compared with baselines. Imagine a test point (a fox image) has unreliable representations. It sits close to a reference point A (a dog image) and far from another point B (a cat image) in one representation space, while the opposite holds true in another representation space. In this case, $\mathsf{NC}_k$ would assign a low score due to this inconsistency. However, $\mathsf{Dist_k}$, and $\mathsf{Norm}$ suggest ``reliable'' since they rely on the relative distance of the test point to the closest reference point or to the origin. $\mathsf{FV}$ may also indicate ``reliable'' if the representations of the test point remain consistent despite falling into clusters of dog or cat images in different representation spaces.
This can happen if A is close to that consistent test point location and B is far from that location for one representation space, while the opposite is true for another representation space.
None of these baselines align different representation spaces before computing these relative distances.

\begin{table*}[t!]
\centering
\fontsize{9pt}{9pt}\selectfont
\caption{Comparison of our neighborhood consistency ($\mathsf{NC}_{100}$) with baseline methods in terms of their correlation with the representation reliability. Here the downstream tasks are transfer learning tasks, where embedding functions are pre-trained on $\texttt{TinyImagenet}$ and fine-tuned on $\texttt{CIFAR-10}$, $\texttt{CIFAR-100}$, and $\texttt{STL-10}$.}
\label{tab:transfer}

\subtable[ResNet-18]{%
\renewcommand{\arraystretch}{0.9}
\begin{tabular}{cc|cccccc}
\toprule
\multirow{3}{*}{\makecell[c]{Pretraining \\ Algorithms}} & \multirow{3}{*}{Method} & \multicolumn{6}{c}{$\texttt{TinyImagenet} \rightarrow$} \\
  & & \multicolumn{2}{c}{\texttt{CIFAR-10}}  & \multicolumn{2}{c}{\texttt{CIFAR-100}} & \multicolumn{2}{c}{\texttt{STL-10}} \\
  & & Entropy  & Brier & Entropy & Brier & Entropy  & Brier \\
\midrule
\multirow{5}{*}{SimCLR} & $\mathsf{NC}_{100}$ & \underline{0.1729} & \underline{0.1292} & \bf 0.1680 & \bf 0.1343 & \bf 0.2176 & \bf 0.1513 \\
& $\amone$ & 0.1311 & 0.0933 & 0.0138 & -0.0251 & 0.1098 & 0.0520 \\
& $\mathsf{Norm}$ & \bf 0.2124 & \bf 0.1642 & \underline{0.1220} & \underline{0.0507} & \underline{0.1641} & \underline{0.0852} \\
& $\mathsf{FV}$ & -0.0919 & -0.0779 & -0.1660 & -0.1668 & -0.1872 & -0.1593 \\
\midrule
\multirow{5}{*}{BYOL} & $\mathsf{NC}_{100}$ & \bf 0.2557 & \bf 0.1431 & \bf 0.2352 & \bf 0.1525 & \underline{0.0374} & \underline{0.0232} \\
& $\amone$ & -0.1458 & -0.0923 & -0.2716 & -0.1525 & -0.2893 & -0.2178 \\
& $\mathsf{Norm}$ & \underline{0.1483} & \underline{0.0749} & \underline{0.1879} & \underline{0.1315} & \bf 0.1021 & \bf 0.0772 \\
& $\mathsf{FV}$ & -0.1386 & -0.0713 & -0.1874 & -0.1202 & -0.1171 & -0.0790 \\
\midrule
\multirow{5}{*}{MoCo} & $\mathsf{NC}_{100}$ & \underline{0.1880} & \underline{0.1313} & \bf 0.1797 & \bf 0.1446 & \bf 0.2263 & \bf 0.1463 \\
& $\amone$ & 0.1213 & 0.0718 & 0.0311 & -0.0143 & 0.0993 & 0.0223 \\
& $\mathsf{Norm}$ & \bf 0.1930 & \bf 0.1355 & \underline{0.1321} & \underline{0.0531} & \underline{0.1327} & \underline{0.0443} \\
& $\mathsf{FV}$ & -0.0793 & -0.0701 & -0.1499 & -0.1658 & -0.1615 & -0.1577 \\
\bottomrule
\end{tabular}
\label{tab:transfer_r18}
}


\subtable[ResNet-50]{%
\renewcommand{\arraystretch}{0.9}
\begin{tabular}{cc|cccccc}
\toprule
\multirow{3}{*}{\makecell[c]{Pretraining \\ Algorithms}} & \multirow{3}{*}{Method} & \multicolumn{6}{c}{$\texttt{TinyImagenet} \rightarrow$} \\
  & & \multicolumn{2}{c}{\texttt{CIFAR-10}}  & \multicolumn{2}{c}{\texttt{CIFAR-100}} & \multicolumn{2}{c}{\texttt{STL-10}} \\
  & & Entropy  & Brier & Entropy & Brier & Entropy  & Brier \\
\midrule
\multirow{5}{*}{SimCLR} & $\mathsf{NC}_{100}$ & \underline{0.1224} & \underline{0.0804} & \bf 0.1467 & \bf 0.1194 & \bf 0.1866 & \bf 0.1169 \\
& $\amone$ & 0.0894 & 0.0531 & -0.0246 & -0.0627 & 0.0340 & -0.0167 \\
& $\mathsf{Norm}$ & \bf 0.1549 & \bf 0.1098 & \underline{0.0895} & \underline{0.0177} & \underline{0.0608} & \underline{-0.0056} \\
& $\mathsf{FV}$ & -0.1078 & -0.0878 & -0.1867 & -0.1840 & -0.2178 & -0.1829 \\
\midrule
\multirow{5}{*}{BYOL} & $\mathsf{NC}_{100}$ & \bf 0.2561 & \bf 0.1258 & \bf 0.2600 & \bf 0.1005 & \bf 0.1410 & 0.0123 \\
& $\amone$ & -0.1885 & -0.1274 & -0.3075 & -0.1679 & -0.2371 & -0.1925 \\
& $\mathsf{Norm}$ & \underline{0.1620} & \underline{0.0918} & \underline{0.1082} & \underline{0.0816} & \underline{0.0963} & \bf 0.0380 \\
& $\mathsf{FV}$ & -0.1668 & -0.0737 & -0.1441 & -0.0705 & -0.0584 & \underline{0.0205} \\
\midrule
\multirow{5}{*}{MoCo} & $\mathsf{NC}_{100}$ & \underline{0.1927} & \underline{0.1300} & \bf 0.1962 & \bf 0.1560 & \underline{0.2170} & \underline{0.1433} \\
& $\amone$ & 0.1751 & 0.1109 & 0.0759 & 0.0257 & 0.1696 & 0.0710 \\
& $\mathsf{Norm}$ & \bf 0.2503 & \bf 0.1803 & \underline{0.1838} & \underline{0.1034} & \bf 0.2524 & \bf 0.1524 \\
& $\mathsf{FV}$ & -0.0560 & -0.0400 & -0.1510 & -0.1578 & -0.1735 & -0.1577 \\
\bottomrule
\end{tabular}
\label{tab:transfer_r50}
}

\end{table*}
\subsubsection{Correlation to Transfer Learning Performance}

We conduct experiments with out-of-distribution tasks, particularly focusing on transfer learning tasks. We use the \texttt{TinyImagenet} dataset \citep{le2015tiny} as the source dataset for pre-training the embedding functions. Subsequently, we fine-tune and evaluate the embedding functions on three target datasets: \texttt{CIFAR-10}, \texttt{CIFAR-100}, and \texttt{STL-10} \citep{coates2011analysis}.
The correlation between each method and the representation reliability is reported in Tables~\ref{tab:transfer_r18}, \ref{tab:transfer_r50}, and \ref{tab:trasnfer_full}. 
As observed, NC consistently captures the representation reliability across a diverse range of settings.

We observe that $\mathsf{Norm}$ achieves comparable performance with NC. One possible explanation is that several factors (e.g., properties of embedding functions and characteristics of test points) influence the reliability of a test point's representation. In the context of transfer learning, OOD test points are more likely to be unreliable. The baseline we selected, $\mathsf{Norm}$, is good at identifying OOD samples \citep{tack2020csi}, thereby enabling it to capture representation reliability to some extent.

\subsubsection{Ablation Studies} \label{sec:ablation}

\paragraph{Performance with a Smaller Ensemble Size ($M$).}
We conduct an ablation study to investigate the impact of ensemble size $M$ on estimating the representation reliability. The results in Figures~\ref{fig:n_ens18} and ~\ref{fig:n_ens50} and Appendix~\ref{app:suppl} show that increasing ensemble size improves the correlation scores. Nonetheless, our method generally outperforms baseline approaches, even with a small ensemble size of $M = 2$.
Exploring cost-effective ensemble construction methods, such as low-rank approximations \citep{wen2020batchensemble} or stochastic weight averaging \citep{izmailov2018averaging, maddox2019simple}, can be a promising future direction.

\paragraph{Trade-off on the Number of Neighbors ($k$).}
As discussed in Section~\ref{sec:nb_consistency}, the choice of $k$ in Equation~\eqref{defn::NC} leads to a trade-off between having more consistent neighbors and preserving the overall reliability of those neighbors. In order to explore this trade-off, we conduct experiments with different values of $k \in \{1, 2, 5, 10, 20, 50, 100, 200, 500, 1000\}$. The correlation between our method and the representation reliability is illustrated in Figures~\ref{fig:k18} and \ref{fig:k50} and Appendix~\ref{app:suppl}: it initially increases and then decreases as expected. We observe that the optimal performance could be achieved with $k$ around 100 (i.e., 2\% of $|\bX_{\text{ref}}|=5000$) for $\mathsf{NC}_k$ across different pre-training algorithms, model architectures, and downstream data.

\begin{figure}[ht!] 
    \centering
    \subfigure[]
    {
        \label{fig:n_ens18}
        \includegraphics[width=0.43\columnwidth]{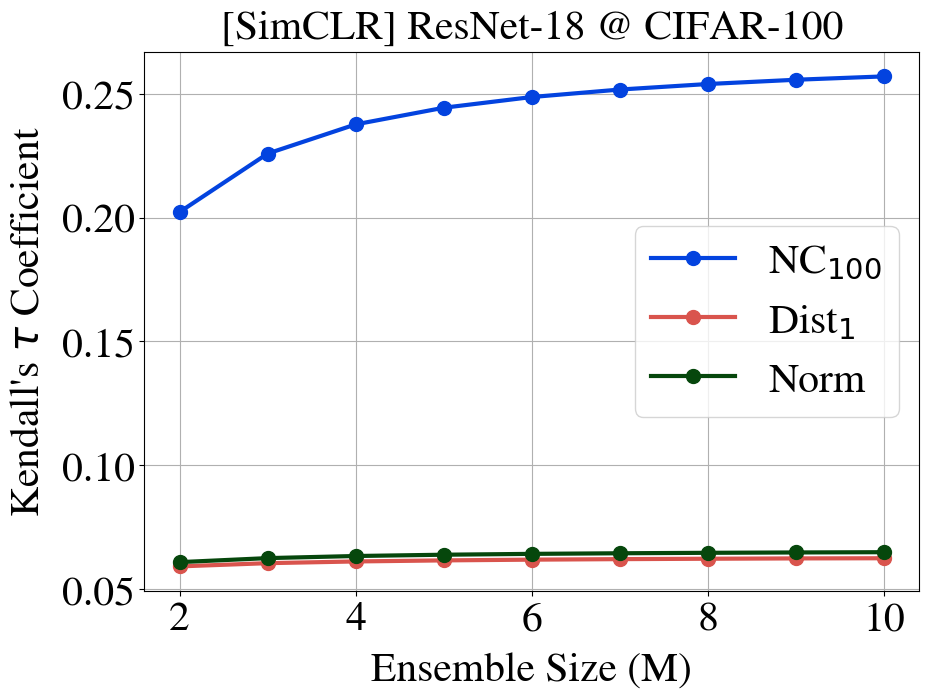}
    }
    \hspace{1em}
    \subfigure[]
    {
        \label{fig:n_ens50}
        \includegraphics[width=0.43\columnwidth]{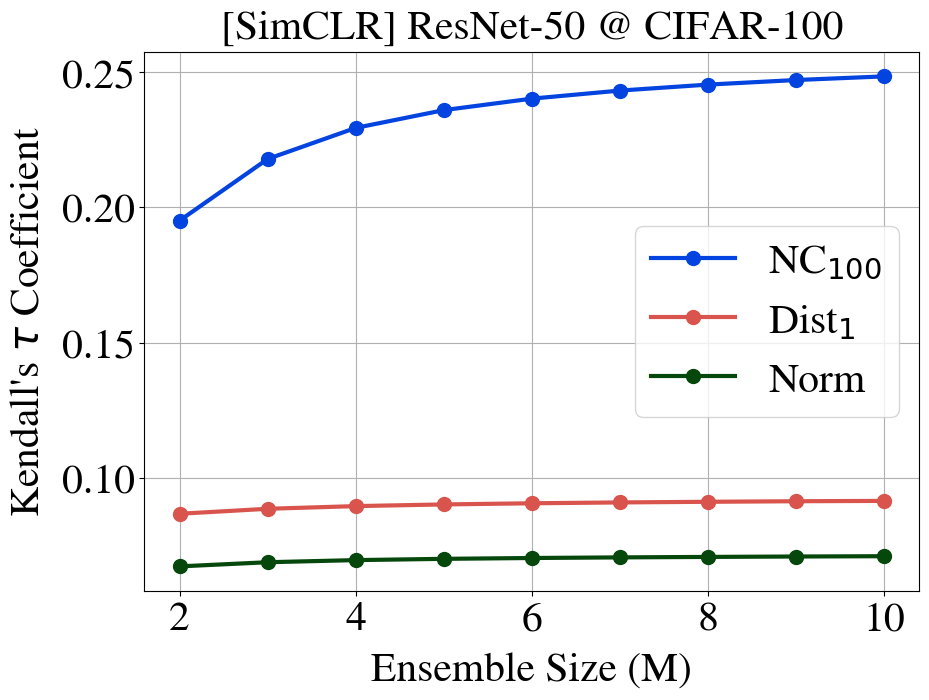}
    }
    \subfigure[]
    {   
        \label{fig:k18}
        \includegraphics[width=0.43\columnwidth]{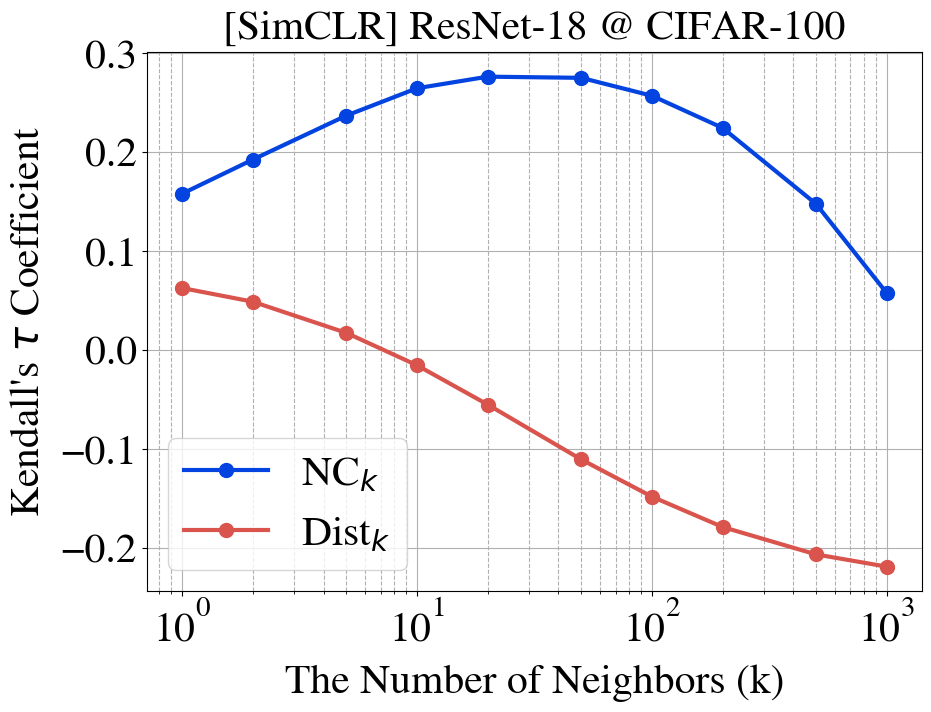}
    }
    \hspace{1em}
    \subfigure[]
    {   
        \label{fig:k50}
        \includegraphics[width=0.43\columnwidth]{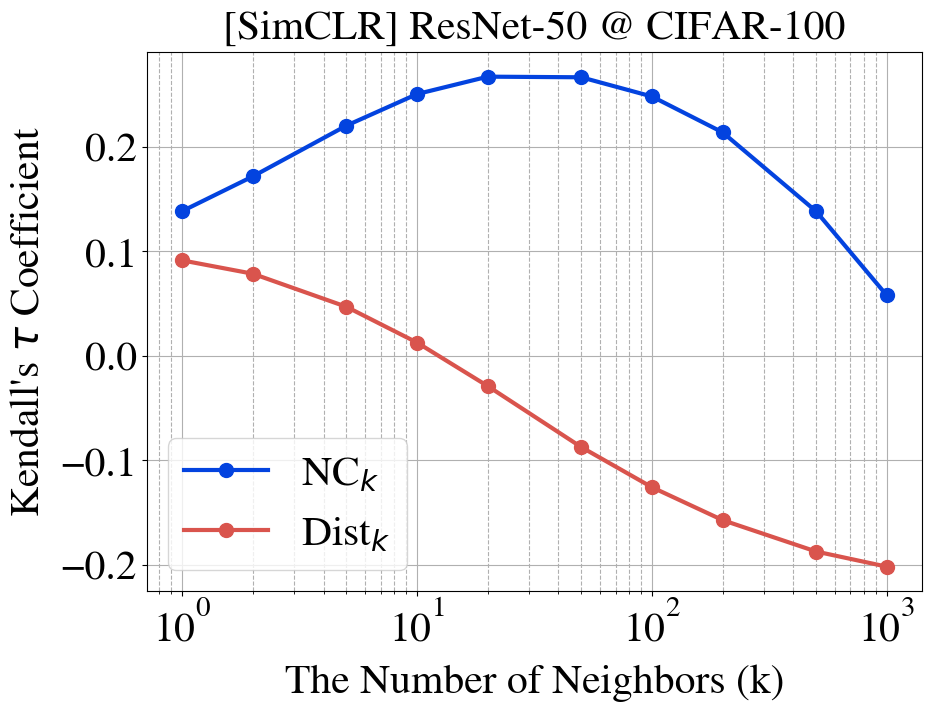}
    }
    \subfigure[]
    {   
        \label{fig:dist_metric_id}
        \includegraphics[width=0.43\columnwidth]{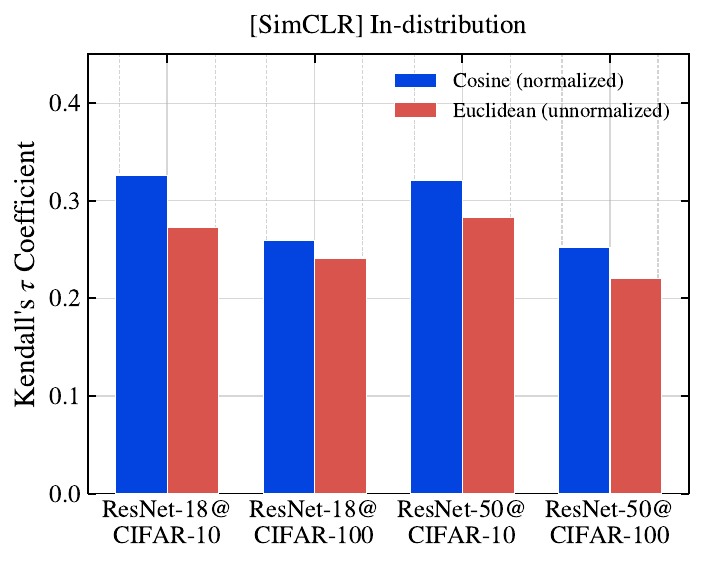}
    }
    \hspace{1em}
    \subfigure[]
    {   
        \label{fig:dist_metric_transfer}
        \includegraphics[width=0.43\columnwidth]{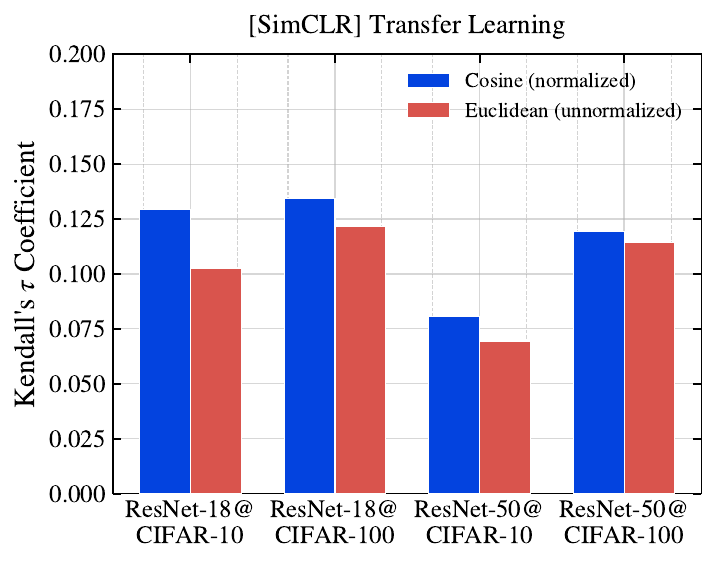}
    }
    \caption{Ablation studies on the (a)--(b) ensemble size ($M$), (c)--(d) the number of neighbors ($k$) for $\mathsf{NC}_k$ (ours) and baselines, and (e)--(f) the robustness to the choice of distance metric for $\mathsf{NC}_k$. Brier score is used for the downstream performance metric. The comprehensive results can be found in Appendix~\ref{app:suppl}.}
    \label{fig:ablation}
\end{figure}

\paragraph{Robustness to the Choice of Distance Metric.} 
We explore how various distance metrics within the representation space affect both our method and the baseline approaches. 
We present the results in Figures~\ref{fig:dist_metric_id} and \ref{fig:dist_metric_transfer} and Appendix~\ref{app:suppl}. Our key observation is: the baselines are sensitive to the choice of distance metric and may even exhibit negative correlations with the representation reliability. In contrast, $\mathsf{NC}_{100}$ consistently demonstrates a positive correlation and ranks within the top $2$ among all baseline methods, regardless of the distance metric chosen.


\subsection{Quantifying the Reliability of Individual Embedding Functions} \label{sec:individual_exp}

We now turn to a more practical question: can we assess the reliability of an \emph{individual} embedding function $h$? In this case, we can apply our $\mathsf{PNC}_{100}^s$ (Section~\ref{sec:pnc:desc}), which constructs a synthetic ensemble by perturbing the parameters of the given model. We use $M=10$ perturbed models with isotropic Gaussian noise and select the perturbation magnitude $\sigma_{\text{pert}}$ via $\mathsf{PNC}$-spread tuning. For the tuning, we run $\mathsf{PNC}_{100}$ on perturbation levels $\sigma_{pert}$ from the set
\[\{10^{(x-21)/4}:x\in\{1,\ldots,21\}\},\]
on the pretraining set and choose the $\sigma_{pert}$ that maximizes the standard deviation of the scores. This set is essentially from $10^{-5}$ to $1.0$ with each subsequent perturbation level being scaled up by a factor of $10^{1/4}$.

We follow the evaluation protocol from Section~\ref{sec:setup}, measuring the \kendall between method scores and downstream performance scores. Similar to the case of evaluating an ensemble of embedding functions, we focus on comparing with the performance scores of downstream OVO classification tasks. A more comprehensive results on a comparison with the performance scores of downstream multi-class classification tasks will be discussed in Appendix \ref{app:individual_exp_multi}. For each combination of SSL algorithm, architecture, and dataset, we evaluate $M=10$ individually pre-trained models.

\subsubsection{Results (PNC)} \label{sec:individual_pnc}

\begin{table*}[t!]
\centering
\fontsize{9pt}{9pt}\selectfont
\caption{Quantifying the reliability of individual embedding functions without retraining (single model only). \kendall correlation between each method and the downstream performance of individual models on in-distribution binary classification tasks. The baselines ($\amone$, $\mathsf{Norm}$) are computed from the single model without ensemble averaging. The highest and second-highest scores are highlighted in bold and underlined, respectively.}

\resizebox{\textwidth}{!}{%
\renewcommand{\arraystretch}{1.1}
\begin{tabular}{cc|cccc|cccc}
\toprule
\multirow{3}{*}{\makecell[c]{Pretraining \\ Algorithms}} & \multirow{3}{*}{Method} & \multicolumn{4}{c|}{ResNet-18} & \multicolumn{4}{c}{ResNet-50}  \\
\cline{3-10} 
 & & \multicolumn{2}{c}{\texttt{CIFAR-10}}  & \multicolumn{2}{c|}{\texttt{CIFAR-100}} & \multicolumn{2}{c}{\texttt{CIFAR-10}}  & \multicolumn{2}{c}{\texttt{CIFAR-100}} \\
 & & Entropy  & Brier & Entropy & Brier & Entropy  & Brier & Entropy & Brier \\
\midrule
\multirow{3}{*}{SimCLR} & $\mathsf{PNC}_{100}^s$ & \bf 0.3262 & \bf 0.2779 & \bf 0.2175 & \bf 0.1724 & \bf 0.2961 & \bf 0.2525 & \bf 0.2396 & \bf 0.1954\\
& $\mathsf{Dist}_1$ & 0.2383 & 0.2058 & 0.0580 & 0.0572 & 0.2351 & 0.2049 & 0.0973 & \underline{0.0838}\\
& $\mathsf{Norm}$ & \underline{0.2668} & \underline{0.2157} & \underline{0.1242} & \underline{0.0620} & \underline{0.2629} & \underline{0.2156} & \underline{0.1370} & 0.0690\\
\midrule
\multirow{3}{*}{BYOL} & $\mathsf{PNC}_{100}^s$ & \bf 0.2306 & \bf 0.1361 & \bf 0.2184 & \bf 0.1364 & \bf 0.3473 & \bf 0.1764 & \bf 0.3130 & \bf 0.1231\\
& $\mathsf{Dist}_1$ & -0.1205 & -0.0430 & -0.1892 & -0.0584 & -0.0964 & \underline{0.0056} & -0.2298 & -0.0688\\
& $\mathsf{Norm}$ & \underline{0.0313} & \underline{0.0207} & \underline{0.0921} & \underline{0.0989} & \underline{-0.0866} & -0.0411 & \underline{0.0450} & \underline{0.0720}\\
\midrule
\multirow{3}{*}{MoCo} & $\mathsf{PNC}_{100}^s$ & \bf 0.3619 & \bf 0.3156 & \bf 0.3000 & \bf 0.2558 & 0.3214 & 0.2749 & \underline{0.2348} & \bf 0.1995\\
& $\mathsf{Dist}_1$ & \underline{0.3153} & \underline{0.2762} & 0.2072 & \underline{0.1767} & \underline{0.3480} & \bf 0.3063 & 0.2336 & \underline{0.1992}\\
& $\mathsf{Norm}$ & 0.3118 & 0.2532 & \underline{0.2214} & 0.1368 & \bf 0.3490 & \underline{0.2947} & \bf 0.2557 & 0.1749\\
\bottomrule
\end{tabular}
}
\label{tab:individual_single_indist}
\end{table*}

\begin{table*}[t!]
\centering
\fontsize{9pt}{9pt}\selectfont
\caption{Quantifying the reliability of individual embedding functions without retraining (single model only). \kendall correlation between each method and the downstream performance on transfer learning binary classification tasks. Embedding functions are pre-trained on \texttt{TinyImagenet} and fine-tuned on \texttt{CIFAR-10}, \texttt{CIFAR-100}, and \texttt{STL-10}.}
\label{tab:individual_single_transfer}

\subtable[ResNet-18]{%
\renewcommand{\arraystretch}{0.9}
\begin{tabular}{cc|cccccc}
\toprule
\multirow{3}{*}{\makecell[c]{Pretraining \\ Algorithms}} & \multirow{3}{*}{Method} & \multicolumn{6}{c}{$\texttt{TinyImagenet} \rightarrow$} \\
  & & \multicolumn{2}{c}{\texttt{CIFAR-10}}  & \multicolumn{2}{c}{\texttt{CIFAR-100}} & \multicolumn{2}{c}{\texttt{STL-10}} \\
  & & Entropy  & Brier & Entropy & Brier & Entropy  & Brier \\
\midrule
\multirow{3}{*}{SimCLR} & $\mathsf{PNC}_{100}^s$ & \underline{0.1810} & \underline{0.1395} & \bf 0.1958 & \bf 0.1472 & \bf 0.2472 & \bf 0.1745\\
& $\mathsf{Dist}_1$ & 0.1171 & 0.0836 & 0.0324 & -0.0105 & 0.1017 & 0.0491\\
& $\mathsf{Norm}$ & \bf 0.2043 & \bf 0.1541 & \underline{0.1399} & \underline{0.0580} & \underline{0.1654} & \underline{0.0846}\\
\midrule
\multirow{3}{*}{BYOL} & $\mathsf{PNC}_{100}^s$ & \bf 0.1998 & \bf 0.1155 & \bf 0.2092 & \bf 0.1415 & \underline{0.0465} & \underline{0.0380}\\
& $\mathsf{Dist}_1$ & -0.1546 & -0.0674 & -0.2653 & -0.1377 & -0.2956 & -0.1801\\
& $\mathsf{Norm}$ & \underline{0.1383} & \underline{0.0750} & \underline{0.1664} & \underline{0.1251} & \bf 0.0833 & \bf 0.0717\\
\midrule
\multirow{3}{*}{MoCo} & $\mathsf{PNC}_{100}^s$ & \bf 0.1954 & \bf 0.1399 & \bf 0.2169 & \bf 0.1753 & \bf 0.2491 & \bf 0.1668\\
& $\mathsf{Dist}_1$ & 0.1123 & 0.0666 & 0.0515 & 0.0064 & 0.0952 & 0.0233\\
& $\mathsf{Norm}$ & \underline{0.1913} & \underline{0.1312} & \underline{0.1467} & \underline{0.0730} & \underline{0.1400} & \underline{0.0483}\\
\bottomrule
\end{tabular}
\label{tab:individual_single_transfer_r18}
}

\subtable[ResNet-50]{%
\renewcommand{\arraystretch}{0.9}
\begin{tabular}{cc|cccccc}
\toprule
\multirow{3}{*}{\makecell[c]{Pretraining \\ Algorithms}} & \multirow{3}{*}{Method} & \multicolumn{6}{c}{$\texttt{TinyImagenet} \rightarrow$} \\
  & & \multicolumn{2}{c}{\texttt{CIFAR-10}}  & \multicolumn{2}{c}{\texttt{CIFAR-100}} & \multicolumn{2}{c}{\texttt{STL-10}} \\
  & & Entropy  & Brier & Entropy & Brier & Entropy  & Brier \\
\midrule
\multirow{3}{*}{SimCLR} & $\mathsf{PNC}_{100}^s$ & \underline{0.1262} & \underline{0.0884} & \bf 0.1672 & \bf 0.1252 & \bf 0.2254 & \bf 0.1358\\
& $\mathsf{Dist}_1$ & 0.0825 & 0.0488 & -0.0003 & -0.0435 & 0.0349 & -0.0134\\
& $\mathsf{Norm}$ & \bf 0.1543 & \bf 0.1064 & \underline{0.1159} & \underline{0.0356} & \underline{0.0704} & \underline{0.0012}\\
\midrule
\multirow{3}{*}{BYOL} & $\mathsf{PNC}_{100}^s$ & \bf 0.3044 & \bf 0.1461 & \bf 0.3367 & \bf 0.1609 & \bf 0.2378 & \bf 0.0891\\
& $\mathsf{Dist}_1$ & -0.1924 & -0.0742 & -0.3197 & -0.1508 & -0.2491 & -0.1132\\
& $\mathsf{Norm}$ & \underline{0.1377} & \underline{0.0736} & \underline{0.0880} & \underline{0.0709} & \underline{0.0678} & \underline{0.0238}\\
\midrule
\multirow{3}{*}{MoCo} & $\mathsf{PNC}_{100}^s$ & 0.1619 & \underline{0.1078} & \bf 0.2056 & \bf 0.1664 & \underline{0.2245} & \bf 0.1500\\
& $\mathsf{Dist}_1$ & \underline{0.1666} & 0.1048 & 0.1088 & 0.0555 & 0.1639 & 0.0688\\
& $\mathsf{Norm}$ & \bf 0.2464 & \bf 0.1732 & \underline{0.2035} & \underline{0.1306} & \bf 0.2491 & \underline{0.1455}\\
\bottomrule
\end{tabular}
\label{tab:individual_single_transfer_r50}
}

\end{table*}

Tables~\ref{tab:individual_single_indist} and~\ref{tab:individual_single_transfer} present the analogous comparison in the single-model setting. $\mathsf{PNC}_{100}^s$ often outperforms the non-ensemble baselines ($\amone$ and $\mathsf{Norm}$) in the majority of configurations for both in-distribution and transfer learning tasks. As with $\mathsf{NC}_{100}$, all correlation scores for $\mathsf{PNC}_{100}^s$ remain positive, confirming that the neighborhood consistency framework consistently and effectively captures representation reliability even with a synthetic ensemble.

\subsubsection{Efficacy of the Proposed PNC-Spread Tuning Strategy} \label{sec:pnc_tuning}

In this section, we validate the $\mathsf{PNC}$-spread tuning strategy introduced in Section~\ref{sec:pnc:desc}. We vary the perturbation magnitude $\sigma_{\text{pert}}$ and compare the \kendall correlation between the resulting $\mathsf{PNC}_{100}$ scores and the
downstream Brier scores
against the standard deviation of $\mathsf{PNC}_{100}$ scores on the pre-training set.
Figure~\ref{fig:stddev_tuning_indist_binary_brier_simclr} shows that the peak correlation aligns closely with the peak spread across all configurations, supporting the effectiveness of the spread-based tuning criterion. As noted in Section~\ref{sec:pnc:desc}, the correlation diminishes rapidly for both very small $\sigma_{\text{pert}}$ (where the perturbed models are nearly identical) and very large $\sigma_{\text{pert}}$ (where neighborhoods become random). 
A comprehensive comparison of oracle vs.\ spread-tuned $\mathsf{PNC}$ is reported in Appendix~\ref{app:pnc_oracle}, along with the complete set of tuning plots for BYOL and MoCo.
\begin{figure}[t] 
    \centering
    \subfigure[]{
        \includegraphics[width=0.47\columnwidth]{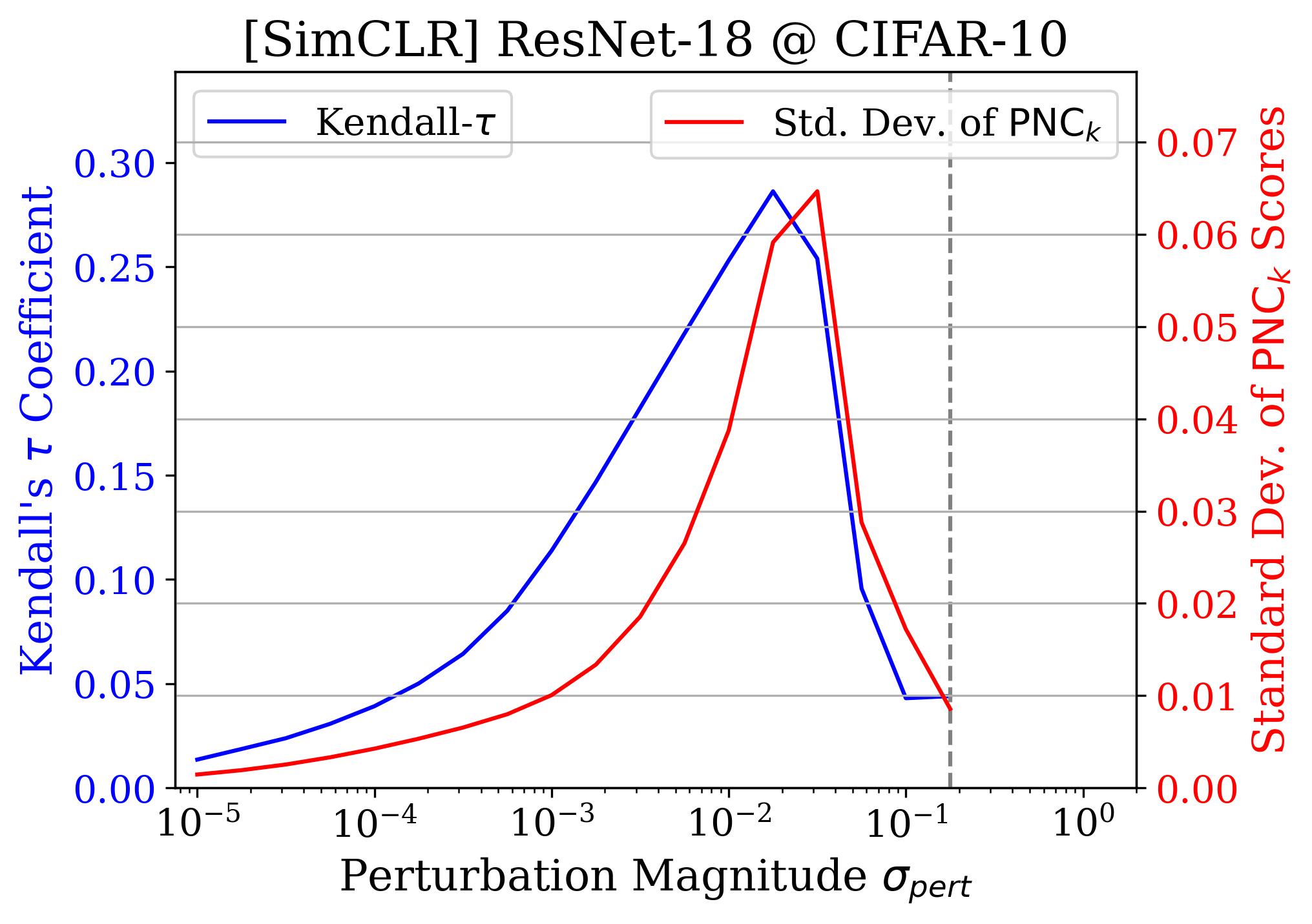}
    }
    \subfigure[]{   
        \includegraphics[width=0.47\columnwidth]{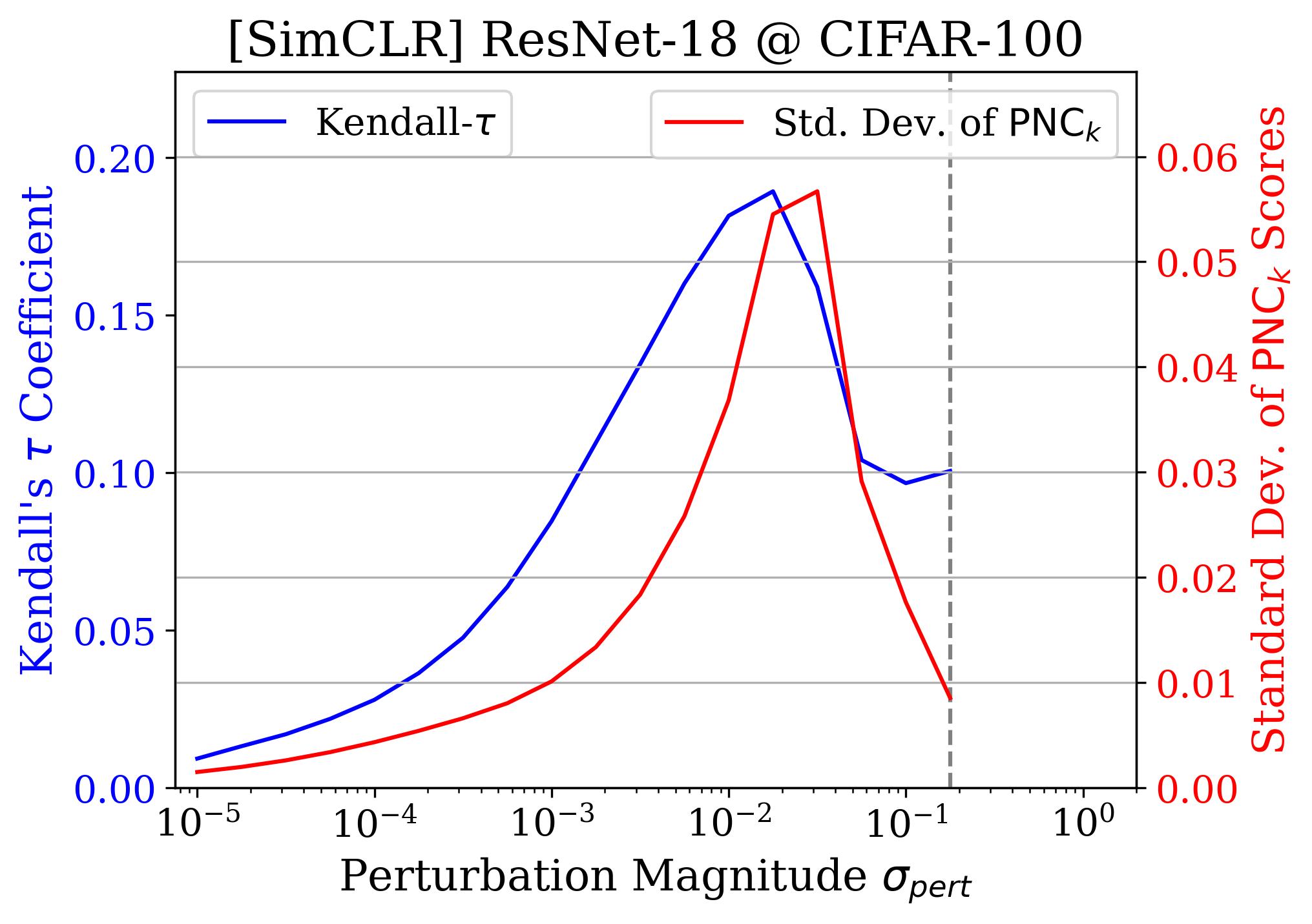}
    }
    \subfigure[]{   
        \includegraphics[width=0.47\columnwidth]{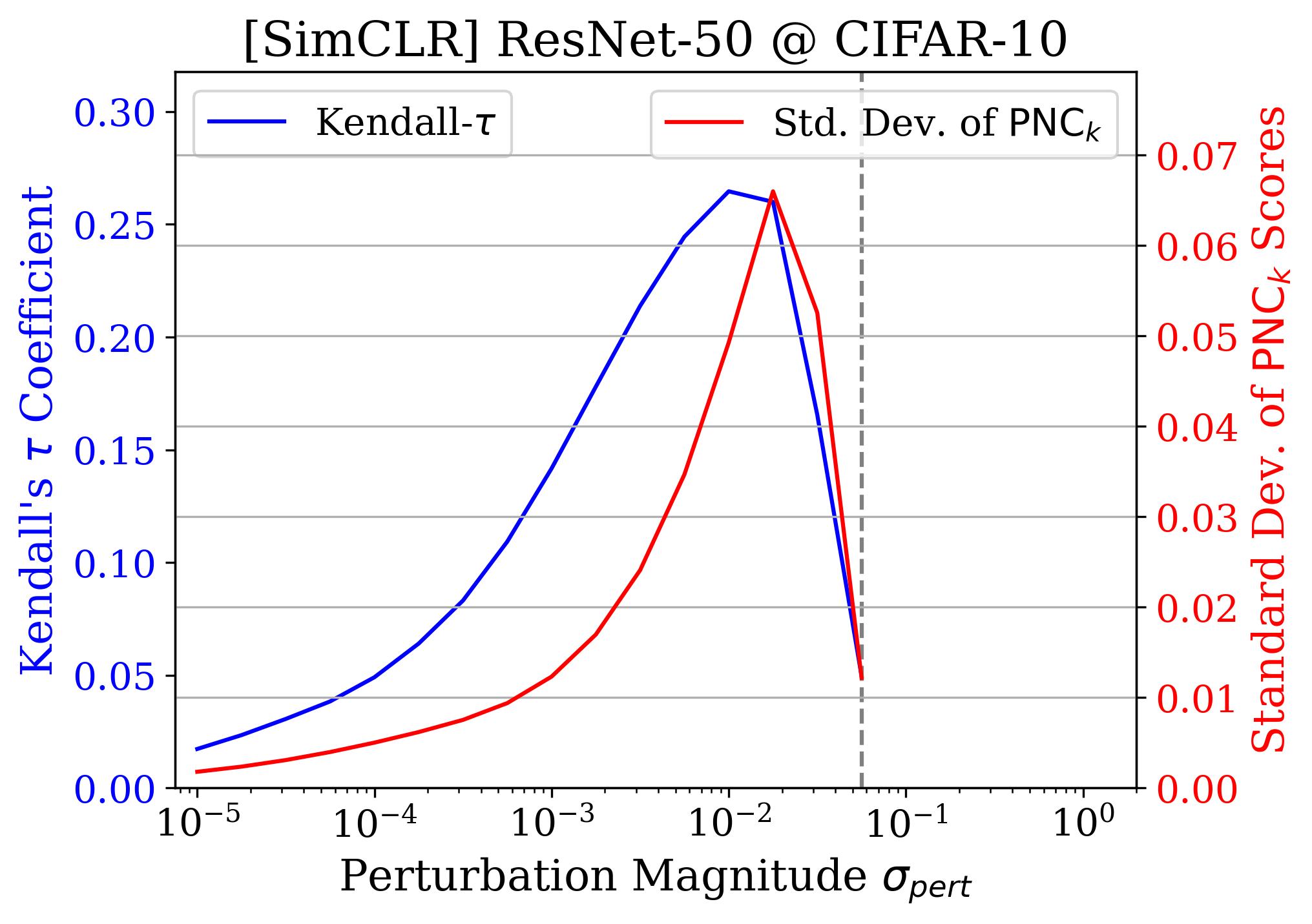}
    }
    \subfigure[]{
        \includegraphics[width=0.47\columnwidth]{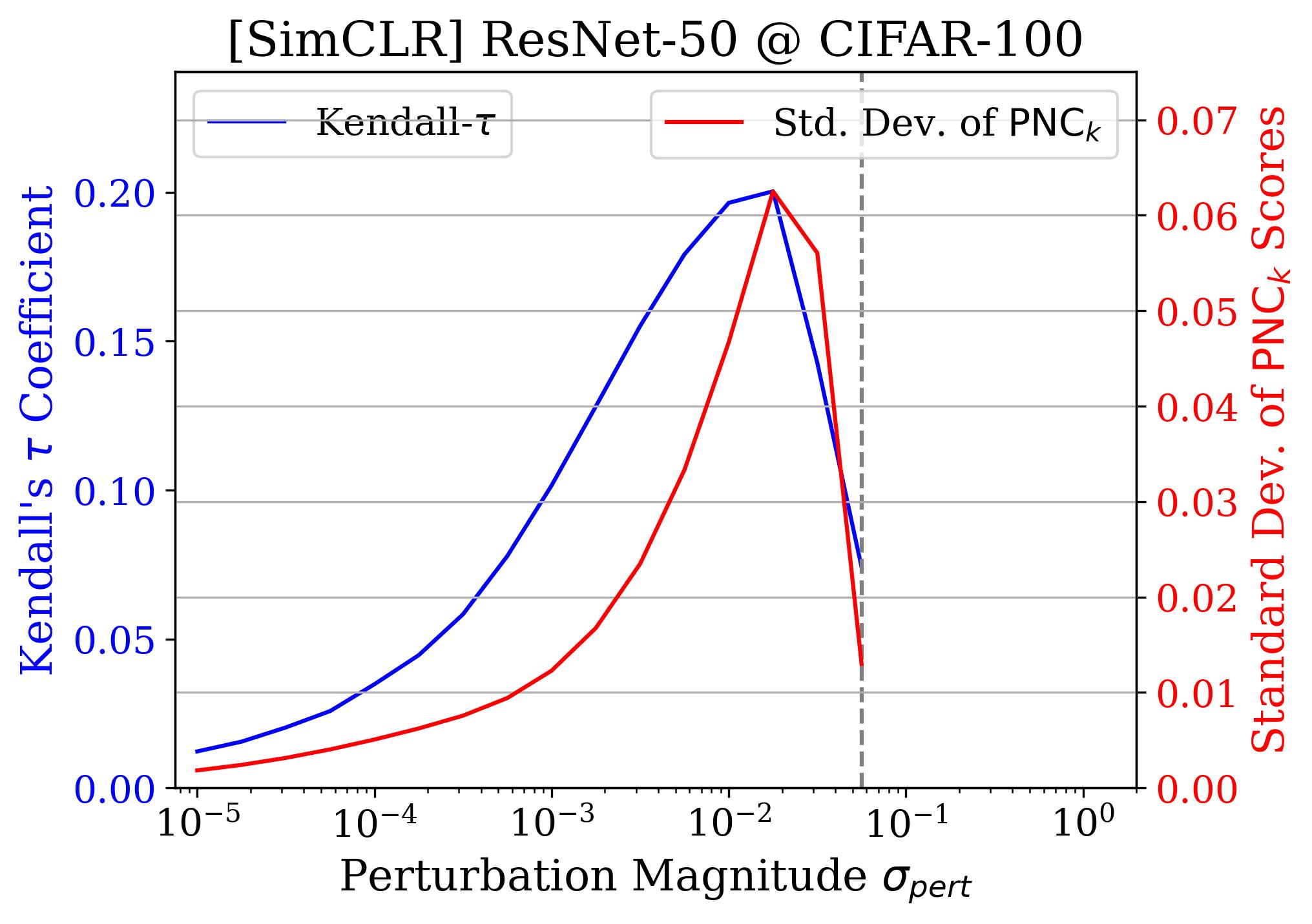}
    }
    \caption{Validation of $\mathsf{PNC}$-spread tuning for SimCLR on in-distribution binary classification tasks. Blue: \kendall between $\mathsf{PNC}_{100}$ scores and downstream Brier scores as a function of $\sigma_{\text{pert}}$. Red: standard deviation of $\mathsf{PNC}_{100}$ scores on $\bX_{\text{ref}}$. The peak correlation aligns with the peak spread, supporting the spread-based tuning criterion. Dashed gray line: largest considered $\sigma_{pert}$ that did not cause a floating-point overflow.}
    \label{fig:stddev_tuning_indist_binary_brier_simclr}
\end{figure}

\subsection{Key Observations and Takeaways}

In summary, our main findings from the experiments are:
\begin{itemize}[leftmargin=1.0em]
    \item $\mathsf{NC}_{k}$ effectively captures the representation reliability and demonstrates strong correlations with downstream performance. When evaluating the reliability of the individual model, $\mathsf{PNC}_{100}^s$ serves as an effective alternative.
    
    \item More importantly, contrary to the baselines, $\mathsf{NC}_{100}$ and $\mathsf{PNC}_{100}^s$ \emph{consistently} exhibit positive correlations with the representation reliability across all settings. Although the baselines may occasionally surpass our methods in individual configurations, their performance fluctuates significantly and sometimes becomes negative, introducing risk when used to assess reliability in safety-critical applications.

    \item The proposed $\mathsf{PNC}$-spread tuning strategy provides a practical, label-free mechanism for selecting the perturbation magnitude, achieving performance close to the oracle across all tested configurations.
\end{itemize}


\section{Conclusion, Limitations, and Future Works}

General-purpose embedding functions, often obtained via self-supervised learning, can be adapted to various downstream tasks. In this paper, we presented a systematic study to evaluate the quality of representations assigned by the embedding functions. We introduced a mathematical definition of representation reliability, demonstrated that existing UQ frameworks in supervised learning cannot be directly applied to estimate uncertainty in representation spaces, and derived an estimation method for the representation reliability of ensembles based on the concept of neighborhood consistency (NC). We further extended our framework to the single-model setting through perturbation-based neighborhood consistency (PNC), which constructs a synthetic ensemble via parameter perturbation, removing the need for an ensemble of models. We validated both NC and PNC through extensive numerical experiments.

This work takes a step towards understanding the reliability and uncertainty of abstract representations. In the case where a downstream model fails to deliver a desirable output for a test point, the neighborhood consistency can provide valuable insight into whether the mistake is due to unreliable representations or downstream heads.
Future work may concern investigating the trustworthiness of representations in terms of robustness, privacy, or interpretability.
Furthermore, while PNC addresses the single-model setting with simple isotropic Gaussian perturbations, more structured perturbation strategies may further improve the efficiency and discriminative power of the synthetic ensemble. Our preliminary experiments with gradient-informed perturbation distributions (Appendix~\ref{app:alt_perturbation}) suggest that perturbations along high-sensitivity directions can improve performance when loss function access is available.



\vskip 0.2in
\bibliography{references}

\appendix
\section{Further Discussion on Related Work} 
\label{app:summary}

Self-supervised representation learning has become standard in many computer vision applications.
Instead of training a neural network that takes in the raw data and outputs the target value (e.g., class label), it optimizes a neural network $h_\theta$ that maps an input $\bx$ into the latent vector $\bz \in \mathbb{R}^{\rdim}$ in the $\rdim$-dimensional representation space.

Uncertainty estimation is the process of figuring out how uncertain or reliable the learned representations of the data are.
Assessing the uncertainty of the neural network's representation is a key step in making a reliable machine learning framework.
This is because the uncertainty provides information about the data and how confident the model is in its modeling.
There are several ways to estimate uncertainty in deep learning, such as Bayesian approaches and ensembling, in supervised learning settings where the ground truth output (e.g., label) is given.
Estimating uncertainty in deep representation learning, on the other hand, is still a relatively undiscovered area of research.

\subsection{Uncertainty-aware Representation Learning}
The representation model is often considered to be deterministic in recently popular frameworks \citep{chen2020simple, he2020momentum, chen2020improved}.
To address the reliability issue of those frameworks, some recent works, including \citet{oh2018modeling, wu2020simple}, extend the prior deterministic frameworks to stochastic ones, allowing for the construction of an uncertainty-aware self-Supervised representation learning framework.

\citet{oh2018modeling} introduces a hedged instance embedding (HIB) that optimizes a representation network that approximates the distribution over the representation vector $p_\theta(\bz|\bx)$ under the (soft) contrastive loss \citep{hadsell2006dimensionality} and variational information bottleneck (VIB) principle \citep{alemi2016deep, achille2017emergence}.
More specifically, HIB encoder parameterizes the distribution as the mixture of $C$ Gaussians: $p_\theta(\bz|\bx) = \sum_{c=1}^{C} \mathcal{N}(\bz; \mu_\theta(\bx,c), \Sigma_\theta(\bx,c))$.
Based on the stochastic embedding, the paper proposes an uncertainty metric, called \emph{self-mismatch} probability:
\begin{equation}
    s_{self\_mismatch}(\bx^*) \defined 1 - p(m | \bx^*, \bx^*)
\end{equation}
where $p(m | \bx_1, \bx_2) \approx \int p(m | \bz_1, \bz_2)p_\theta(\bz_1 | \bx_1)p_\theta(\bz_2 | \bx_2) dz_1 d\bz_2$ and $p(m|\bz_1, \bz_2) \defined \sigma(-a \| \bz_1 - \bz_2\|_2 + b)$ based on their contrastive learning method.
Self-mismatch probability can be interpreted as an expectation of the distance between two points randomly sampled from the output distribution.
In other words, this uncertainty metric is based on the idea that an input with a large aleatoric uncertainty will span a wider region, resulting in a smaller $p(m|x,x)$.

\citet{wu2020simple} proposed a similar extension.
Their paper introduces a distribution encoder that outputs the representation of Gaussian distribution with diagonal covariance matrix $\Sigma_\theta (x)$ and extends the normalized temperature-scaled cross-entropy loss (NT-Xent) \citep{chen2020simple} to distribution-level contrastive objective.
The norm of the covariance matrix determined by the distribution encoder is used to assess the reliability of a given input:
\begin{equation}
    s_{var}(x^*) \defined || \Sigma_\theta (x^*)||
\end{equation}

Despite the benefits of stochastic representation, there are still some shortcomings.
One limitation would be that it requires re-training.
Large models that have gotten a lot of attention lately are usually trained on a lot of data and are getting bigger, which means they take more time and computing power to train.
As a result, it may not always be practical or feasible for users to re-train a model.
Additionally, new training schemes can impose unexpected inductive bias to algorithms that are already working well.
For example, most probability-based methods are based on standard distributions like the Gaussian or a mixture of them.
However, these assumptions may reduce the effectiveness of the model or slow down the training procedure.

\subsection{Novelty Detection in Representation Space}

We extend our discussion in Section \ref{sec:bg_novelty} and show other OOD detection methods.
\citet{lee2018simple} and \citet{van2020uncertainty} construct reference points by empirical class means:
\begin{equation}
\Big\{ \boldsymbol{\hat{\mu}}_c = \frac{1}{N_c}\sum_{\bx_i \in \bX_{\text{ref}}}h_\theta(\bx_i) \mathbbm{1}(y_i=c) \Big\}_{c=1}^C
\end{equation}
where $\boldsymbol{\hat{\mu}}_c$ is a centroid of training representations of which the label $y_i$ is equal to $c$, $N_c$ is the number of training instances belonging to the label $c$, and $C$ is the number of classes.
Then, \citet{lee2018simple} defines an uncertainty score as a minimum Mahalanobis distance to each of the centroids using a tied empirical covariance matrix, whereas \citet{van2020uncertainty} calculates a distance using a Radial Basis Function (RBF) kernel.
These approaches can also be viewed as estimating the probability density of the representation to each class, but they do necessitate training labels.


The main limitation of the above approaches, along with the approach in Section \ref{sec:bg_novelty}, is the lack of theoretical justification for the proposed metric, as it is founded upon heuristic rules.
For example, as demonstrated by our empirical studies in Section~\ref{app:unnormalized}, a wrong selection of the distance metric can lead to contradictory results.
In addition, as depicted in Figure~\ref{fig:abl_k_full}, using a larger $k$ in those schemes does not necessarily ensure its effectiveness.
Considering that the primary goal of these studies is to deploy foundation models in safety-critical settings, establishing a robust reliability measure and analyzing its theoretical validity is vital.

\section{Theoretical Results and Omitted Proofs}

\subsection{Proof of Theorem~\ref{thm::sl_count_exam_inf}}
\label{append::sl_count_ex}

We present a more rigorous statement of Theorem~\ref{thm::sl_count_exam_inf} along with its proof.
\begin{thm}
\label{thm::sl_count_exam}
Consider downstream tasks as prediction tasks, with the collection of downstream heads closed under linear transformations (i.e., if $g$ belongs to this collection, $g\circ \bA$ also belongs to it for any linear transformation $\bA$). Assume that for any two points $\bz$ and $\bz'$ in the embedding space $\mathcal{Z}$, there exists an invertible linear transformation $\bA$ that maps $\bz$ to $\bz'$ and maps every point in $\mathcal{Z}$ to another point in $\mathcal{Z}$ (e.g. $\mathcal{Z}$ being $\mathbb{R}^d$ or the hyper-sphere $\mathcal{S}^{d-1}$ would satisfy this assumption). Let $\mathsf{diam}(\mathcal{Z}) = \sup_{\bz,\bz'\in\mathcal{Z}} \|\bz-\bz'\|_2$ be the diameter of the representation space. Then for any $A < (\mathsf{diam}(\mathcal{Z})/2)^2$ and a test point $\bx^*$, there exist embedding functions $h_1,\cdots,h_{2M} \in \mathcal{H}$ such that $\Varr{i\sim \{1,\ldots,2M\}}{h_i(\bx^*)} \geq A$ but $\Varr{{i\sim\{1,\ldots,2M\}}}{g_{i,t} \circ h_i(\bx^*)} = 0$ for any downstream task $t$, where $g_{i,t}$ is an optimal downstream head for $h_i$ under task $t$. 
\end{thm}

\begin{proof}
We prove this theorem via construction, assuming that there are only two embedding functions $h_1$, $h_2$, without loss of generality. Let $\delta>0$ be a small constant. We select two points $\bz_1, \bz_2 \in \mathcal{Z}$ such that $\|\bz_1 - \bz_2\|_2 \geq \mathsf{diam}(\mathcal{Z}) - \delta$. Then we define $h_1$ and $h_2$ to satisfy the condition: $h_i(\bx^*) = \bz_i$ and $h_1 = \bA \circ h_2$ where $\bA$ is an invertible matrix that maps every point in $\mathcal{Z}$ to another point in $\mathcal{Z}$. Consequently,
\begin{align*}
    \Varr{i\sim \{1,2\}}{h_i(\bx^*)} 
    = \left\|\frac{\bz_1 - \bz_2}{2}\right\|_2^2
    \geq \left(\frac{\mathsf{diam}(\mathcal{Z}) - \delta}{2}\right)^2.
\end{align*}
For a given task $t$, let $g_{1,t}$ be an optimal downstream head for the embedding function $h_1$. Next, we prove that $g_{1,t} \circ \bA$ is an optimal downstream head for $h_2$ under the same task. Suppose otherwise, if there exists another downstream head $g'_{2,t}$ achieving higher performance than $g_{1,t} \circ \bA$ when combined with $h_2$. Since $h_1 = \bA \circ h_2$, it implies $g'_{2,t} \circ \bA^{-1}$ can achieve higher performance than $g_{1,t}$  when combined with $h_1$. This contradicts the optimally of $g_{1,t}$. Therefore, $g_{1,t} \circ \bA$ is an optimal downstream head for $h_2$ for task $t$ so we denote it as $g_{2,t}$. Now, we have
\begin{align*}
    g_{2,t} \circ h_2(\bx^*)  = g_{1,t} \circ \bA \circ h_2(\bx^*) = g_{1,t} \circ \bA \circ \bA^{-1} \circ h_1(\bx^*) = g_{1,t} \circ h_1(\bx^*), 
\end{align*}
which leads to $\Varr{i\sim \{1,2\}}{g_{i,t} \circ h_i(\bx^*)} = 0$. As $\delta$ can be chosen arbitrarily small, we can achieve the desired result.
By setting  $h_{2i-1}(\bx^*) = \bz_1$ and $h_{2i}(\bx^*) = \bz_2$, the same proof can be extended to the case of 2$M$ case.
\end{proof}

\begin{figure}[t!]
  \centering
  \includegraphics[width=0.6\textwidth]{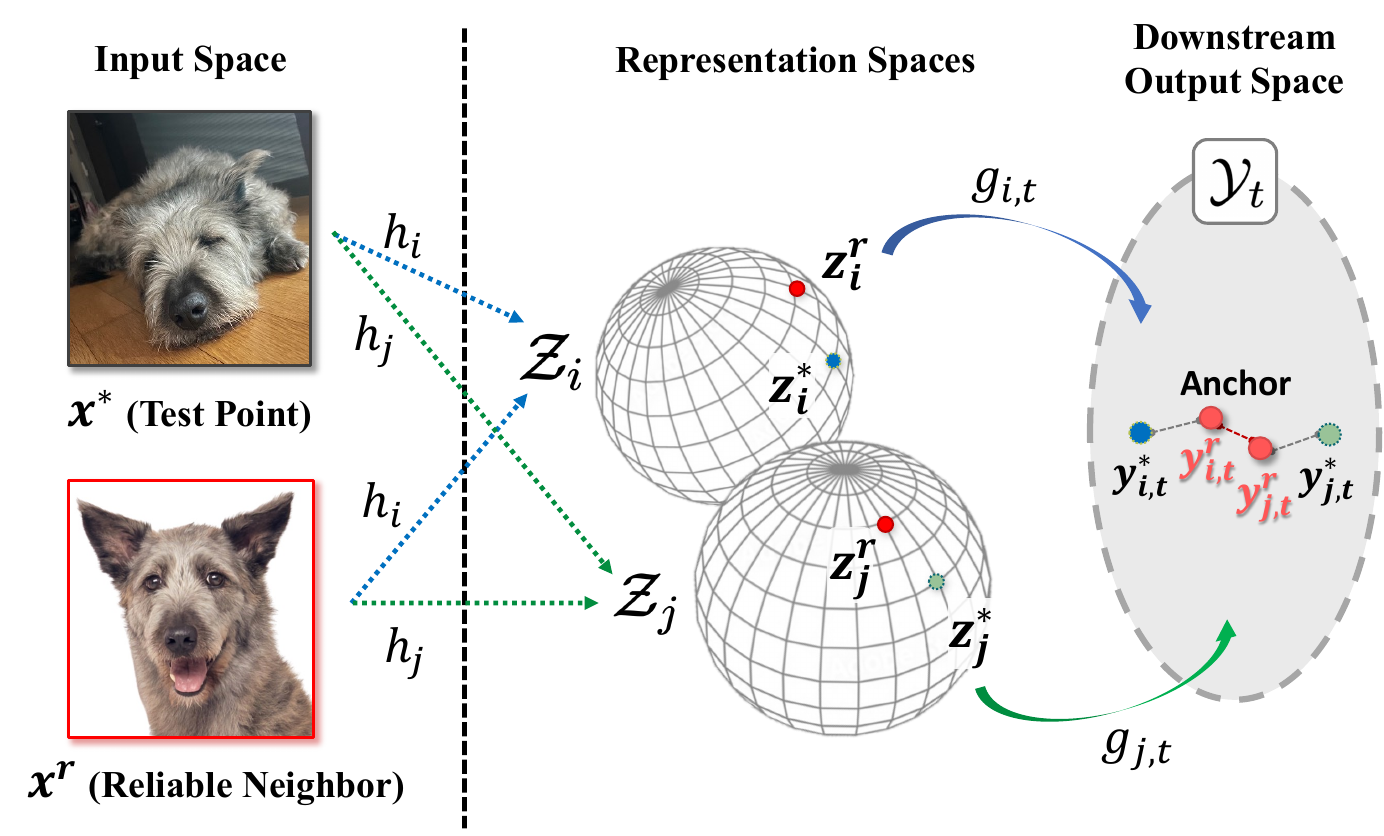}
  \caption{Graphical visualization of the sketch for the proof of Theorem~\ref{thm::nb_consistency}. Let $\mathcal{Z}i$ and $\mathcal{Z}j$ denote the representation spaces defined by the embedding functions $h_i$ and $h_j$, respectively.
  Suppose that there is a reliable neighboring point $\bx^r$ that is located close to the test point $\bx^*$ in each representation space. For any downstream task $t$, a reliable neighboring point $\bx^r$ serves as an anchor for comparing different representations $\bz_i^* = h_i(\bx^*)$ and $\bz_j^* = h_j(\bx^*)$ of the test point~$\bx^*$.
  The key idea is that $y_{i,t}^*$ and $y_{j,t}^*$  --- the downstream predictions on the test point using the two different embedding functions --- should be similar because the predictions $y_{i,t}^*$ and $y_{i,t}^r$ as well as $y_{j,t}^r$ and $y_{j,t}^*$ are similar due to the Lipschitz continuity of the downstream predictors. Additionally, since $\bx^r$ is a reliable point, the predictions $y_{i,t}^r$ and $y_{j,t}^r$ are similar. Thus, it follows that $y_{i,t}^*$ and $y_{j,t}^*$ are similar as well.
  } \label{fig:sketch_proof}
\end{figure}

\subsection{Proof of Theorem~\ref{thm::nb_consistency}} \label{pf::nb_consistency}

\begin{proof}
For the sake of notational simplicity, we denote $f_i = g_{i, t} \circ h_i$. When $g_{i, t}$ is Lipschitz continuous (see Appendix~\ref{app:lips} for more detail) and Equation~\eqref{eq:nb} --- $\lVert h_i(\bx^r) - h_i(\bx^*) \rVert_2 \le \epsilon_{nb}$ --- holds, we have:
\begin{equation}
\lVert f_i(\bx^r) - f_i(\bx^*) \rVert_2  \le L_{i, t} \cdot \lVert h_i(\bx^r) - h_i(\bx^*) \rVert_2 \le L_{i, t} \cdot \epsilon_{nb} \le L_{t} \cdot \epsilon_{nb} ~, ~ \forall i \in \{1,\ldots,M\} .
\end{equation}
By the triangle inequality, we have the following upper bound for the output difference:
\begin{align}
\lVert f_i(\bx^*) - f_j(\bx^*) \rVert_2^2 &\le \bigl( \lVert f_i(\bx^*) - f_i(\bx^r) \rVert_2 + \lVert f_i(\bx^r) - f_j(\bx^r)\rVert_2 + \lVert f_j(\bx^r) - f_j(\bx^*) \rVert_2 \bigr)^2 \nonumber \\
&\le \bigl(L_{t} \epsilon_{nb} + \lVert f_i(\bx^r) - f_j(\bx^r)\rVert_2 + L_{t} \epsilon_{nb} \bigr)^2 \nonumber \\
&= 4\bigl(L_{t} \epsilon_{nb} \bigr)^2 + 4 \bigl(L_{t} \epsilon_{nb} \bigr) \underbrace{\lVert f_i(\bx^r) - f_j(\bx^r)\rVert_2}_{L_2 ~ \text{difference}} + \underbrace{\lVert f_i(\bx^r) - f_j(\bx^r)\rVert_2^2}_{L_2 \text{-squared difference}} . 
\end{align}
Note that the ensemble variance is proportional to the average pairwise $L_2$-squared difference across the ensemble of $f_i$:
\begin{align}
\Varr{i\sim\{1,\ldots,M\}}{g_{i,t} \circ h_i(\bx^*)} &= \frac{1}{M^2} \sum_{i < j} \lVert f_i(\bx^*) - f_j(\bx^*) \rVert_2 ^2 \\
\Varr{i\sim\{1,\ldots,M\}}{g_{i,t} \circ h_i(\bx^r)} &= \frac{1}{M^2} \sum_{i < j} \lVert f_i(\bx^r) - f_j(\bx^r) \rVert_2 ^2 \equiv \sigma_{r, t}^2.
\end{align}

Furthermore, by Cauchy-Schwarz inequality, the average pairwise $L_2$ difference is bounded by:
\begin{align}
\frac{1}{M^2} \sum_{i < j} \lVert f_i(\bx^r) - f_j(\bx^r) \rVert_2 & \le \frac{1}{M^2} \Bigl( \sum_{i < j} \lVert f_i(\bx^r) - f_j(\bx^r) \rVert_2^2 \Bigr)^{1/2} \cdot \Bigl( \sum_{i < j} 1 \Bigr)^{1/2} \nonumber \\
& \le \frac{1}{M^2} \Bigl( M \sigma_{r, t} \Bigr) \cdot \Bigl( \frac{M(M-1)}{2} \Bigr)^{1/2} \le \frac{\sqrt{2}}{2} \sigma_{r, t}
\end{align}

Thus: 
\begin{align}
\Varr{i\sim\{1,\ldots,M\}}{g_{i,t} \circ h_i(\bx^*)} &= \frac{1}{M^2} \sum_{i < j} \lVert f_i(\bx^*) - f_j(\bx^*) \rVert_2 ^2 \nonumber \\
&\le \Bigl(\frac{1}{M^2}\Bigr) \Bigl(\frac{M(M-1)}{2}\Bigr) \Bigl( 4\Bigl(L_t \epsilon_{nb}\Bigr)^2 \Bigr) + 2\sqrt{2} \bigl(L_{t} \epsilon_{nb} \bigr) \sigma_{r, t} + \sigma_{r, t}^2 \nonumber \\
&\le \Bigl(\sqrt{2}L_t \epsilon_{nb}\Bigr)^2 + 2\sqrt{2} L_{t} \epsilon_{nb} \sigma_{r, t} + \sigma_{r, t}^2 \nonumber \nonumber \\ 
&= \Bigl(\sqrt{2}L_t \epsilon_{nb} + \sigma_{r, t} \Bigr)^2.
\end{align}
\end{proof}

The authors refer to Figure~\ref{fig:sketch_proof} for the visualization accompanying the sketch of the proof.

\subsection{Lipschitz Continuity of Neural Networks} \label{app:lips}

\begin{lem} \label{lem::lip_cont}
For a fully-connected layer $g_l(\bz) = a(\bw_l^T \bz + \bb_l)$ of downstream predictor, the similarity of outputs is bounded above by the Lipschitz continuity:
\begin{equation}
  \lVert g_l (\bz) - g_l (\bz^*) \rVert_2 \le \lVert \bw_l \rVert_2 \cdot \lVert \bz - \bz^* \rVert_2.
\end{equation}

\end{lem}
\begin{proof}
Suppose $a(\cdot)$ is 1-Lipschitz continuous (e.g., identity, ReLU, sigmoid, and softmax) ,
\begin{align}
    \lVert g_l (\bz) - g_l (\bz^*) \rVert_2 &= \lVert a\big( \bw_l^T \bz + \bb_l \big) - a\big( \bw_l^T \bz^* + \bb_l \big) \rVert_2 \nonumber \\
    &\le 1 \cdot \lVert \big( \bw_l^T \bz + \bb_l \big) - \big( \bw_l^T \bz^* + \bb_l \big) \rVert_2 \nonumber \\ 
    & = \lVert \bw_l^T \big( \bz- \bz^* \big) \rVert_2 \nonumber \\
    &\le \lVert \bw_l \rVert_2 \cdot \lVert \bz - \bz^* \rVert_2 .
\end{align}
where $\lVert \cdot \rVert_2$ is the spectral norm for a matrix and $L_2$ norm for a vector.
\end{proof}

\begin{cor} 
For a feed-forward neural networks, $g = g_1 \circ g_2 \circ ... \circ g_L$, composed of multiple fully-connected layers, the similarity of predictions at $\bx$ and $\bx^*$ is bounded above by:
\begin{equation}
  \lVert g (\bx) - g (\bx^*) \rVert_2 \le L_t \cdot \lVert \bx - \bx^* \rVert_2.
\end{equation}
where $L_t = \lVert \bw_1 \rVert_2 \cdot \lVert \bw_2 \rVert_2 \cdots \lVert \bw_L \rVert_2$ is a Lipschitz constant for $g$.
\end{cor}

\subsection{Proof of Theorem~\ref{thm:nb_necessary}} \label{app:necessary}

We prove Theorem~\ref{thm:nb_necessary} by establishing its contrapositive. We show that if a point $\bx^*$ is both locally in-distribution (violating condition i) and reliable (violating condition ii), it must possess a consistent neighbor.

First, we formalize the bi-Lipschitz assumption. For each head $g_{i,t}$, we assume there exist constants $0 < m_{i,t} \le L_{i,t} < \infty$ such that for all $\bz, \bz'$ in the representation space:
\begin{equation} \label{eq:app_bilip}
m_{i,t}\|\bz - \bz'\|_2 \le \|g_{i,t}(\bz) - g_{i,t}(\bz')\|_2 \le L_{i,t}\|\bz - \bz'\|_2.
\end{equation}
Let $m_t = \min_i m_{i,t}$ and $L_t = \max_i L_{i,t}$ be the global bounds.

\begin{thm} \label{thm:consistency_implication}
Suppose that a test point $\bx^*$ is locally in-distribution, meaning there exists at least one representation $i_0 \in \{1,\ldots,M\}$ and a reliable reference point $\bx^r \in \bX_{\text{ref}}$ such that $\bx^r$ is an $\epsilon$-neighbor of $\bx^*$:
\begin{equation} \label{eq:app_neighbor_exist}
\|h_{i_0}(\bx^*) - h_{i_0}(\bx^r)\|_2 \le \epsilon
\end{equation}
where
\begin{equation}
\label{eq:app_ref_uniform_reliable}
\max_{i,j\in\{1,\ldots,M\}}
\big\|
g_{i,t}\circ h_i(\bx^r) - g_{j,t}\circ h_j(\bx^r)
\big\|_2
\le \sigma_{\rm ref}.
\end{equation}
Suppose further that $\bx^*$ is reliable with respect to $\sigma_{\text{ref}}$, satisfying:
\begin{equation} \label{eq:app_test_reliable}
\max_{i,j\in\{1,\ldots,M\}} \big\| g_{i,t}\circ h_i(\bx^*) - g_{j,t}\circ h_j(\bx^*) \big\|_2 \le \sigma_{\text{ref}}.
\end{equation}
Then, $\bx^r$ is a consistent $\epsilon'$-neighbor of $\bx^*$ across all representations $j \in \{1,\ldots,M\}$:
\begin{equation}
\|h_j(\bx^*) - h_j(\bx^r)\|_2 \le \epsilon' := \frac{L_t}{m_t}\epsilon + \frac{2\sigma_{\text{ref}}}{m_t}.
\end{equation}
\end{thm}

\begin{proof}
Let $f_i = g_{i,t} \circ h_i$. Fix the representation $i_0$ and reference point $\bx^r$ satisfying \eqref{eq:app_neighbor_exist}.
First, we map the neighborhood distance in space $i_0$ to the output space using the Lipschitz upper bound from \eqref{eq:app_bilip}:
\begin{equation} \label{eq:app_step1}
\|f_{i_0}(\bx^*) - f_{i_0}(\bx^r)\|_2 \le L_{i_0,t} \|h_{i_0}(\bx^*) - h_{i_0}(\bx^r)\|_2 \le L_t \epsilon.
\end{equation}

Next, consider any other representation $j \in \{1,\ldots,M\}$. We bound the prediction difference between $\bx^*$ and $\bx^r$ under head $j$ using the triangle inequality:
\begin{equation}
\|f_j(\bx^*) - f_j(\bx^r)\|_2 \le \underbrace{\|f_j(\bx^*) - f_{i_0}(\bx^*)\|_2}_{\text{Test Reliability}} + \underbrace{\|f_{i_0}(\bx^*) - f_{i_0}(\bx^r)\|_2}_{\text{Lipschitz Bound}} + \underbrace{\|f_{i_0}(\bx^r) - f_j(\bx^r)\|_2}_{\text{Ref Reliability}}.
\end{equation}
Substituting the bounds from \eqref{eq:app_ref_uniform_reliable}, \eqref{eq:app_test_reliable}, and \eqref{eq:app_step1}:
\begin{equation} \label{eq:app_step2}
\|f_j(\bx^*) - f_j(\bx^r)\|_2 \le \sigma_{\text{ref}} + L_t \epsilon + \sigma_{\text{ref}} = L_t \epsilon + 2\sigma_{\text{ref}}.
\end{equation}

Finally, we map this output distance back to the representation space $\mathcal{Z}_j$ using the inverse Lipschitz (lower) bound from \eqref{eq:app_bilip}:
\begin{equation} \label{eq:app_step3}
m_{j,t} \|h_j(\bx^*) - h_j(\bx^r)\|_2 \le \|f_j(\bx^*) - f_j(\bx^r)\|_2.
\end{equation}
Combining \eqref{eq:app_step2} and \eqref{eq:app_step3}, and using $m_t \le m_{j,t}$:
\begin{equation}
\|h_j(\bx^*) - h_j(\bx^r)\|_2 \le \frac{1}{m_t} (L_t \epsilon + 2\sigma_{\text{ref}}) = \epsilon'.
\end{equation}
Since this holds for all $j \in \{1,\ldots,M\}$, $\bx^r$ is an $\epsilon'$-consistent neighbor.
\end{proof}

\paragraph{Conclusion of Proof for Theorem~\ref{thm:nb_necessary}.}
Theorem~\ref{thm:consistency_implication} establishes that:
\[ (\bx^* \text{ is Reliable}) \land (\exists \text{ Reliable Neighbor in some } i) \implies (\exists \text{ Consistent Neighbor}). \]
Taking the contrapositive: if no consistent neighbor exists, then either $\bx^*$ is unreliable (negating reliability) or $\bx^*$ has no neighbor in any representation (negating the existence of a neighbor in some $i$, i.e., OOD).

\subsection{Proof of Theorem \ref{thm::perturbed_consistency}}

\begin{proof}
    Denote $f(x^*;\theta)=g_{\theta,t}\circ h(x^*;\theta)$ and $\mu$ to be the probability measure of $\mathcal{P}$. Therefore,
    \begin{align*}
        \mathsf{Var}_{\theta\sim\mathcal{P}}[g_{\theta,t}\circ h(x^*;\theta)]
        &=\frac{1}{2}\int\int\|f(x^*;\theta)-f(x^*;\theta')\|_2^2d\mu(\theta)d\mu(\theta').
    \end{align*}
    For the term inside the double integral, we can decompose it into the following using the triangle inequality and the quadratic mean inequality,
    \begin{align*}
        \|f(x^*;\theta)-f(x^*;\theta')\|_2^2
        &\leq(\|f(x^*;\theta')-f(x^r;\theta')\|_2+\|f(x^r;\theta')-f(x^r;\theta)\|_2\\&+\|f(x^r;\theta)-f(x^*;\theta)\|_2)^2\\
        &\leq(\|f(x^*;\theta')-f(x^r;\theta')\|_2+\|f(x^r;\theta)-f(x^*;\theta)\|_2)^2\\&+2\|f(x^r;\theta')-f(x^r;\theta)\|_2\|f(x^*;\theta')-f(x^r;\theta')\|_2\\&+2\|f(x^r;\theta')-f(x^r;\theta)\|f(x^*;\theta)-f(x^r;\theta)\|_2\\&+\|f(x^r;\theta')-f(x^r;\theta)\|_2^2.\\
        &\leq2\|f(x^*;\theta')-f(x^r;\theta')\|_2^2+2\|f(x^*;\theta)-f(x^r;\theta)\|_2^2\\&+2\|f(x^r;\theta')-f(x^r;\theta)\|_2\|f(x^*;\theta')-f(x^r;\theta')\|_2\\&+2\|f(x^r;\theta')-f(x^r;\theta)\|_2\|f(x^*;\theta)-f(x^r;\theta)\|_2\\&+\|f(x^r;\theta')-f(x^r;\theta)\|_2^2.
    \end{align*}
    From here, we can bound the above double integral as a sum of four, easier-to-analyze double integrals
    \[I_1=\int\int(\|f(x^*;\theta')-f(x^r;\theta')\|_2^2+\|f(x^*;\theta)-f(x^r;\theta)\|_2^2)d\mu(\theta)d\mu(\theta'),\]
    \[I_2=\int\int \|f(x^r;\theta')-f(x^r;\theta)\|_2\|f(x^*;\theta')-f(x^r;\theta')\|_2d\mu(\theta)d\mu(\theta'),\]
    \[I_3=\int\int \|f(x^r;\theta')-f(x^r;\theta)\|_2\|f(x^*;\theta)-f(x^r;\theta)\|_2d\mu(\theta)d\mu(\theta'),\]
    \[I_4=\frac{1}{2}\int\int \|f(x^r;\theta')-f(x^r;\theta)\|_2^2d\mu(\theta)d\mu(\theta').\]
    We bound them one by one using Cauchy-Schwarz.
    \begin{align*}
        I_1
        &=2\int \|f(x^*;\theta)-f(x^r;\theta)\|_2^2d\mu(\theta)\\
        &\leq2\int L_{\theta,t}^2\|h(x^*;\theta)-h(x^r;\theta)\|_2^2d\mu(\theta)\\
        &\leq2\left(\int L_{\theta,t}^4d\mu(\theta)\right)^{1/2}\left(\int \|h(x^*;\theta)-h(x^r;\theta)\|_2^4d\mu(\theta)\right)^{1/2}\\
        &=2\EEE{\theta\sim\mathcal{P}}{L_{\theta,t}^4}^{1/2}\EEE{\theta\sim\mathcal{P}}{\|h(x^*;\theta)-h(x^r;\theta)\|_2^4}^{1/2}\leq2L_t^2\epsilon_{nb}^2.
    \end{align*}
    Next,
    \begin{align*}
        I_2
        &\leq\left(\int\int\|f(x^r;\theta')-f(x^r;\theta)\|_2^2d\mu(\theta)d\mu(\theta')\right)^{1/2}\\&\times\left(\int\int\|f(x^*;\theta')-f(x^r;\theta')\|_2^2d\mu(\theta)d\mu(\theta')\right)^{1/2}\\
        &=\sqrt{2}\mathsf{Var}_{\theta\sim\mathcal{P}}[f(x^r;\theta)]^{1/2}\left(\int\|f(x^*;\theta)-f(x^r;\theta)\|_2^2d\mu(\theta)\right)^{1/2}\\
        &=\sqrt{2}\sigma\EEE{\theta\sim\mathcal{P}}{\|f(x^*;\theta)-f(x^r;\theta)\|_2^2}^{1/2}\\
        &\leq\sqrt{2}\sigma\EEE{\theta\sim\mathcal{P}}{L_{\theta,t}^2\|h(x^*;\theta)-h(x^r;\theta)\|_2^2}^{1/2}\\
        &\leq\sqrt{2}\sigma\EEE{\theta\sim\mathcal{P}}{L_{\theta,t}^4}^{1/4}\EEE{\theta\sim\mathcal{P}}{\|h(x^*;\theta)-h(x^r;\theta)\|_2^4}^{1/4}\leq\sqrt{2}L_t\sigma\epsilon_{nb}.
    \end{align*}
    Similarly,
    \[I_3\leq\sqrt{2}L_t\sigma\epsilon_{nb}.\]
    Finally,
    \[I_4=\mathsf{Var}_{\theta\sim\mathcal{P}}[f(x^r;\theta)]=\sigma^2.\]
    Thus,
    \[\mathsf{Var}_{\theta\sim\mathcal{P}}[g_{\theta,t}\circ h(x^*;\theta)]\leq2L_t^2\epsilon_{nb}^2+2\sqrt{2}L_t\sigma\epsilon_{nb}+\sigma^2=(\sqrt{2}L_t\epsilon_{nb}+\sigma)^2\]
\end{proof}

\section{Further Discussion on Neighborhood Consistency} 
\label{app:alternative}

In this section, we present an alternative viewpoint to clarify the underlying motivation behind neighborhood consistency and demonstrate how it addresses the shortcomings of existing supervised learning frameworks.
As illustrated in Section~\ref{sec:bg}, the uncertainty in supervised learning could be expressed by the variance of predictions across various functions. 
It is important to note that this approach is applicable since different functions transform the input into the "same" output space and the distance/similarity within the output space is well-defined.

From the perspective of unsupervised learning, however, there is no "unique" ground truth representation space which makes it difficult to compare the representations constructed by different embedding functions.
Therefore, to properly compare the two distinct abstract representations $h_i(\bx^*)$ and $h_j(\bx^*)$, we need a transformation function $\phi(\cdot)$ that maps different representations into a comparable form.
Then, following Equation~\eqref{eq:unc_sup}, the consistency across the set of representations $\cH = \{h_1(\bx^*), \cdots, h_M(\bx^*) \}$ can be assessed as we do in :
\begin{equation} \label{eq:repr_cs}
    \mathsf{Unc}(\bx^* ; \cH) = \frac{1}{M^2} \sum_{i < j} \Sim{\phi(\bz^*_i), \phi(\bz^*_j})
\end{equation}
where $\phi(\bz^*_i)$ can be viewed as a \emph{surrogate representation} of $\bz^*_i = h_i(\bx^*)$.

What characteristics should an appropriate $\phi(\cdot)$ have?
First and foremost, it would require an anchor that connects different representation spaces, and this paper suggests using (reliable) "pre-training data" as such anchors.
As such, we first introduce a surrogate vector representation consisting of the relative distances from the test point to each reference point:
\begin{equation}
    \phi_\text{rel-fc}(\bz_i^* ; \bX_{\text{ref}}) \defined \Big[ \mathsf{dist} \big( \bz_i^*, \bz_i^{(1)} \big), \cdots, \mathsf{dist} \big( \bz_i^*, \bz_i^{(n)} \big) \Big]^T
\end{equation}
where $\bz_i^{(l)} = h_i(\bx^{(l)})$ and $\mathsf{dist}(\cdot)$ is a distance metric in the representation space (e.g., cosine distance).
This surrogate representation captures the relational information to reference data points rather than taking account for the absolute location of the test representation.

The aforementioned surrogate, however, can be burdened with an excessive amount of information.
To address the issue, we may consider a \emph{sparsified representation} as in following:
\begin{equation} \label{eq:sparse}
    \phi_\text{rel-sparse}(\bz_i^* ; \bX_{\text{ref}}) \defined \Big[ \mathbbm{1} \big( \mathsf{dist} \big( \bz_i^*, \bz_i^{(1)} \big) \le \epsilon \big) , \cdots, \mathbbm{1} \big( \mathsf{dist} \big( \bz_i^*, \bz_i^{(n)} \big) \le \epsilon \big) \Big]^T
\end{equation}
where $\mathbbm{1}$ is a indicator function; and $\epsilon \in \Reals$ is a small real number.
The equivalent set notation of the vector representation is: $\epsilon\text{-NN}_i(\bx^*) \defined \{ l \mid \mathsf{dist}\big( h_i(\bx^*), ~ h_i(\bx^{(l)}) \big) \le \epsilon \}$.
The corresponding neighborhood consistency measure can be computed as follows:
\begin{equation}
    \mathsf{NC}_{\epsilon}(\bx^*) \defined \frac{1}{M^2} \sum_{i < j} \Sim{\epsilon\text{-NN}_i \big( \bx^* \big), ~ \epsilon\text{-NN}_j \big( \bx^*\big)} .
\end{equation}

However, selecting an appropriate value for $\epsilon$ can be challenging since determining the proper scale of the representation space is not straightforward. As a result, we introduce a scale-free version of neighborhood consistency that relies on the local $k$-nearest neighborhood:
\begin{equation}
    \mathsf{NC}_{k}(\bx^*) \defined \frac{1}{M^2} \sum_{i < j} \Sim{k\text{-NN}_i \big( \bx^* \big), ~ k\text{-NN}_j \big( \bx^*\big) }
\end{equation}
where $k\text{-NN}_i(\bx^*)$ is the index set of $k$ nearest neighbors of $h_i(\bx^*)$ among $\bZ_{i, \text{ref}}$.

For a set similarity metric $\mathsf{Sim}$, we can use either \emph{Jaccard Similarity} or \emph{Overlap Coefficient}:
\begin{align}
    \text{Jaccard Similarity}(S_1, S_2) &\defined \frac{\lvert S_1 \cap S_2 \rvert}{\lvert S_1 \cup S_2 \rvert} \\
    \text{Overlap Coefficient}(S_1, S_2) &\defined \frac{\lvert S_1 \cap S_2 \rvert}{\min (\lvert S_1\rvert, \lvert S_2\rvert)} .
\end{align}
If the size of $S_1$ and $S_2$ are equal to $k$: $\lvert S_1 \cup S_2 \rvert = 2k - \lvert S_1 \cap S_2 \rvert$ and $\min (\lvert S_1\rvert, \lvert S_2\rvert) = k$.
Thus, regardless of the selection, both similarity metrics solely depends on $\lvert S_1 \cap S_2 \rvert$ when the $k$-NN approach is used to determine the neighbors.
Additionally, given a set of test points, both metrics provide the same order (i.e., rank) in terms of their representation reliability.

\section{Additional Experimental Results} \label{app:suppl}

\subsection{Multi-class Downstream Classification Tasks for NC} \label{app:multi}
As noted in Section~\ref{sec:setup}, our evaluation protocol employs one-vs-one (OVO) binary classification tasks for downstream evaluation.
While it is intuitively plausible that the average performance across OVO tasks correlates well with that of multi-class classification tasks, we undertake empirical studies on the original multi-class setups to validate this assumption.
The key distinction from the original protocol is that, the number of tasks is just one (i.e., $|\cT| = 1$) and the output size for the linear heads corresponds to the number of classes $C$.
For $\texttt{CIFAR-10}$ and $\texttt{STL-10}$ datasets, $C=10$ class labels are used, while $\texttt{CIFAR-100}$ uses $C=20$ coarse (i.e., superclass) labels.
The outcomes, presented in Tables~\ref{tab:comparison_multi} and \ref{tab:trasnfer_multi}, exhibit a consistent trend with previous observations.

\subsection{Performance with Unnormalized Representation for NC} \label{app:unnormalized}
In addition to the result presented in Table~\ref{tab:comparison} (ID) and Table~\ref{tab:transfer} (Transfer), we further explore the efficacy of our method with unnormalized representation (e,g., Euclidean distance instead of cosine distance).
We also include the standard deviation resulting from the randomness involved in selecting the reference dataset.
As indicated by the results in Table~\ref{tab:comparison_unnormalized} (ID), Table~\ref{tab:transfer_unnormalized} (Transfer), our method exhibits consistent and strong performance, even when utilizing unnormalized representation.

On the other hand, the baseline methods suffer from a significant performance reduction when paired with Euclidean distance. 
Our conjecture is as follows: as noted in our results and \cite{tack2020csi}, points with a larger $L_2$-norm tend to exhibit higher reliability.
This implies that reliable points are often located in the outer regions of the representation space, resulting in larger distances between points in those regions compared to points near the origin.
This observation, however, contradicts the predictions of $\amone$ and $\mathsf{FV}$ (i.e., smaller value is presumed to indicate higher reliability).
As a result, those measures on the unnormalized representations struggle to accurately reflect reliability, potentially indicating negative correlations that defy their assumptions.

\subsection{Ablation Studies for NC}
In addition to the exemplary results in the main manuscript, we further provide the ablation results in Figure~\ref{fig:abl_m_full} ($M$) and Figure~\ref{fig:abl_k_full} ($k$).
These ablation results cover various combinations of pre-training algorithms, model architectures, and datasets.
Brier score is used for the downstream performance metric.

\subsection{Multi-class Downstream Classification Tasks for PNC}
\label{app:individual_exp_multi}

Tables \ref{tab:individual_single_indist} and \ref{tab:individual_single_transfer} have shown the Kendall's $\tau$ correlation between the $\mathsf{PNC}_{100}^s$ scores with the downstream performance for binary OVO tasks. To complement their results, we present Tables \ref{tab:indist_single_multi} and \ref{tab:transfer_single_multi} to compare our method with the baselines when we use the Kendall's $\tau$ correlations between these methods and a multi-class downstream classification tasks.

These additional tables show that for all the different combinations of downstream task type and downstream performance metrics, all our $\mathsf{PNC}_{100}^s$ scores remain to have, on average, a positive correlation with the downstream uncertainty and reliability, indicating further the reliability of the perturbation-based $\mathsf{PNC}_k$ method and the $\mathsf{PNC}$-spread tuning method. On a closer look, our method also has higher Kendall's $\tau$ correlation scores than the corresponding baseline methods for most architectures, SSL algorithms, pretraining sets, and downstream datasets, when we look at each combination of downstream task type (binary or multi-class) and downstream performance metric (Brier or prediction entropy scores).

\subsection{Efficacy of the Proposed PNC-Spread Tuning Strategy} \label{app:pnc_oracle}
To further validate the $\mathsf{PNC}$-spread tuning strategy, we compare the spread-tuned score $\mathsf{PNC}_{100}^s$ against the oracle score $\mathsf{PNC}_{100}^o$, which selects $\sigma_{\text{pert}}$ by maximizing the \kendall correlation with the downstream performance (information that is unavailable in practice). Tables~\ref{tab:pnc_oracle_indist} and~\ref{tab:pnc_oracle_transfer} report this comparison for binary classification tasks, and Tables \ref{tab:pnc_oracle_indist_multi} and \ref{tab:pnc_oracle_transfer_multi} report this comparison for the multi-class classification tasks. The spread-tuned version incurs only a marginal drop relative to the oracle, confirming the practical effectiveness of the tuning criterion.

Furthermore, we extend the result of Figure \ref{fig:stddev_tuning_indist_binary_brier_simclr} show the tuning process for the case of predicting the binary in-distribution downstream Brier scores for BYOL and MoCo models in Figure \ref{fig:stddev_tuning_indist_binary_brier}. We see that the peaks for the Kendall-$\tau$ graphs are close to those of the $\mathsf{PNC}$ spread graphs for BYOL and MoCo models as well, indicating that the $\mathsf{PNC}$-spread method works well for models trained using those algorithms.

Note that in each of the subgraphs of these figures, we have a dashed gray line that indicates the maximum perturbation levels that we consider and does not induce a floating-point overflow in our numerical simulator. If there are no dashed gray lines in the subgraph, it means that for all the considered $\sigma_{pert}$, there were no overflows. These overflows may occur in the inference process of the embedding model or in the computation of the neighborhoods. The former of which can happen due to the compounding of layers with parameters that are getting too large from the perturbations, and the latter can happen due to the embedding having entries that are too large.

\subsection{Alternative to Gaussian Distribution for PNC} \label{app:alt_perturbation}

We want to explore some alternatives to $\mathcal{P}(\theta^*)$, with a slight relaxation on the assumption on what we have access to for constructing the distribution. Specifically, here, we have access to the set of loss functions that were used $\mathcal{L}=\{l_1,l_2,...,l_N\}$, where each $l_i$ map the parameter $\theta$ to some real number. For example, $l_i$ could indicate the sample-wise loss.

One potential heuristic that we can use is that we want to sample a perturbation $\Delta\theta$ that minimizes the changes in loss so that the resulting perturbed model $h(\cdot;\theta^*+\Delta\theta)$ can have a similar performance, i.e. $l_i(\theta^*)\approx l_i(\theta^*+\Delta\theta)$. To do this, we can sample $\Delta\theta$ from some distribution and then project it to the space that is orthogonal to $\nabla l_1(\theta^*),\nabla l_2(\theta^*),...,\nabla l_N(\theta^*)$. Denote $T_{\theta^*}$ to be the desired projection function and $\nabla L(\theta^*):=[\nabla l_1(\theta^*),\nabla l_2(\theta^*),...,\nabla l_N(\theta^*)]$, so more formally
\[T_{\theta^*}(\Delta\theta)=\Delta\theta-\nabla L(\theta^*)(\nabla L(\theta^*)^\top\nabla L(\theta^*))^{-1}\nabla L(\theta^*)^\top\Delta\theta.\]
For our usage, we sample from this distribution by sampling $\Delta\theta'$ from $\mathcal{N}(0,\sigma_{pert}^2I)$, and for a fair comparison with the distribution used in the main body of the paper at each perturbation magnitude $\sigma_{pert}$, we re-normalize the perturbation and get
\[\Delta\theta=\frac{\|\Delta\theta'\|_2}{\|T_{\theta^*}(\Delta\theta')\|_2}\cdot T_{\theta^*}(\Delta\theta').\]
Furthermore, for this section, we consider only the results for the SimCLR framework, so the losses $l_i$ will be the loss described by \citet{chen2020simple} for the samples. However, since $\theta$ is high-dimensional, we only use a subset of the sample loss functions for tractability, taking a subset of $400$. Denote the $\mathsf{PNC}_k$ method that uses this orthogonal sampling as the $\mathsf{PNC}^\perp_k$ method.

Conversely, perhaps we want a perturbation $\Delta\theta$ that changes the loss the most by only perturbing in the directions of the gradients. We can think of this perturbation as a random linear combination of the gradients at that point. This is similar in nature to SWAG \citep{maddox2019swag}, where they used the parameter iterates to construct a distribution that is fitted on the last few iterates of the training algorithm. Similar to our proposed method of sampling $\Delta\theta$, in the case of gradient-based algorithms, this distribution would have a higher spread in the direction of the gradients. Here, we can use a similar procedure as the one in $\mathsf{PNC}^\perp_k$ to sample $\Delta\theta$, but instead of using the projection function $T_{\theta^*}$, we use the projection function $T'_{\theta^*}$ that satisfies
\[T'_{\theta^*}(\Delta\theta)=\nabla L(\theta^*)(\nabla L(\theta^*)^\top\nabla L(\theta^*))^{-1}\nabla L(\theta^*)^\top\Delta\theta.\]
We denote the $\mathsf{PNC}_k$ method that uses this $\mathcal{P}(\theta^*)$ distribution as $\mathsf{PNC}_k^{\mathbin{\|}}$.

Finally, inspired by the works of \citet{nobari2025activationinformedmerginglargelanguage}, we consider the case where the perturbation $\Delta\theta$ is sampled from a distribution where the variance of each entry is scaled inversely proportional to how sensitive the performance is with respect to that entry. More formally, we want the variance to be inversely proportional to
\[\frac{1}{N}\sum_{i=1}^N\nabla l_i(\theta^*)^2,\]
here the squaring is entry-wise. Thus, for a perturbation magnitude $\sigma_{pert}$, we sample
\[\Delta\theta\sim\mathcal{N}\left(0,\frac{\sigma_{pert}^2dG^{-1}}{\operatorname{trace}(G^{-1})}\right),\]
where $d$ is the number of trainable parameters and
\[G:=\sum_{i=1}^N\operatorname{diag}(\nabla l_i(\theta^*))^2.\]
For this $\mathcal{P}(\theta^*)$, we denote the resulting $\mathsf{PNC}_k$ method as $\mathsf{PNC}_k^g$. For these three new methods, we use the same $\mathsf{PNC}$-spread tuning method to tune $\sigma_{pert}$, and get $\mathsf{PNC}_k^{\perp,s},\mathsf{PNC}_k^{\|,s}$, and $\mathsf{PNC}_k^{g,s}$. Furthermore, we also consider the case where we can perfectly optimize the perturbation magnitude, similar to $\mathsf{PNC}_k^{o}$, and denote them as $\mathsf{PNC}_k^{\perp,o}$, $\mathsf{PNC}_k^{\|,o}$, and $\mathsf{PNC}_k^{g,o}$. We evaluate all of them in the same way we evaluated $\mathsf{PNC}_k^s$ and $\mathsf{PNC}_k^{o}$.

We can see their results for the in-distribution downstream task in Table \ref{tab::pnc_indist_others}, and for transfer learning tasks in Table \ref{tab::pnc_transfer_others}. When comparing the results of Table \ref{tab::pnc_indist_others} with Tables \ref{tab:individual_single_indist} and \ref{tab:indist_single_multi}, we see that the only non-oracle $\mathsf{PNC}$ methods that consistently surpasses the baselines are $\mathsf{PNC}_{100}^s$ and $\mathsf{PNC}_{100}^{\mathbin{\|},s}$. Furthermore, the Kendall-$\tau$ scores for $\mathsf{PNC}_{100}^{\mathbin{\|},s}$ are always higher than those of $\mathsf{PNC}_{100}^s$, meaning that for SimCLR models in in-distribution downstream tasks, it is better to have $\mathcal{P}(\theta^*)$ that has a higher variance in the direction of greater change of model performance, which can be achieved by traversing the gradients. We can get a similar conclusion for transfer learning tasks when we compare the results in Table \ref{tab::pnc_transfer_others} with Tables \ref{tab:individual_single_transfer} and \ref{tab:transfer_single_multi}.

Tables \ref{tab::pnc_indist_others} and \ref{tab::pnc_transfer_others} also shows that there is only a marginal drop in correlation between the oracle methods ($\mathsf{PNC}_k^{\perp,o}$, $\mathsf{PNC}_k^{\|,o}$, and $\mathsf{PNC}_k^{g,o}$) and the spread-tuned methods ($\mathsf{PNC}_k^{\perp,s}$, $\mathsf{PNC}_k^{\|,s}$, and $\mathsf{PNC}_k^{g,s}$), indicating that $\mathsf{PNC}$-spread works well even for other family of distributions $\mathcal{P}(\theta^*)$.\newpage

\begin{table*}[t!] 
\centering
\fontsize{9pt}{9pt}\selectfont
\caption{Comparison on the correlation between our neighborhood consistency ($\mathsf{NC}_{100}$) and baseline methods in relation to performance on in-distribution downstream tasks.}
\renewcommand{\arraystretch}{1.25}
\begin{tabular}{cc|cccc}
\toprule
\multirow{3}{*}{\makecell[c]{Pretraining \\ Algorithms}} & \multirow{3}{*}{Method} & \multicolumn{4}{c}{ResNet-18}  \\  \cline{3-6} 
 & & \multicolumn{2}{c}{\texttt{CIFAR-10}}  & \multicolumn{2}{c}{\texttt{CIFAR-100}} \\
 & & Entropy  & Brier & Entropy & Brier  \\
\midrule
\multirow{5}{*}{SimCLR} & $\mathsf{NC}_{100}$ & \textbf{0.3865} \tiny{$\pm$ 0.0024} & \textbf{0.3262} \tiny{$\pm$ 0.0022} & \textbf{0.3130} \tiny{$\pm$ 0.0031} & \textbf{0.2590} \tiny{$\pm$ 0.0028} \\
& $\amone$ & \underline{0.3021} \tiny{$\pm$ 0.0063} & \underline{0.2544} \tiny{$\pm$ 0.0051} & 0.0757 \tiny{$\pm$ 0.0084} & 0.0635 \tiny{$\pm$ 0.0075} \\
& $\mathsf{Norm}$ & 0.2907 \hspace{2.2em} & 0.2429 \hspace{2.2em} & \underline{0.1198} \hspace{2.2em} & \underline{0.0649} \hspace{2.2em} \\
& $\mathsf{FV}$ & -0.0476 \hspace{2.2em} & -0.0458 \hspace{2.2em} & -0.1891 \hspace{2.2em} & -0.1710 \hspace{2.2em} \\
\midrule
\multirow{5}{*}{BYOL} & $\mathsf{NC}_{100}$ & \textbf{0.2749} \tiny{$\pm$ 0.0049} & \textbf{0.1826} \tiny{$\pm$ 0.0030} & \textbf{0.2354} \tiny{$\pm$ 0.0048} & \textbf{0.1753} \tiny{$\pm$ 0.0042} \\
& $\amone$ & 0.0322 \tiny{$\pm$ 0.0029} & 0.0385 \tiny{$\pm$ 0.0018} & -0.1477 \tiny{$\pm$ 0.0045} & -0.0709 \tiny{$\pm$ 0.0031} \\
& $\mathsf{Norm}$ & 0.0725 \hspace{2.2em} & 0.0235 \hspace{2.2em} & \underline{0.1659} \hspace{2.2em} & \underline{0.1231} \hspace{2.2em} \\
& $\mathsf{FV}$ & \underline{0.1439} \hspace{2.2em} & \underline{0.1003} \hspace{2.2em} & -0.1727 \hspace{2.2em} & -0.1284 \hspace{2.2em} \\
\midrule
\multirow{5}{*}{MoCo} & $\mathsf{NC}_{100}$ & \textbf{0.4157} \tiny{$\pm$ 0.0042} & \textbf{0.3653} \tiny{$\pm$ 0.0038} & \textbf{0.3632} \tiny{$\pm$ 0.0019} & \textbf{0.3128} \tiny{$\pm$ 0.0024} \\
& $\amone$ & \underline{0.3523} \tiny{$\pm$ 0.0048} & \underline{0.3093} \tiny{$\pm$ 0.0037} & \underline{0.2246} \tiny{$\pm$ 0.0056} & \underline{0.1925} \tiny{$\pm$ 0.0070} \\
& $\mathsf{Norm}$ & 0.3238 \hspace{2.2em} & 0.2708 \hspace{2.2em} & 0.2186 \hspace{2.2em} & 0.1432 \hspace{2.2em} \\
& $\mathsf{FV}$ & 0.0519 \hspace{2.2em} & 0.0321 \hspace{2.2em} & -0.1420 \hspace{2.2em} & -0.1578 \hspace{2.2em} \\
\bottomrule
\end{tabular}

~
\vspace{1em}
~

\begin{tabular}{cc|cccc}
\toprule
\multirow{3}{*}{\makecell[c]{Pretraining \\ Algorithms}} & \multirow{3}{*}{Method} & \multicolumn{4}{c}{ResNet-50}  \\  \cline{3-6} 
 & & \multicolumn{2}{c}{\texttt{CIFAR-10}}  & \multicolumn{2}{c}{\texttt{CIFAR-100}} \\
 & & Entropy  & Brier & Entropy & Brier  \\
\midrule
\multirow{5}{*}{SimCLR} & $\mathsf{NC}_{100}$ & \textbf{0.3786} \tiny{$\pm$ 0.0013} & \textbf{0.3206} \tiny{$\pm$ 0.0014} & \textbf{0.3070} \tiny{$\pm$ 0.0040} & \textbf{0.2516} \tiny{$\pm$ 0.0036} \\
& $\amone$ & 0.2750 \tiny{$\pm$ 0.0052} & 0.2376 \tiny{$\pm$ 0.0044} & 0.1100 \tiny{$\pm$ 0.0060} & \underline{0.0919} \tiny{$\pm$ 0.0061} \\
& $\mathsf{Norm}$ & \underline{0.2789} \hspace{2.2em} & \underline{0.2385} \hspace{2.2em} & \underline{0.1311} \hspace{2.2em} & 0.0710 \hspace{2.2em} \\
& $\mathsf{FV}$ & -0.0402 \hspace{2.2em} & -0.0393 \hspace{2.2em} & -0.1984 \hspace{2.2em} & -0.1870 \hspace{2.2em} \\
\midrule
\multirow{5}{*}{BYOL} & $\mathsf{NC}_{100}$ & \textbf{0.3524} \tiny{$\pm$ 0.0028} & \textbf{0.2131} \tiny{$\pm$ 0.0014} & \textbf{0.3233} \tiny{$\pm$ 0.0061} & \textbf{0.1733} \tiny{$\pm$ 0.0039} \\
& $\amone$ & -0.1858 \tiny{$\pm$ 0.0022} & -0.1442 \tiny{$\pm$ 0.0010} & -0.3124 \tiny{$\pm$ 0.0019} & -0.1598 \tiny{$\pm$ 0.0014} \\
& $\mathsf{Norm}$ & \underline{0.1990} \hspace{2.2em} & \underline{0.1184} \hspace{2.2em} & \underline{0.1990} \hspace{2.2em} & \underline{0.1476} \hspace{2.2em} \\
& $\mathsf{FV}$ & 0.0676 \hspace{2.2em} & 0.0524 \hspace{2.2em} & -0.0213 \hspace{2.2em} & -0.0625 \hspace{2.2em} \\
\midrule
\multirow{5}{*}{MoCo} & $\mathsf{NC}_{100}$ & 0.3757 \tiny{$\pm$ 0.0041} & 0.3236 \tiny{$\pm$ 0.0038} & \textbf{0.2667} \tiny{$\pm$ 0.0043} & \textbf{0.2225} \tiny{$\pm$ 0.0028} \\
& $\amone$ & \textbf{0.3831} \tiny{$\pm$ 0.0032} & \textbf{0.3397} \tiny{$\pm$ 0.0027} & \underline{0.2453} \tiny{$\pm$ 0.0051} & \underline{0.2183} \tiny{$\pm$ 0.0054} \\
& $\mathsf{Norm}$ & \underline{0.3798} \hspace{2.2em} & \underline{0.3294} \hspace{2.2em} & 0.2172 \hspace{2.2em} & 0.1659 \hspace{2.2em} \\
& $\mathsf{FV}$ & 0.1228 \hspace{2.2em} & 0.1122 \hspace{2.2em} & -0.1081 \hspace{2.2em} & -0.1100 \hspace{2.2em} \\
\bottomrule
\end{tabular}
\label{tab:comparison_full}
\end{table*}
\clearpage

\begin{table*}[t!] 
\centering
\fontsize{9pt}{9pt}\selectfont
\caption{Comparison on the correlation between our neighborhood consistency ($\mathsf{NC}_{100}$) and baseline methods in relation to performance on transfer learning tasks from $\texttt{TinyImagenet}$ to $\texttt{CIFAR-10}$, $\texttt{CIFAR-100}$, and $\texttt{STL-10}$.}
\renewcommand{\arraystretch}{1.25}
\resizebox{\textwidth}{!}{%
\begin{tabular}{cc|cccccc}
\toprule
\multirow{4}{*}{\makecell[c]{Pretraining \\ Algorithms}} & \multirow{4}{*}{Method} & \multicolumn{6}{c}{ResNet-18} \\  \cline{3-8} 
 & & \multicolumn{6}{c}{$\texttt{TinyImagenet} \rightarrow$} \\
  & & \multicolumn{2}{c}{\texttt{CIFAR-10}}  & \multicolumn{2}{c}{\texttt{CIFAR-100}} & \multicolumn{2}{c}{\texttt{STL-10}} \\
  & & Entropy  & Brier & Entropy & Brier & Entropy  & Brier \\
\midrule
\multirow{5}{*}{SimCLR} & $\mathsf{NC}_{100}$ & \underline{0.1729} \tiny{$\pm$ 0.0090} & \underline{0.1292} \tiny{$\pm$ 0.0089} & \textbf{0.1680} \tiny{$\pm$ 0.0032} & \textbf{0.1343} \tiny{$\pm$ 0.0032} & \textbf{0.2176} \tiny{$\pm$ 0.0112} & \textbf{0.1513} \tiny{$\pm$ 0.0106} \\
& $\amone$ & 0.1311 \tiny{$\pm$ 0.0127} & 0.0933 \tiny{$\pm$ 0.0115} & 0.0138 \tiny{$\pm$ 0.0071} & -0.0251 \tiny{$\pm$ 0.0049} & 0.1098 \tiny{$\pm$ 0.0151} & 0.0520 \tiny{$\pm$ 0.0116} \\
& $\mathsf{Norm}$ & \textbf{0.2124} \hspace{2.2em} & \textbf{0.1642} \hspace{2.2em} & \underline{0.1220} \hspace{2.2em} & \underline{0.0507} \hspace{2.2em} & \underline{0.1641} \hspace{2.2em} & \underline{0.0852} \hspace{2.2em} \\
& $\mathsf{FV}$ & -0.0919 \hspace{2.2em} & -0.0779 \hspace{2.2em} & -0.1660 \hspace{2.2em} & -0.1668 \hspace{2.2em} & -0.1872 \hspace{2.2em} & -0.1593 \hspace{2.2em} \\
\midrule
\multirow{5}{*}{BYOL} & $\mathsf{NC}_{100}$ & \textbf{0.2557} \tiny{$\pm$ 0.0080} & \textbf{0.1431} \tiny{$\pm$ 0.0049} & \textbf{0.2352} \tiny{$\pm$ 0.0027} & \textbf{0.1525} \tiny{$\pm$ 0.0028} & \underline{0.0374} \tiny{$\pm$ 0.0101} & \underline{0.0232} \tiny{$\pm$ 0.0052} \\
& $\amone$ & -0.1458 \tiny{$\pm$ 0.0073} & -0.0923 \tiny{$\pm$ 0.0048} & -0.2716 \tiny{$\pm$ 0.0053} & -0.1525 \tiny{$\pm$ 0.0057} & -0.2893 \tiny{$\pm$ 0.0089} & -0.2178 \tiny{$\pm$ 0.0071} \\
& $\mathsf{Norm}$ & \underline{0.1483} \hspace{2.2em} & \underline{0.0749} \hspace{2.2em} & \underline{0.1879} \hspace{2.2em} & \underline{0.1315} \hspace{2.2em} & \textbf{0.1021} \hspace{2.2em} & \textbf{0.0772} \hspace{2.2em} \\
& $\mathsf{FV}$ & -0.1386 \hspace{2.2em} & -0.0713 \hspace{2.2em} & -0.1874 \hspace{2.2em} & -0.1202 \hspace{2.2em} & -0.1171 \hspace{2.2em} & -0.0790 \hspace{2.2em} \\
\midrule
\multirow{5}{*}{MoCo} & $\mathsf{NC}_{100}$ & \underline{0.1880} \tiny{$\pm$ 0.0080} & \underline{0.1313} \tiny{$\pm$ 0.0071} & \textbf{0.1797} \tiny{$\pm$ 0.0054} & \textbf{0.1446} \tiny{$\pm$ 0.0056} & \textbf{0.2263} \tiny{$\pm$ 0.0128} & \textbf{0.1463} \tiny{$\pm$ 0.0129} \\
& $\amone$ & 0.1213 \tiny{$\pm$ 0.0097} & 0.0718 \tiny{$\pm$ 0.0093} & 0.0311 \tiny{$\pm$ 0.0060} & -0.0143 \tiny{$\pm$ 0.0057} & 0.0993 \tiny{$\pm$ 0.0142} & 0.0223 \tiny{$\pm$ 0.0101} \\
& $\mathsf{Norm}$ & \textbf{0.1930} \hspace{2.2em} & \textbf{0.1355} \hspace{2.2em} & \underline{0.1321} \hspace{2.2em} & \underline{0.0531} \hspace{2.2em} & \underline{0.1327} \hspace{2.2em} & \underline{0.0443} \hspace{2.2em} \\
& $\mathsf{FV}$ & -0.0793 \hspace{2.2em} & -0.0701 \hspace{2.2em} & -0.1499 \hspace{2.2em} & -0.1658 \hspace{2.2em} & -0.1615 \hspace{2.2em} & -0.1577 \hspace{2.2em} \\
\bottomrule
\end{tabular}
}

~
\vspace{1em}
~

\resizebox{\textwidth}{!}{%
\begin{tabular}{cc|cccccc}
\toprule
\multirow{4}{*}{\makecell[c]{Pretraining \\ Algorithms}} & \multirow{4}{*}{Method} & \multicolumn{6}{c}{ResNet-50} \\  \cline{3-8} 
 & & \multicolumn{6}{c}{$\texttt{TinyImagenet} \rightarrow$} \\
  & & \multicolumn{2}{c}{\texttt{CIFAR-10}}  & \multicolumn{2}{c}{\texttt{CIFAR-100}} & \multicolumn{2}{c}{\texttt{STL-10}} \\
  & & Entropy  & Brier & Entropy & Brier & Entropy  & Brier \\
\midrule
\multirow{5}{*}{SimCLR} & $\mathsf{NC}_{100}$ & \underline{0.1224} \tiny{$\pm$ 0.0061} & \underline{0.0804} \tiny{$\pm$ 0.0048} & \textbf{0.1467} \tiny{$\pm$ 0.0075} & \textbf{0.1194} \tiny{$\pm$ 0.0062} & \textbf{0.1866} \tiny{$\pm$ 0.0062} & \textbf{0.1169} \tiny{$\pm$ 0.0055} \\
& $\amone$ & 0.0894 \tiny{$\pm$ 0.0077} & 0.0531 \tiny{$\pm$ 0.0081} & -0.0246 \tiny{$\pm$ 0.0072} & -0.0627 \tiny{$\pm$ 0.0066} & 0.0340 \tiny{$\pm$ 0.0142} & -0.0167 \tiny{$\pm$ 0.0117} \\
& $\mathsf{Norm}$ & \textbf{0.1549} \hspace{2.2em} & \textbf{0.1098} \hspace{2.2em} & \underline{0.0895} \hspace{2.2em} & \underline{0.0177} \hspace{2.2em} & \underline{0.0608} \hspace{2.2em} & \underline{-0.0056} \hspace{2.2em} \\
& $\mathsf{FV}$ & -0.1078 \hspace{2.2em} & -0.0878 \hspace{2.2em} & -0.1867 \hspace{2.2em} & -0.1840 \hspace{2.2em} & -0.2178 \hspace{2.2em} & -0.1829 \hspace{2.2em} \\
\midrule
\multirow{5}{*}{BYOL} & $\mathsf{NC}_{100}$ & \textbf{0.2561} \tiny{$\pm$ 0.0066} & \textbf{0.1258} \tiny{$\pm$ 0.0037} & \textbf{0.2600} \tiny{$\pm$ 0.0028} & \textbf{0.1005} \tiny{$\pm$ 0.0020} & \textbf{0.1410} \tiny{$\pm$ 0.0054} & 0.0123 \tiny{$\pm$ 0.0038} \\
& $\amone$ & -0.1885 \tiny{$\pm$ 0.0016} & -0.1274 \tiny{$\pm$ 0.0016} & -0.3075 \tiny{$\pm$ 0.0019} & -0.1679 \tiny{$\pm$ 0.0040} & -0.2371 \tiny{$\pm$ 0.0022} & -0.1925 \tiny{$\pm$ 0.0020} \\
& $\mathsf{Norm}$ & \underline{0.1620} \hspace{2.2em} & \underline{0.0918} \hspace{2.2em} & \underline{0.1082} \hspace{2.2em} & \underline{0.0816} \hspace{2.2em} & \underline{0.0963} \hspace{2.2em} & \textbf{0.0380} \hspace{2.2em} \\
& $\mathsf{FV}$ & -0.1668 \hspace{2.2em} & -0.0737 \hspace{2.2em} & -0.1441 \hspace{2.2em} & -0.0705 \hspace{2.2em} & -0.0584 \hspace{2.2em} & \underline{0.0205} \hspace{2.2em} \\
\midrule
\multirow{5}{*}{MoCo} & $\mathsf{NC}_{100}$ & \underline{0.1927} \tiny{$\pm$ 0.0045} & \underline{0.1300} \tiny{$\pm$ 0.0040} & \textbf{0.1962} \tiny{$\pm$ 0.0019} & \textbf{0.1560} \tiny{$\pm$ 0.0022} & \underline{0.2170} \tiny{$\pm$ 0.0072} & \underline{0.1433} \tiny{$\pm$ 0.0046} \\
& $\amone$ & 0.1751 \tiny{$\pm$ 0.0081} & 0.1109 \tiny{$\pm$ 0.0077} & 0.0759 \tiny{$\pm$ 0.0054} & 0.0257 \tiny{$\pm$ 0.0045} & 0.1696 \tiny{$\pm$ 0.0106} & 0.0710 \tiny{$\pm$ 0.0089} \\
& $\mathsf{Norm}$ & \textbf{0.2503} \hspace{2.2em} & \textbf{0.1803} \hspace{2.2em} & \underline{0.1838} \hspace{2.2em} & \underline{0.1034} \hspace{2.2em} & \textbf{0.2524} \hspace{2.2em} & \textbf{0.1524} \hspace{2.2em} \\
& $\mathsf{FV}$ & -0.0560 \hspace{2.2em} & -0.0400 \hspace{2.2em} & -0.1510 \hspace{2.2em} & -0.1578 \hspace{2.2em} & -0.1735 \hspace{2.2em} & -0.1577 \hspace{2.2em} \\
\bottomrule
\end{tabular}
}
\label{tab:trasnfer_full}
\end{table*}
\clearpage

\begin{table*}[t!] 
\centering
\fontsize{9pt}{9pt}\selectfont
\caption{Comparison on the correlation between our neighborhood consistency ($\mathsf{NC}_{100}$) and baseline methods in relation to performance on in-distribution downstream tasks, for the \textbf{multi-class} downstream classification tasks. Our approach consistently demonstrates strong performance when compared to baseline methods, similar to the results observed when using the OVO binary classification tasks, as shown in Table~\ref{tab:comparison_full}.}
\renewcommand{\arraystretch}{1.25}
\begin{tabular}{cc|cccc}
\toprule
\multirow{3}{*}{\makecell[c]{Pretraining \\ Algorithms}} & \multirow{3}{*}{Method} & \multicolumn{4}{c}{ResNet-18}  \\  \cline{3-6} 
 & & \multicolumn{2}{c}{\texttt{CIFAR-10}}  & \multicolumn{2}{c}{\texttt{CIFAR-100}} \\
 & & Entropy  & Brier & Entropy & Brier  \\
\midrule
\multirow{5}{*}{SimCLR} & $\mathsf{NC}_{100}$ & \textbf{0.3848} \tiny{$\pm$ 0.0016} & \textbf{0.3202} \tiny{$\pm$ 0.0019} & \textbf{0.2998} \tiny{$\pm$ 0.0051} & \textbf{0.2166} \tiny{$\pm$ 0.0030} \\
& $\amone$ & 0.3025 \tiny{$\pm$ 0.0079} & \underline{0.2533} \tiny{$\pm$ 0.0055} & 0.0641 \tiny{$\pm$ 0.0070} & \underline{0.0698} \tiny{$\pm$ 0.0072} \\
& $\mathsf{Norm}$ & \underline{0.3098} \hspace{2.2em} & 0.2462 \hspace{2.2em} & \underline{0.1725} \hspace{2.2em} & 0.0467 \hspace{2.2em} \\
& $\mathsf{FV}$ & -0.0485 \hspace{2.2em} & -0.0357 \hspace{2.2em} & -0.1755 \hspace{2.2em} & -0.1358 \hspace{2.2em} \\
\midrule
\multirow{5}{*}{BYOL} & $\mathsf{NC}_{100}$ & \textbf{0.3375} \tiny{$\pm$ 0.0047} & \textbf{0.1431} \tiny{$\pm$ 0.0030} & \textbf{0.3116} \tiny{$\pm$ 0.0056} & \textbf{0.0957} \tiny{$\pm$ 0.0034} \\
& $\amone$ & 0.0466 \tiny{$\pm$ 0.0028} & 0.0540 \tiny{$\pm$ 0.0021} & -0.2296 \tiny{$\pm$ 0.0042} & 0.0093 \tiny{$\pm$ 0.0034} \\
& $\mathsf{Norm}$ & \underline{0.1723} \hspace{2.2em} & 0.0261 \hspace{2.2em} & \underline{0.2030} \hspace{2.2em} & \underline{0.0808} \hspace{2.2em} \\
& $\mathsf{FV}$ & 0.1620 \hspace{2.2em} & \underline{0.0621} \hspace{2.2em} & -0.2077 \hspace{2.2em} & -0.0791 \hspace{2.2em} \\
\midrule
\multirow{5}{*}{MoCo} & $\mathsf{NC}_{100}$ & \textbf{0.3972} \tiny{$\pm$ 0.0039} & \textbf{0.3451} \tiny{$\pm$ 0.0034} & \textbf{0.3519} \tiny{$\pm$ 0.0019} & \textbf{0.2619} \tiny{$\pm$ 0.0034} \\
& $\amone$ & \underline{0.3267} \tiny{$\pm$ 0.0060} & \underline{0.2903} \tiny{$\pm$ 0.0036} & 0.1942 \tiny{$\pm$ 0.0048} & \underline{0.1669} \tiny{$\pm$ 0.0067} \\
& $\mathsf{Norm}$ & 0.3177 \hspace{2.2em} & 0.2626 \hspace{2.2em} & \underline{0.2273} \hspace{2.2em} & 0.1046 \hspace{2.2em} \\
& $\mathsf{FV}$ & 0.0541 \hspace{2.2em} & 0.0401 \hspace{2.2em} & -0.1534 \hspace{2.2em} & -0.1386 \hspace{2.2em} \\
\bottomrule
\end{tabular}

~
\vspace{1em}
~

\begin{tabular}{cc|cccc}
\toprule
\multirow{3}{*}{\makecell[c]{Pretraining \\ Algorithms}} & \multirow{3}{*}{Method} & \multicolumn{4}{c}{ResNet-50}  \\  \cline{3-6} 
 & & \multicolumn{2}{c}{\texttt{CIFAR-10}}  & \multicolumn{2}{c}{\texttt{CIFAR-100}} \\
 & & Entropy  & Brier & Entropy & Brier  \\
\midrule
\multirow{5}{*}{SimCLR} & $\mathsf{NC}_{100}$ & \textbf{0.3756} \tiny{$\pm$ 0.0005} & \textbf{0.3143} \tiny{$\pm$ 0.0012} & \textbf{0.2969} \tiny{$\pm$ 0.0052} & \textbf{0.2091} \tiny{$\pm$ 0.0038} \\
& $\amone$ & 0.2709 \tiny{$\pm$ 0.0061} & 0.2363 \tiny{$\pm$ 0.0047} & 0.0954 \tiny{$\pm$ 0.0055} & \underline{0.0875} \tiny{$\pm$ 0.0058} \\
& $\mathsf{Norm}$ & \underline{0.2928} \hspace{2.2em} & \underline{0.2387} \hspace{2.2em} & \underline{0.1888} \hspace{2.2em} & 0.0493 \hspace{2.2em} \\
& $\mathsf{FV}$ & -0.0425 \hspace{2.2em} & -0.0319 \hspace{2.2em} & -0.1856 \hspace{2.2em} & -0.1531 \hspace{2.2em} \\
\midrule
\multirow{5}{*}{BYOL} & $\mathsf{NC}_{100}$ & \textbf{0.4480} \tiny{$\pm$ 0.0038} & \textbf{0.1542} \tiny{$\pm$ 0.0009} & \textbf{0.4393} \tiny{$\pm$ 0.0079} & \underline{0.0670} \tiny{$\pm$ 0.0017} \\
& $\amone$ & -0.2822 \tiny{$\pm$ 0.0022} & -0.1029 \tiny{$\pm$ 0.0014} & -0.4107 \tiny{$\pm$ 0.0028} & -0.0475 \tiny{$\pm$ 0.0012} \\
& $\mathsf{Norm}$ & \underline{0.2253} \hspace{2.2em} & \underline{0.0941} \hspace{2.2em} & \underline{0.2490} \hspace{2.2em} & \textbf{0.0760} \hspace{2.2em} \\
& $\mathsf{FV}$ & 0.1362 \hspace{2.2em} & 0.0252 \hspace{2.2em} & 0.0017 \hspace{2.2em} & -0.0464 \hspace{2.2em} \\
\midrule
\multirow{5}{*}{MoCo} & $\mathsf{NC}_{100}$ & \underline{0.3537} \tiny{$\pm$ 0.0044} & 0.3002 \tiny{$\pm$ 0.0035} & \textbf{0.2713} \tiny{$\pm$ 0.0049} & \underline{0.1857} \tiny{$\pm$ 0.0026} \\
& $\amone$ & 0.3532 \tiny{$\pm$ 0.0046} & \textbf{0.3161} \tiny{$\pm$ 0.0027} & 0.2199 \tiny{$\pm$ 0.0049} & \textbf{0.1904} \tiny{$\pm$ 0.0053} \\
& $\mathsf{Norm}$ & \textbf{0.3764} \hspace{2.2em} & \underline{0.3156} \hspace{2.2em} & \underline{0.2321} \hspace{2.2em} & 0.1270 \hspace{2.2em} \\
& $\mathsf{FV}$ & 0.1279 \hspace{2.2em} & 0.1196 \hspace{2.2em} & -0.1226 \hspace{2.2em} & -0.0945 \hspace{2.2em} \\
\bottomrule
\end{tabular}
\label{tab:comparison_multi}
\end{table*}
\clearpage

\begin{table*}[t!] 
\centering
\fontsize{9pt}{9pt}\selectfont
\caption{Comparison on the correlation between our neighborhood consistency ($\mathsf{NC}_{100}$) and baseline methods in relation to performance on transfer learning tasks from $\texttt{TinyImagenet}$ to $\texttt{CIFAR-10}$, $\texttt{CIFAR-100}$, and $\texttt{STL-10}$, for the \textbf{multi-class} downstream classification tasks.}
\renewcommand{\arraystretch}{1.25}
\resizebox{\textwidth}{!}{%
\begin{tabular}{cc|cccccc}
\toprule
\multirow{4}{*}{\makecell[c]{Pretraining \\ Algorithms}} & \multirow{4}{*}{Method} & \multicolumn{6}{c}{ResNet-18} \\  \cline{3-8} 
 & & \multicolumn{6}{c}{$\texttt{TinyImagenet} \rightarrow$} \\
  & & \multicolumn{2}{c}{\texttt{CIFAR-10}}  & \multicolumn{2}{c}{\texttt{CIFAR-100}} & \multicolumn{2}{c}{\texttt{STL-10}} \\
  & & Entropy  & Brier & Entropy & Brier & Entropy  & Brier \\
\midrule
\multirow{5}{*}{SimCLR} & $\mathsf{NC}_{100}$ & \underline{0.1673} \tiny{$\pm$ 0.0080} & \underline{0.1230} \tiny{$\pm$ 0.0073} & \underline{0.1486} \tiny{$\pm$ 0.0046} & \textbf{0.1031} \tiny{$\pm$ 0.0018} & \textbf{0.2012} \tiny{$\pm$ 0.0106} & \textbf{0.1271} \tiny{$\pm$ 0.0100} \\
& $\amone$ & 0.1331 \tiny{$\pm$ 0.0122} & 0.0982 \tiny{$\pm$ 0.0116} & 0.0082 \tiny{$\pm$ 0.0064} & -0.0162 \tiny{$\pm$ 0.0051} & 0.0633 \tiny{$\pm$ 0.0122} & 0.0382 \tiny{$\pm$ 0.0107} \\
& $\mathsf{Norm}$ & \textbf{0.2294} \hspace{2.2em} & \textbf{0.1654} \hspace{2.2em} & \textbf{0.1527} \hspace{2.2em} & \underline{0.0348} \hspace{2.2em} & \underline{0.1431} \hspace{2.2em} & \underline{0.0624} \hspace{2.2em} \\
& $\mathsf{FV}$ & -0.0865 \hspace{2.2em} & -0.0627 \hspace{2.2em} & -0.1452 \hspace{2.2em} & -0.1268 \hspace{2.2em} & -0.2071 \hspace{2.2em} & -0.1423 \hspace{2.2em} \\
\midrule
\multirow{5}{*}{BYOL} & $\mathsf{NC}_{100}$ & \textbf{0.3447} \tiny{$\pm$ 0.0110} & \textbf{0.0960} \tiny{$\pm$ 0.0039} & \textbf{0.3263} \tiny{$\pm$ 0.0039} & \underline{0.0615} \tiny{$\pm$ 0.0029} & \underline{0.0867} \tiny{$\pm$ 0.0115} & \underline{-0.0033} \tiny{$\pm$ 0.0044} \\
& $\amone$ & -0.2253 \tiny{$\pm$ 0.0086} & -0.0370 \tiny{$\pm$ 0.0033} & -0.3650 \tiny{$\pm$ 0.0054} & -0.0531 \tiny{$\pm$ 0.0047} & -0.3468 \tiny{$\pm$ 0.0088} & -0.1432 \tiny{$\pm$ 0.0079} \\
& $\mathsf{Norm}$ & \underline{0.1897} \hspace{2.2em} & \underline{0.0630} \hspace{2.2em} & \underline{0.2239} \hspace{2.2em} & \textbf{0.0638} \hspace{2.2em} & \textbf{0.1282} \hspace{2.2em} & \textbf{0.0488} \hspace{2.2em} \\
& $\mathsf{FV}$ & -0.1990 \hspace{2.2em} & -0.0564 \hspace{2.2em} & -0.2246 \hspace{2.2em} & -0.0565 \hspace{2.2em} & -0.1513 \hspace{2.2em} & -0.0462 \hspace{2.2em} \\
\midrule
\multirow{5}{*}{MoCo} & $\mathsf{NC}_{100}$ & \underline{0.1674} \tiny{$\pm$ 0.0074} & \underline{0.1184} \tiny{$\pm$ 0.0062} & \textbf{0.1631} \tiny{$\pm$ 0.0053} & \textbf{0.1146} \tiny{$\pm$ 0.0054} & \textbf{0.2047} \tiny{$\pm$ 0.0136} & \textbf{0.1110} \tiny{$\pm$ 0.0122} \\
& $\amone$ & 0.0927 \tiny{$\pm$ 0.0084} & 0.0649 \tiny{$\pm$ 0.0089} & 0.0019 \tiny{$\pm$ 0.0053} & -0.0195 \tiny{$\pm$ 0.0055} & 0.0443 \tiny{$\pm$ 0.0120} & 0.0022 \tiny{$\pm$ 0.0082} \\
& $\mathsf{Norm}$ & \textbf{0.1931} \hspace{2.2em} & \textbf{0.1304} \hspace{2.2em} & \underline{0.1458} \hspace{2.2em} & \underline{0.0221} \hspace{2.2em} & \underline{0.1163} \hspace{2.2em} & \underline{0.0235} \hspace{2.2em} \\
& $\mathsf{FV}$ & -0.0811 \hspace{2.2em} & -0.0561 \hspace{2.2em} & -0.1582 \hspace{2.2em} & -0.1415 \hspace{2.2em} & -0.1919 \hspace{2.2em} & -0.1371 \hspace{2.2em} \\
\bottomrule
\end{tabular}
}

~
\vspace{1em}
~

\resizebox{\textwidth}{!}{%
\begin{tabular}{cc|cccccc}
\toprule
\multirow{4}{*}{\makecell[c]{Pretraining \\ Algorithms}} & \multirow{4}{*}{Method} & \multicolumn{6}{c}{ResNet-50} \\  \cline{3-8} 
 & & \multicolumn{6}{c}{$\texttt{TinyImagenet} \rightarrow$} \\
  & & \multicolumn{2}{c}{\texttt{CIFAR-10}}  & \multicolumn{2}{c}{\texttt{CIFAR-100}} & \multicolumn{2}{c}{\texttt{STL-10}} \\
  & & Entropy  & Brier & Entropy & Brier & Entropy  & Brier \\
\midrule
\multirow{5}{*}{SimCLR} & $\mathsf{NC}_{100}$ & \underline{0.1205} \tiny{$\pm$ 0.0055} & \underline{0.0788} \tiny{$\pm$ 0.0037} & \textbf{0.1350} \tiny{$\pm$ 0.0079} & \textbf{0.0877} \tiny{$\pm$ 0.0053} & \textbf{0.1693} \tiny{$\pm$ 0.0056} & \textbf{0.0983} \tiny{$\pm$ 0.0042} \\
& $\amone$ & 0.0824 \tiny{$\pm$ 0.0073} & 0.0599 \tiny{$\pm$ 0.0083} & -0.0303 \tiny{$\pm$ 0.0058} & -0.0482 \tiny{$\pm$ 0.0059} & -0.0197 \tiny{$\pm$ 0.0119} & -0.0260 \tiny{$\pm$ 0.0105} \\
& $\mathsf{Norm}$ & \textbf{0.1597} \hspace{2.2em} & \textbf{0.1078} \hspace{2.2em} & \underline{0.1218} \hspace{2.2em} & \underline{0.0022} \hspace{2.2em} & \underline{0.0366} \hspace{2.2em} & \underline{-0.0194} \hspace{2.2em} \\
& $\mathsf{FV}$ & -0.1159 \hspace{2.2em} & -0.0776 \hspace{2.2em} & -0.1817 \hspace{2.2em} & -0.1397 \hspace{2.2em} & -0.2504 \hspace{2.2em} & -0.1702 \hspace{2.2em} \\
\midrule
\multirow{5}{*}{BYOL} & $\mathsf{NC}_{100}$ & \textbf{0.4270} \tiny{$\pm$ 0.0101} & \textbf{0.1025} \tiny{$\pm$ 0.0034} & \textbf{0.4118} \tiny{$\pm$ 0.0071} & \underline{0.0168} \tiny{$\pm$ 0.0012} & \textbf{0.2684} \tiny{$\pm$ 0.0090} & -0.0226 \tiny{$\pm$ 0.0032} \\
& $\amone$ & -0.1754 \tiny{$\pm$ 0.0028} & -0.0533 \tiny{$\pm$ 0.0018} & -0.2933 \tiny{$\pm$ 0.0030} & -0.0617 \tiny{$\pm$ 0.0034} & -0.2536 \tiny{$\pm$ 0.0031} & -0.1122 \tiny{$\pm$ 0.0021} \\
& $\mathsf{Norm}$ & \underline{0.1403} \hspace{2.2em} & \underline{0.0674} \hspace{2.2em} & \underline{0.0850} \hspace{2.2em} & \textbf{0.0276} \hspace{2.2em} & \underline{0.1012} \hspace{2.2em} & \underline{0.0055} \hspace{2.2em} \\
& $\mathsf{FV}$ & -0.2530 \hspace{2.2em} & -0.0784 \hspace{2.2em} & -0.2430 \hspace{2.2em} & -0.0272 \hspace{2.2em} & -0.1099 \hspace{2.2em} & \textbf{0.0267} \hspace{2.2em} \\
\midrule
\multirow{5}{*}{MoCo} & $\mathsf{NC}_{100}$ & \underline{0.1738} \tiny{$\pm$ 0.0051} & \underline{0.1109} \tiny{$\pm$ 0.0041} & \underline{0.1854} \tiny{$\pm$ 0.0009} & \textbf{0.1191} \tiny{$\pm$ 0.0021} & \underline{0.1999} \tiny{$\pm$ 0.0065} & \underline{0.1030} \tiny{$\pm$ 0.0039} \\
& $\amone$ & 0.1372 \tiny{$\pm$ 0.0069} & 0.0979 \tiny{$\pm$ 0.0075} & 0.0415 \tiny{$\pm$ 0.0043} & 0.0077 \tiny{$\pm$ 0.0038} & 0.1079 \tiny{$\pm$ 0.0094} & 0.0350 \tiny{$\pm$ 0.0074} \\
& $\mathsf{Norm}$ & \textbf{0.2553} \hspace{2.2em} & \textbf{0.1715} \hspace{2.2em} & \textbf{0.2025} \hspace{2.2em} & \underline{0.0562} \hspace{2.2em} & \textbf{0.2469} \hspace{2.2em} & \textbf{0.1174} \hspace{2.2em} \\
& $\mathsf{FV}$ & -0.0508 \hspace{2.2em} & -0.0204 \hspace{2.2em} & -0.1617 \hspace{2.2em} & -0.1351 \hspace{2.2em} & -0.2007 \hspace{2.2em} & -0.1293 \hspace{2.2em} \\
\bottomrule
\end{tabular}
}
\label{tab:trasnfer_multi}
\end{table*}
\clearpage

\begin{table*}[t!] 
\centering
\fontsize{9pt}{9pt}\selectfont
\caption{Comparison on the correlation between our neighborhood consistency ($\mathsf{NC}_{100}$) and baseline methods in relation to performance on in-distribution downstream tasks, for the \textbf{unnormalized} representation. The overall correlation is weaker compared to the one observed when using normalized representation as presented in Table~\ref{tab:comparison_full}. Nevertheless, our approach consistently shows a robust performance compared to baseline methods.}
\renewcommand{\arraystretch}{1.25}
\begin{tabular}{cc|cccc}
\toprule
\multirow{3}{*}{\makecell[c]{Pretraining \\ Algorithms}} & \multirow{3}{*}{Method} & \multicolumn{4}{c}{ResNet-18}  \\  \cline{3-6} 
 & & \multicolumn{2}{c}{\texttt{CIFAR-10}}  & \multicolumn{2}{c}{\texttt{CIFAR-100}} \\
 & & Entropy  & Brier & Entropy & Brier  \\
\midrule
\multirow{5}{*}{SimCLR} & $\mathsf{NC}_{100}$ & \textbf{0.3237} \tiny{$\pm$ 0.0023} & \textbf{0.2727} \tiny{$\pm$ 0.0021} & \textbf{0.2868} \tiny{$\pm$ 0.0014} & \textbf{0.2408} \tiny{$\pm$ 0.0021} \\
& $\amone$ & -0.0590 \tiny{$\pm$ 0.0071} & -0.0460 \tiny{$\pm$ 0.0065} & -0.0317 \tiny{$\pm$ 0.0061} & 0.0027 \tiny{$\pm$ 0.0064} \\
& $\mathsf{Norm}$ & \underline{0.2907} \hspace{2.2em} & \underline{0.2429} \hspace{2.2em} & \underline{0.1198} \hspace{2.2em} & \underline{0.0649} \hspace{2.2em} \\
& $\mathsf{FV}$ & -0.3224 \hspace{2.2em} & -0.2704 \hspace{2.2em} & -0.2370 \hspace{2.2em} & -0.1723 \hspace{2.2em} \\
\midrule
\multirow{5}{*}{BYOL} & $\mathsf{NC}_{100}$ & \textbf{0.2381} \tiny{$\pm$ 0.0051} & \textbf{0.1643} \tiny{$\pm$ 0.0037} & \textbf{0.2284} \tiny{$\pm$ 0.0044} & \textbf{0.1720} \tiny{$\pm$ 0.0040} \\
& $\amone$ & -0.0801 \tiny{$\pm$ 0.0016} & -0.0144 \tiny{$\pm$ 0.0022} & -0.2077 \tiny{$\pm$ 0.0038} & -0.1149 \tiny{$\pm$ 0.0025} \\
& $\mathsf{Norm}$ & \underline{0.0725} \hspace{2.2em} & \underline{0.0235} \hspace{2.2em} & \underline{0.1659} \hspace{2.2em} & \underline{0.1231} \hspace{2.2em} \\
& $\mathsf{FV}$ & -0.0700 \hspace{2.2em} & -0.0228 \hspace{2.2em} & -0.1734 \hspace{2.2em} & -0.1286 \hspace{2.2em} \\
\midrule
\multirow{5}{*}{MoCo} & $\mathsf{NC}_{100}$ & \textbf{0.3739} \tiny{$\pm$ 0.0068} & \textbf{0.3181} \tiny{$\pm$ 0.0067} & \textbf{0.3273} \tiny{$\pm$ 0.0042} & \textbf{0.2816} \tiny{$\pm$ 0.0032} \\
& $\amone$ & -0.0741 \tiny{$\pm$ 0.0049} & -0.0506 \tiny{$\pm$ 0.0046} & -0.0305 \tiny{$\pm$ 0.0060} & 0.0137 \tiny{$\pm$ 0.0068} \\
& $\mathsf{Norm}$ & \underline{0.3238} \hspace{2.2em} & \underline{0.2708} \hspace{2.2em} & \underline{0.2186} \hspace{2.2em} & \underline{0.1432} \hspace{2.2em} \\
& $\mathsf{FV}$ & -0.3410 \hspace{2.2em} & -0.2899 \hspace{2.2em} & -0.3155 \hspace{2.2em} & -0.2432 \hspace{2.2em} \\
\bottomrule
\end{tabular}
~
\vspace{1em}
~

\begin{tabular}{cc|cccc}
\toprule
\multirow{3}{*}{\makecell[c]{Pretraining \\ Algorithms}} & \multirow{3}{*}{Method} & \multicolumn{4}{c}{ResNet-50}  \\  \cline{3-6} 
 & & \multicolumn{2}{c}{\texttt{CIFAR-10}}  & \multicolumn{2}{c}{\texttt{CIFAR-100}} \\
 & & Entropy  & Brier & Entropy & Brier  \\
\midrule
\multirow{5}{*}{SimCLR} & $\mathsf{NC}_{100}$ & \textbf{0.3366} \tiny{$\pm$ 0.0062} & \textbf{0.2833} \tiny{$\pm$ 0.0055} & \textbf{0.2635} \tiny{$\pm$ 0.0040} & \textbf{0.2203} \tiny{$\pm$ 0.0033} \\
& $\amone$ & -0.0445 \tiny{$\pm$ 0.0056} & -0.0337 \tiny{$\pm$ 0.0052} & -0.0328 \tiny{$\pm$ 0.0056} & 0.0053 \tiny{$\pm$ 0.0055} \\
& $\mathsf{Norm}$ & \underline{0.2789} \hspace{2.2em} & \underline{0.2385} \hspace{2.2em} & \underline{0.1311} \hspace{2.2em} & \underline{0.0710} \hspace{2.2em} \\
& $\mathsf{FV}$ & -0.3155 \hspace{2.2em} & -0.2706 \hspace{2.2em} & -0.2472 \hspace{2.2em} & -0.1833 \hspace{2.2em} \\
\midrule
\multirow{5}{*}{BYOL} & $\mathsf{NC}_{100}$ & \textbf{0.3704} \tiny{$\pm$ 0.0021} & \textbf{0.2291} \tiny{$\pm$ 0.0012} & \textbf{0.3194} \tiny{$\pm$ 0.0061} & \textbf{0.1737} \tiny{$\pm$ 0.0034} \\
& $\amone$ & -0.2229 \tiny{$\pm$ 0.0023} & -0.1601 \tiny{$\pm$ 0.0011} & -0.3281 \tiny{$\pm$ 0.0021} & -0.1754 \tiny{$\pm$ 0.0013} \\
& $\mathsf{Norm}$ & \underline{0.1990} \hspace{2.2em} & \underline{0.1184} \hspace{2.2em} & \underline{0.1990} \hspace{2.2em} & \underline{0.1476} \hspace{2.2em} \\
& $\mathsf{FV}$ & -0.2001 \hspace{2.2em} & -0.1177 \hspace{2.2em} & -0.1881 \hspace{2.2em} & -0.1433 \hspace{2.2em} \\
\midrule
\multirow{5}{*}{MoCo} & $\mathsf{NC}_{100}$ & \underline{0.3586} \tiny{$\pm$ 0.0049} & \underline{0.3038} \tiny{$\pm$ 0.0060} & \textbf{0.2788} \tiny{$\pm$ 0.0089} & \textbf{0.2247} \tiny{$\pm$ 0.0067} \\
& $\amone$ & -0.1344 \tiny{$\pm$ 0.0046} & -0.1149 \tiny{$\pm$ 0.0047} & -0.0778 \tiny{$\pm$ 0.0067} & -0.0351 \tiny{$\pm$ 0.0070} \\
& $\mathsf{Norm}$ & \textbf{0.3798} \hspace{2.2em} & \textbf{0.3294} \hspace{2.2em} & \underline{0.2172} \hspace{2.2em} & \underline{0.1659} \hspace{2.2em} \\
& $\mathsf{FV}$ & -0.3707 \hspace{2.2em} & -0.3201 \hspace{2.2em} & -0.2604 \hspace{2.2em} & -0.2145 \hspace{2.2em} \\
\bottomrule
\end{tabular}
\label{tab:comparison_unnormalized}
\end{table*}
\clearpage

\begin{table*}[t!] 
\centering
\fontsize{9pt}{9pt}\selectfont
\caption{Comparison on the correlation between our neighborhood consistency ($\mathsf{NC}_{100}$) and baseline methods in relation to performance on transfer learning tasks from $\texttt{TinyImagenet}$ to $\texttt{CIFAR-10}$, $\texttt{CIFAR-100}$, and $\texttt{STL-10}$, for the \textbf{unnormalized} representation.}
\renewcommand{\arraystretch}{1.25}
\resizebox{\textwidth}{!}{%
\begin{tabular}{cc|cccccc}
\toprule
\multirow{4}{*}{\makecell[c]{Pretraining \\ Algorithms}} & \multirow{4}{*}{Method} & \multicolumn{6}{c}{ResNet-18} \\  \cline{3-8} 
 & & \multicolumn{6}{c}{$\texttt{TinyImagenet} \rightarrow$} \\
  & & \multicolumn{2}{c}{\texttt{CIFAR-10}}  & \multicolumn{2}{c}{\texttt{CIFAR-100}} & \multicolumn{2}{c}{\texttt{STL-10}} \\
  & & Entropy  & Brier & Entropy & Brier & Entropy  & Brier \\
\midrule
\multirow{5}{*}{SimCLR} & $\mathsf{NC}_{100}$ & \underline{0.1392} \tiny{$\pm$ 0.0072} & \underline{0.1025} \tiny{$\pm$ 0.0067} & \textbf{0.1520} \tiny{$\pm$ 0.0052} & \textbf{0.1218} \tiny{$\pm$ 0.0057} & \underline{0.1540} \tiny{$\pm$ 0.0103} & \textbf{0.1054} \tiny{$\pm$ 0.0080} \\
& $\amone$ & -0.1318 \tiny{$\pm$ 0.0092} & -0.1101 \tiny{$\pm$ 0.0090} & -0.1067 \tiny{$\pm$ 0.0079} & -0.0671 \tiny{$\pm$ 0.0053} & -0.1211 \tiny{$\pm$ 0.0120} & -0.0838 \tiny{$\pm$ 0.0086} \\
& $\mathsf{Norm}$ & \textbf{0.2124} \hspace{2.2em} & \textbf{0.1642} \hspace{2.2em} & \underline{0.1220} \hspace{2.2em} & \underline{0.0507} \hspace{2.2em} & \textbf{0.1641} \hspace{2.2em} & \underline{0.0852} \hspace{2.2em} \\
& $\mathsf{FV}$ & -0.2816 \hspace{2.2em} & -0.2204 \hspace{2.2em} & -0.2308 \hspace{2.2em} & -0.1655 \hspace{2.2em} & -0.2912 \hspace{2.2em} & -0.2040 \hspace{2.2em} \\
\midrule
\multirow{5}{*}{BYOL} & $\mathsf{NC}_{100}$ & \textbf{0.2551} \tiny{$\pm$ 0.0078} & \textbf{0.1437} \tiny{$\pm$ 0.0047} & \textbf{0.2317} \tiny{$\pm$ 0.0026} & \textbf{0.1505} \tiny{$\pm$ 0.0026} & \underline{0.0354} \tiny{$\pm$ 0.0103} & \underline{0.0234} \tiny{$\pm$ 0.0056} \\
& $\amone$ & -0.1759 \tiny{$\pm$ 0.0052} & -0.1076 \tiny{$\pm$ 0.0032} & -0.2996 \tiny{$\pm$ 0.0045} & -0.1732 \tiny{$\pm$ 0.0046} & -0.3114 \tiny{$\pm$ 0.0080} & -0.2346 \tiny{$\pm$ 0.0061} \\
& $\mathsf{Norm}$ & \underline{0.1483} \hspace{2.2em} & \underline{0.0749} \hspace{2.2em} & \underline{0.1879} \hspace{2.2em} & \underline{0.1315} \hspace{2.2em} & \textbf{0.1021} \hspace{2.2em} & \textbf{0.0772} \hspace{2.2em} \\
& $\mathsf{FV}$ & -0.1461 \hspace{2.2em} & -0.0743 \hspace{2.2em} & -0.1886 \hspace{2.2em} & -0.1296 \hspace{2.2em} & -0.1100 \hspace{2.2em} & -0.0814 \hspace{2.2em} \\
\midrule
\multirow{5}{*}{MoCo} & $\mathsf{NC}_{100}$ & \underline{0.1728} \tiny{$\pm$ 0.0075} & \underline{0.1227} \tiny{$\pm$ 0.0064} & \textbf{0.1738} \tiny{$\pm$ 0.0042} & \textbf{0.1393} \tiny{$\pm$ 0.0041} & \textbf{0.1886} \tiny{$\pm$ 0.0108} & \textbf{0.1229} \tiny{$\pm$ 0.0092} \\
& $\amone$ & -0.0949 \tiny{$\pm$ 0.0081} & -0.0767 \tiny{$\pm$ 0.0084} & -0.0948 \tiny{$\pm$ 0.0071} & -0.0570 \tiny{$\pm$ 0.0058} & -0.0650 \tiny{$\pm$ 0.0112} & -0.0440 \tiny{$\pm$ 0.0076} \\
& $\mathsf{Norm}$ & \textbf{0.1930} \hspace{2.2em} & \textbf{0.1355} \hspace{2.2em} & \underline{0.1321} \hspace{2.2em} & \underline{0.0531} \hspace{2.2em} & \underline{0.1327} \hspace{2.2em} & \underline{0.0443} \hspace{2.2em} \\
& $\mathsf{FV}$ & -0.2590 \hspace{2.2em} & -0.1912 \hspace{2.2em} & -0.2404 \hspace{2.2em} & -0.1730 \hspace{2.2em} & -0.2452 \hspace{2.2em} & -0.1591 \hspace{2.2em} \\
\bottomrule
\end{tabular}
}

~
\vspace{1em}
~

\resizebox{\textwidth}{!}{%
\begin{tabular}{cc|cccccc}
\toprule
\multirow{4}{*}{\makecell[c]{Pretraining \\ Algorithms}} & \multirow{4}{*}{Method} & \multicolumn{6}{c}{ResNet-50} \\  \cline{3-8} 
 & & \multicolumn{6}{c}{$\texttt{TinyImagenet} \rightarrow$} \\
  & & \multicolumn{2}{c}{\texttt{CIFAR-10}}  & \multicolumn{2}{c}{\texttt{CIFAR-100}} & \multicolumn{2}{c}{\texttt{STL-10}} \\
  & & Entropy  & Brier & Entropy & Brier & Entropy  & Brier \\
\midrule
\multirow{5}{*}{SimCLR} & $\mathsf{NC}_{100}$ & \underline{0.1061} \tiny{$\pm$ 0.0104} & \underline{0.0693} \tiny{$\pm$ 0.0091} & \textbf{0.1385} \tiny{$\pm$ 0.0102} & \textbf{0.1144} \tiny{$\pm$ 0.0088} & \textbf{0.1216} \tiny{$\pm$ 0.0154} & \textbf{0.0710} \tiny{$\pm$ 0.0129} \\
& $\amone$ & -0.0954 \tiny{$\pm$ 0.0092} & -0.0787 \tiny{$\pm$ 0.0090} & -0.1092 \tiny{$\pm$ 0.0099} & -0.0748 \tiny{$\pm$ 0.0088} & -0.0775 \tiny{$\pm$ 0.0148} & -0.0544 \tiny{$\pm$ 0.0108} \\
& $\mathsf{Norm}$ & \textbf{0.1549} \hspace{2.2em} & \textbf{0.1098} \hspace{2.2em} & \underline{0.0895} \hspace{2.2em} & \underline{0.0177} \hspace{2.2em} & \underline{0.0608} \hspace{2.2em} & \underline{-0.0056} \hspace{2.2em} \\
& $\mathsf{FV}$ & -0.2344 \hspace{2.2em} & -0.1760 \hspace{2.2em} & -0.2218 \hspace{2.2em} & -0.1556 \hspace{2.2em} & -0.2328 \hspace{2.2em} & -0.1520 \hspace{2.2em} \\
\midrule
\multirow{5}{*}{BYOL} & $\mathsf{NC}_{100}$ & \textbf{0.2578} \tiny{$\pm$ 0.0067} & \textbf{0.1273} \tiny{$\pm$ 0.0039} & \textbf{0.2598} \tiny{$\pm$ 0.0029} & \textbf{0.1014} \tiny{$\pm$ 0.0021} & \textbf{0.1404} \tiny{$\pm$ 0.0058} & \underline{0.0101} \tiny{$\pm$ 0.0038} \\
& $\amone$ & -0.2212 \tiny{$\pm$ 0.0015} & -0.1375 \tiny{$\pm$ 0.0019} & -0.3117 \tiny{$\pm$ 0.0011} & -0.1690 \tiny{$\pm$ 0.0029} & -0.2528 \tiny{$\pm$ 0.0026} & -0.1894 \tiny{$\pm$ 0.0027} \\
& $\mathsf{Norm}$ & \underline{0.1620} \hspace{2.2em} & \underline{0.0918} \hspace{2.2em} & \underline{0.1082} \hspace{2.2em} & \underline{0.0816} \hspace{2.2em} & \underline{0.0963} \hspace{2.2em} & \textbf{0.0380} \hspace{2.2em} \\
& $\mathsf{FV}$ & -0.1729 \hspace{2.2em} & -0.0944 \hspace{2.2em} & -0.1309 \hspace{2.2em} & -0.0868 \hspace{2.2em} & -0.1125 \hspace{2.2em} & -0.0368 \hspace{2.2em} \\
\midrule
\multirow{5}{*}{MoCo} & $\mathsf{NC}_{100}$ & \underline{0.2118} \tiny{$\pm$ 0.0062} & \underline{0.1453} \tiny{$\pm$ 0.0052} & \textbf{0.2091} \tiny{$\pm$ 0.0045} & \textbf{0.1575} \tiny{$\pm$ 0.0035} & \underline{0.1827} \tiny{$\pm$ 0.0088} & \underline{0.1071} \tiny{$\pm$ 0.0094} \\
& $\amone$ & -0.1440 \tiny{$\pm$ 0.0047} & -0.1171 \tiny{$\pm$ 0.0049} & -0.1306 \tiny{$\pm$ 0.0076} & -0.0882 \tiny{$\pm$ 0.0062} & -0.1507 \tiny{$\pm$ 0.0081} & -0.1256 \tiny{$\pm$ 0.0052} \\
& $\mathsf{Norm}$ & \textbf{0.2503} \hspace{2.2em} & \textbf{0.1803} \hspace{2.2em} & \underline{0.1838} \hspace{2.2em} & \underline{0.1034} \hspace{2.2em} & \textbf{0.2524} \hspace{2.2em} & \textbf{0.1524} \hspace{2.2em} \\
& $\mathsf{FV}$ & -0.2906 \hspace{2.2em} & -0.2087 \hspace{2.2em} & -0.2559 \hspace{2.2em} & -0.1829 \hspace{2.2em} & -0.3238 \hspace{2.2em} & -0.2290 \hspace{2.2em} \\
\bottomrule
\end{tabular}
}
\label{tab:transfer_unnormalized}
\end{table*}

\clearpage

\begin{figure}[t] 
    \centering
    \subfigure[]
    {
        \includegraphics[width=0.3\columnwidth]{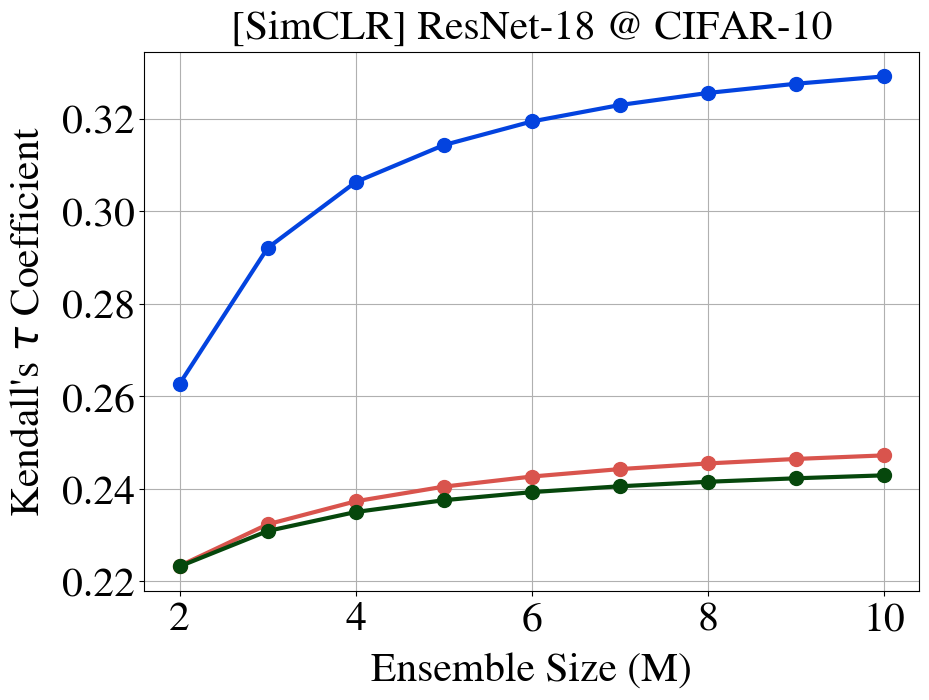}
    }
    \subfigure[]
    {   
        \includegraphics[width=0.3\columnwidth]{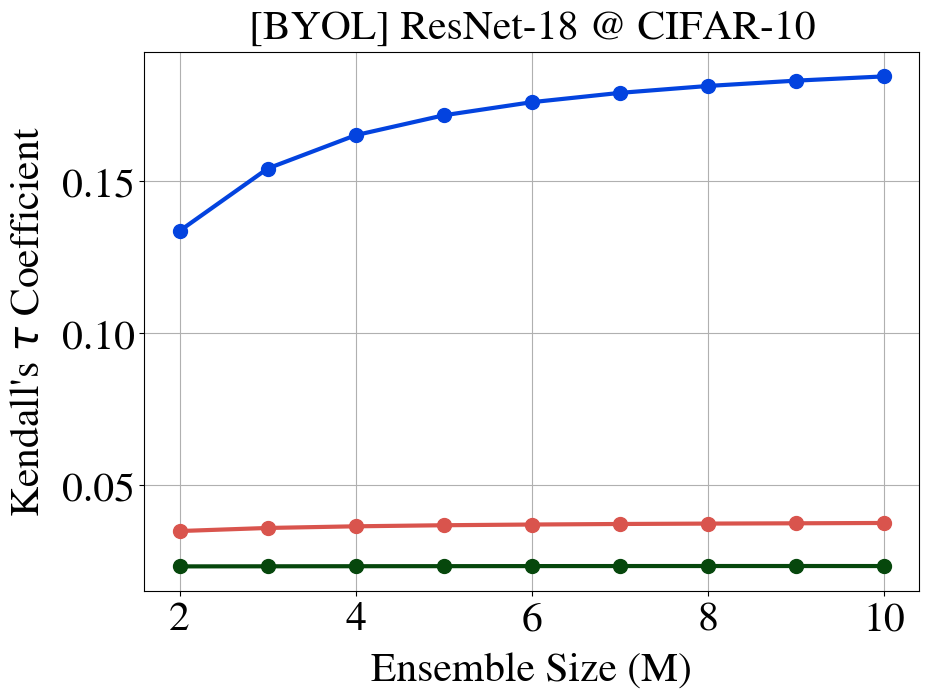}
    }
    \subfigure[]
    {   
        \includegraphics[width=0.3\columnwidth]{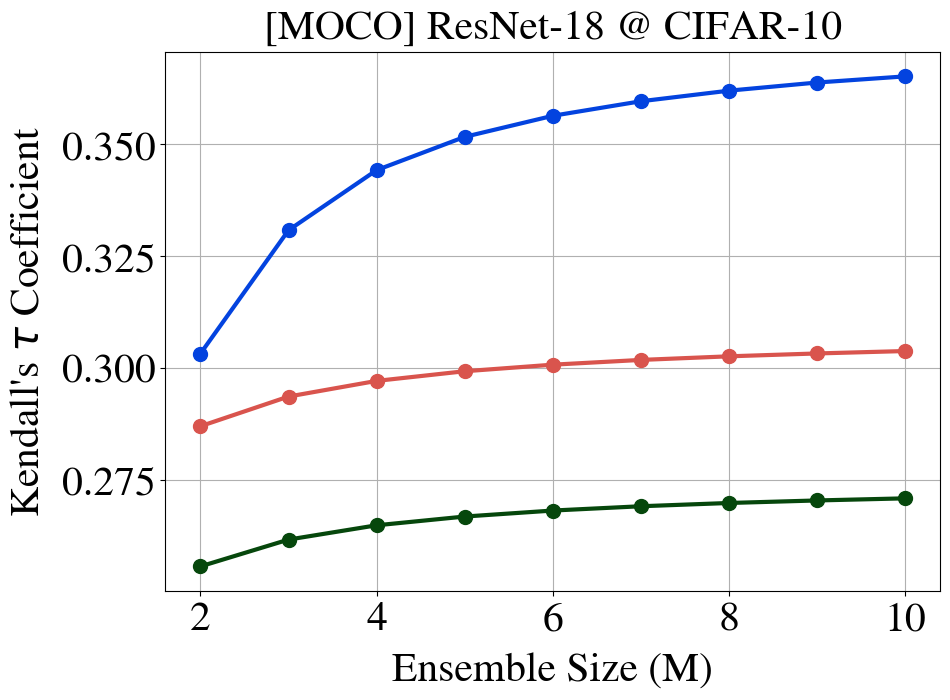}
    }
    \subfigure[]
    {
        \includegraphics[width=0.3\columnwidth]{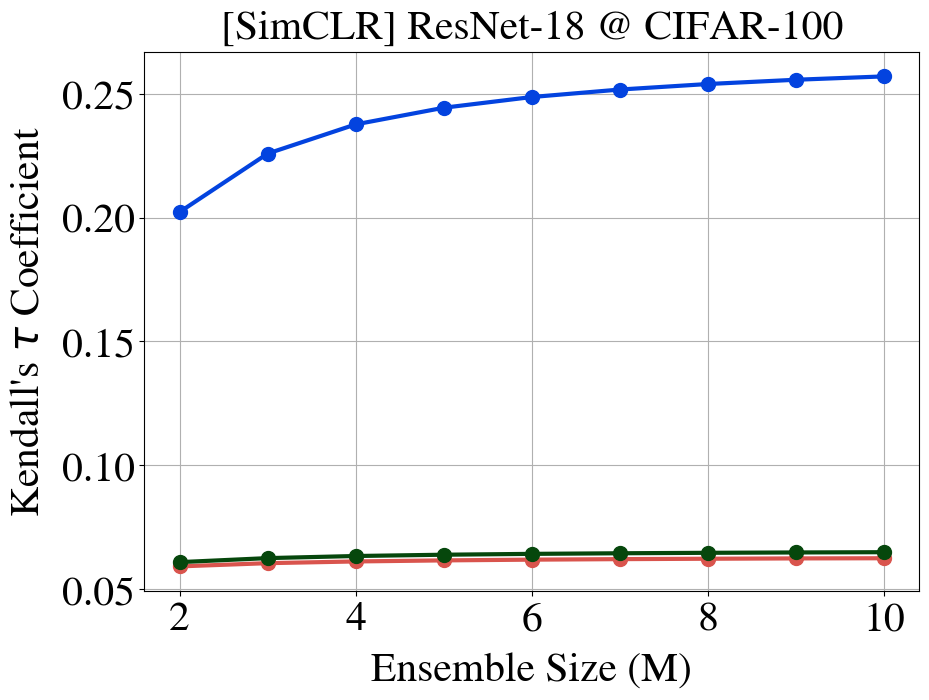}
    }
    \subfigure[]
    {   
        \includegraphics[width=0.3\columnwidth]{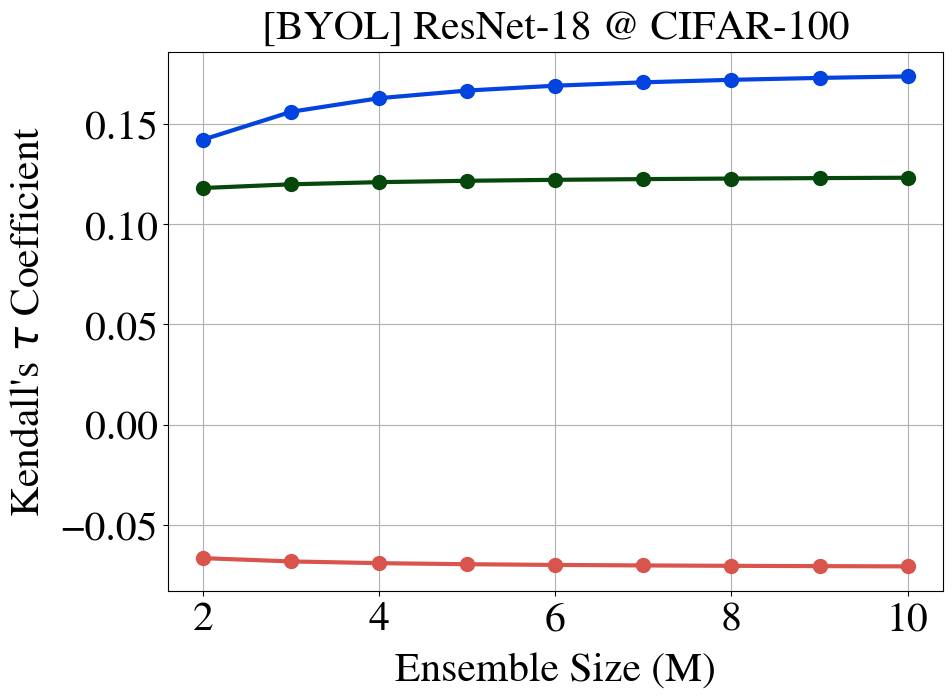}
    }
    \subfigure[]
    {   
        \includegraphics[width=0.3\columnwidth]{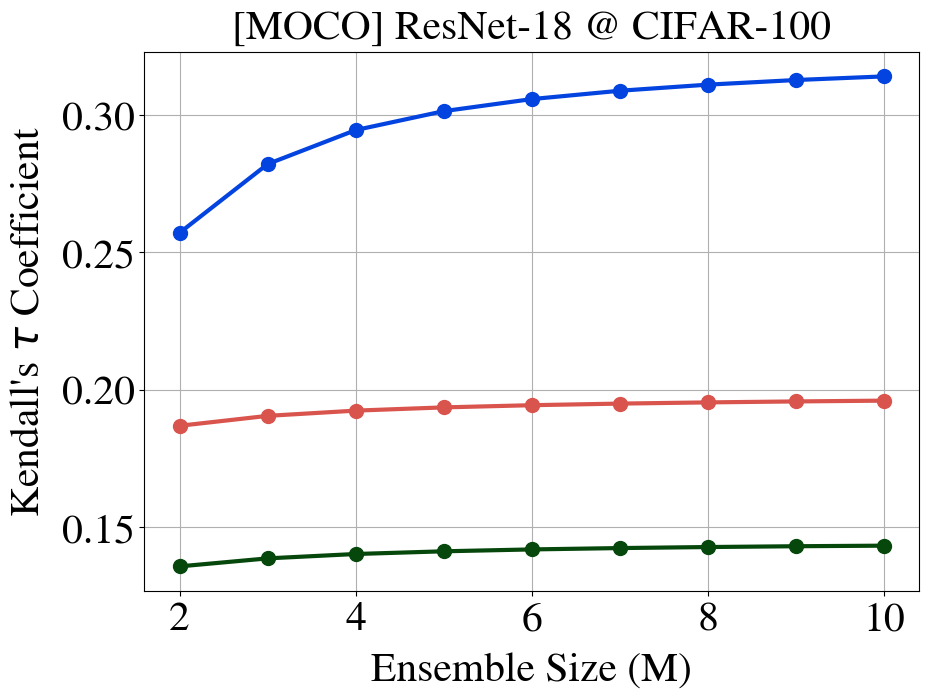}
    }
    \subfigure[]
    {
        \includegraphics[width=0.3\columnwidth]{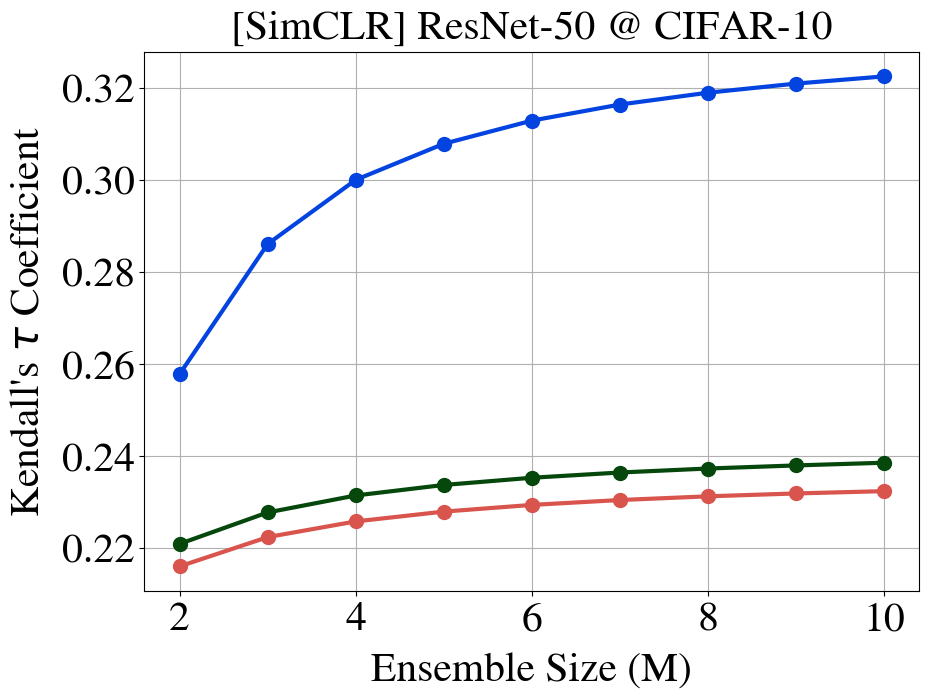}
    }
    \subfigure[]
    {   
        \includegraphics[width=0.3\columnwidth]{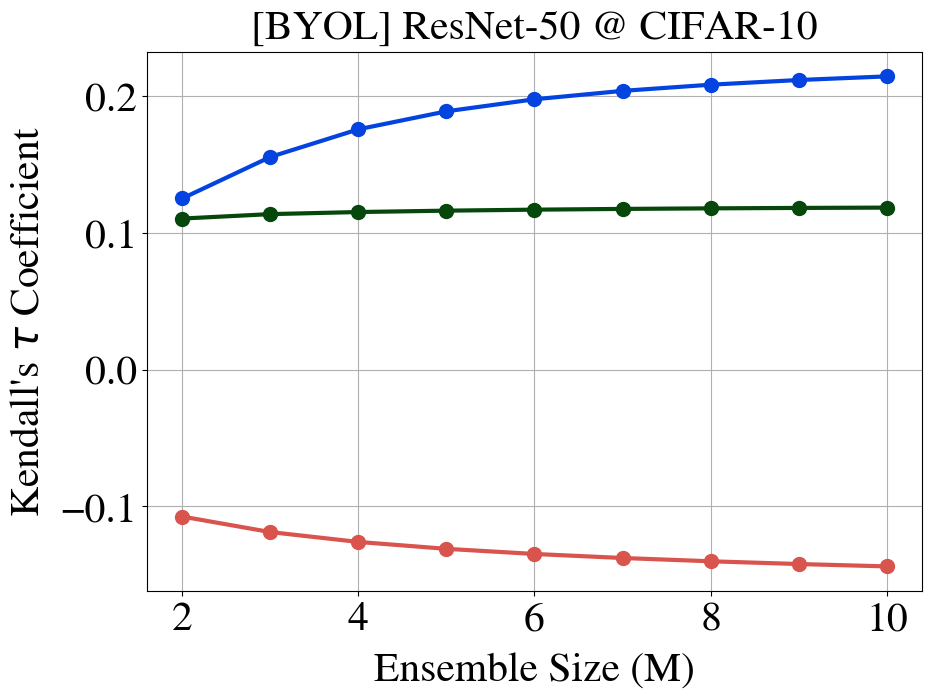}
    }
    \subfigure[]
    {   
        \includegraphics[width=0.3\columnwidth]{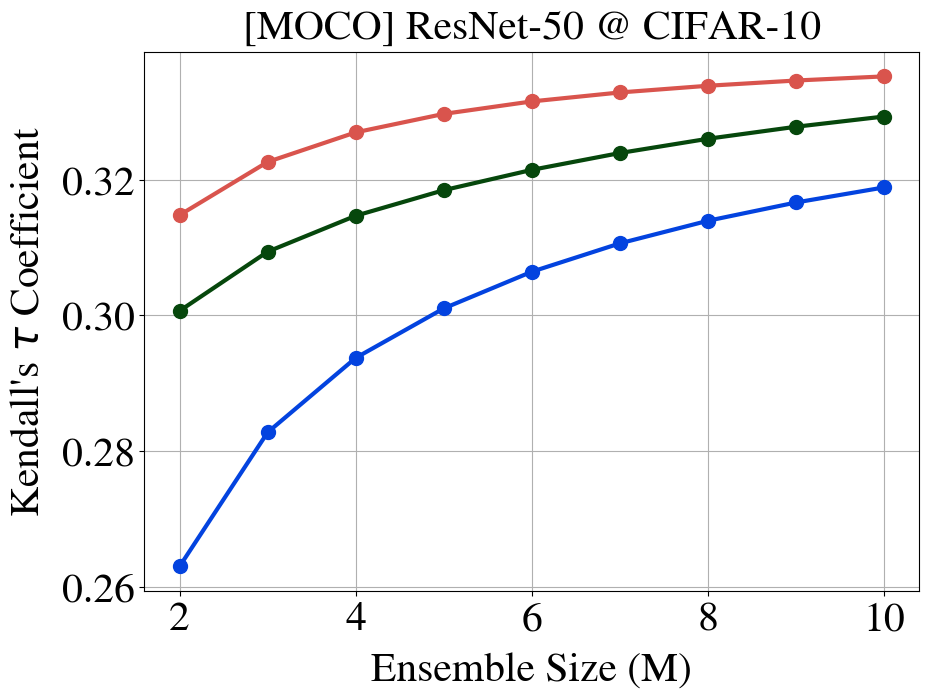}
    }
    \subfigure[]
    {
        \includegraphics[width=0.3\columnwidth]{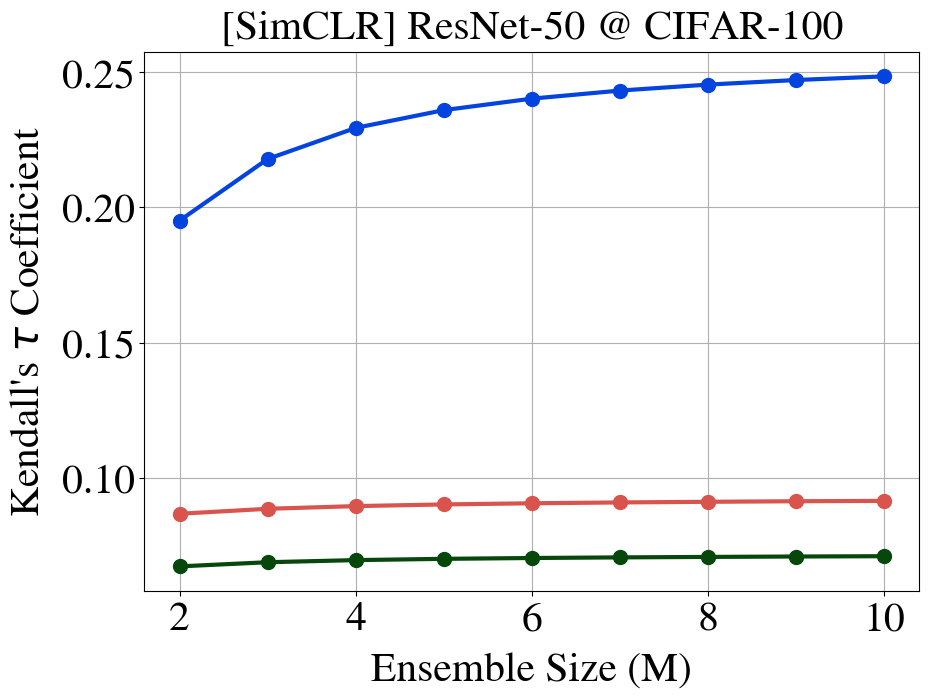}
    }
    \subfigure[]
    {   
        \includegraphics[width=0.3\columnwidth]{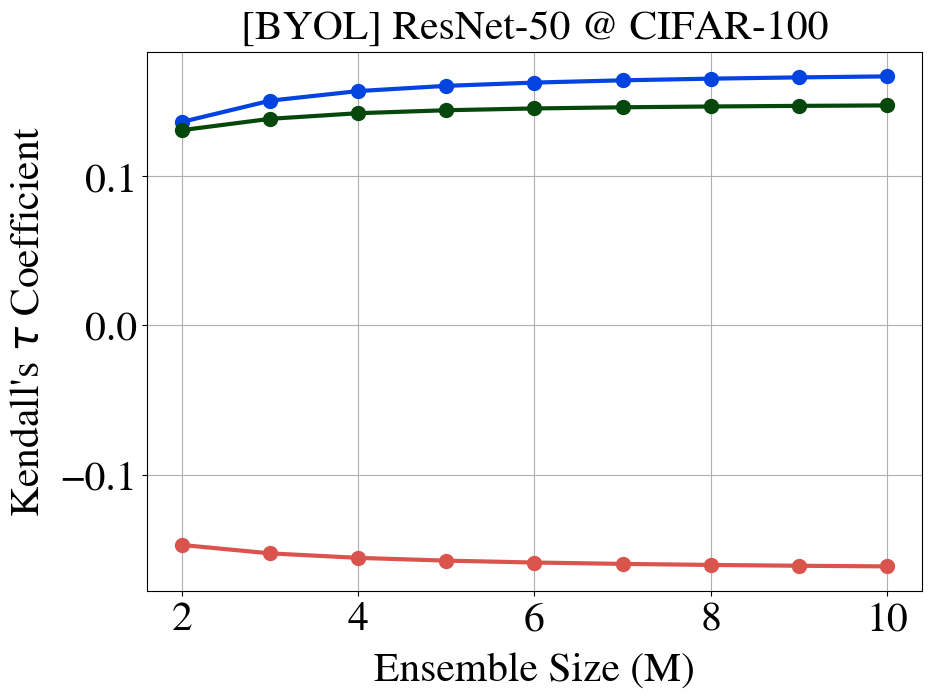}
    }
    \subfigure[]
    {   
        \includegraphics[width=0.3\columnwidth]{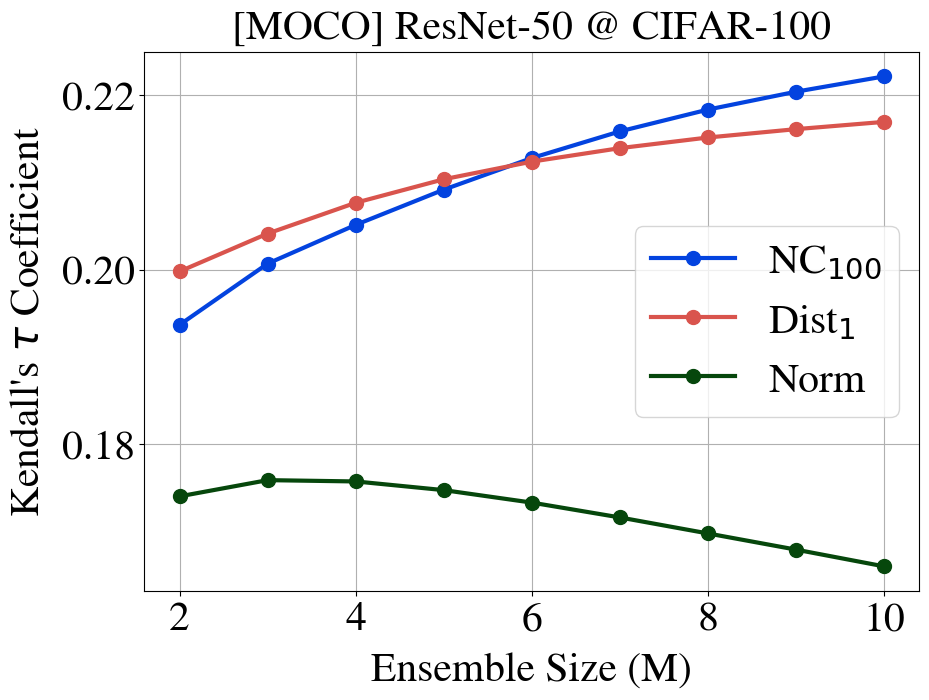}
    }
    \caption{Ablation over the ensemble size ($M$) for $\mathsf{NC}_k$ (ours) and baselines. Brier score is used for the downstream performance metric.}
    \label{fig:abl_m_full}
\end{figure}
\clearpage

\begin{figure}[t] 
    \centering
    \subfigure[]
    {
        \includegraphics[width=0.3\columnwidth]{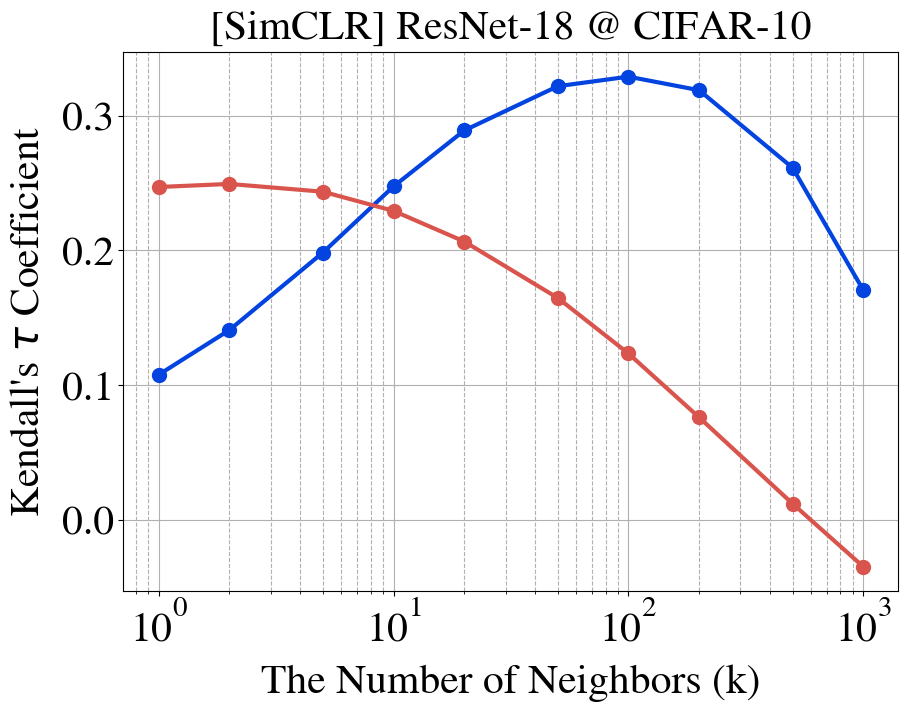}
    }
    \subfigure[]
    {   
        \includegraphics[width=0.3\columnwidth]{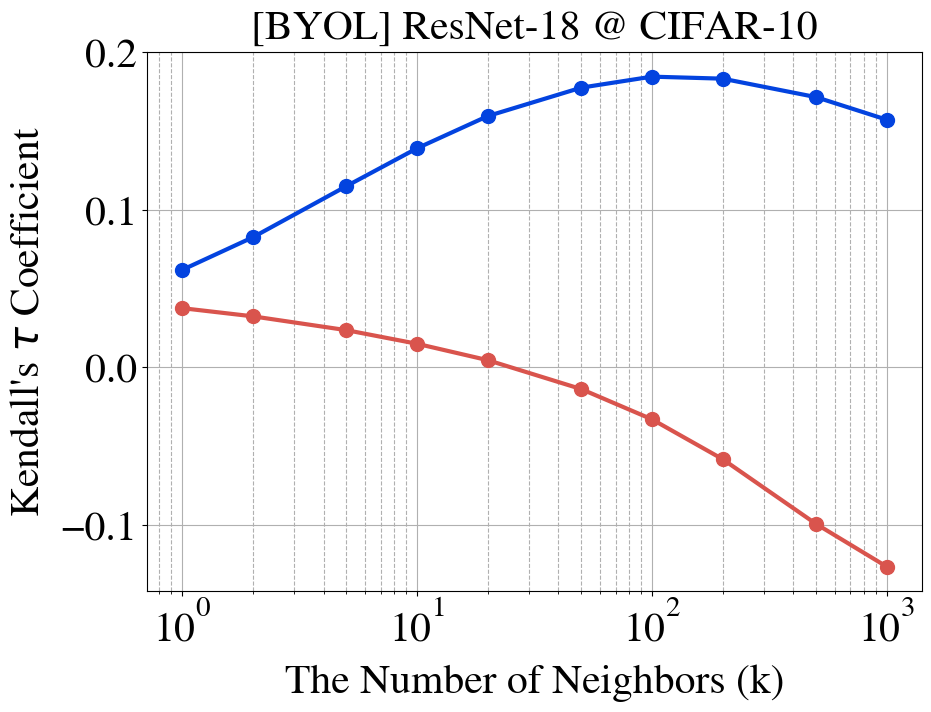}
    }
    \subfigure[]
    {   
        \includegraphics[width=0.3\columnwidth]{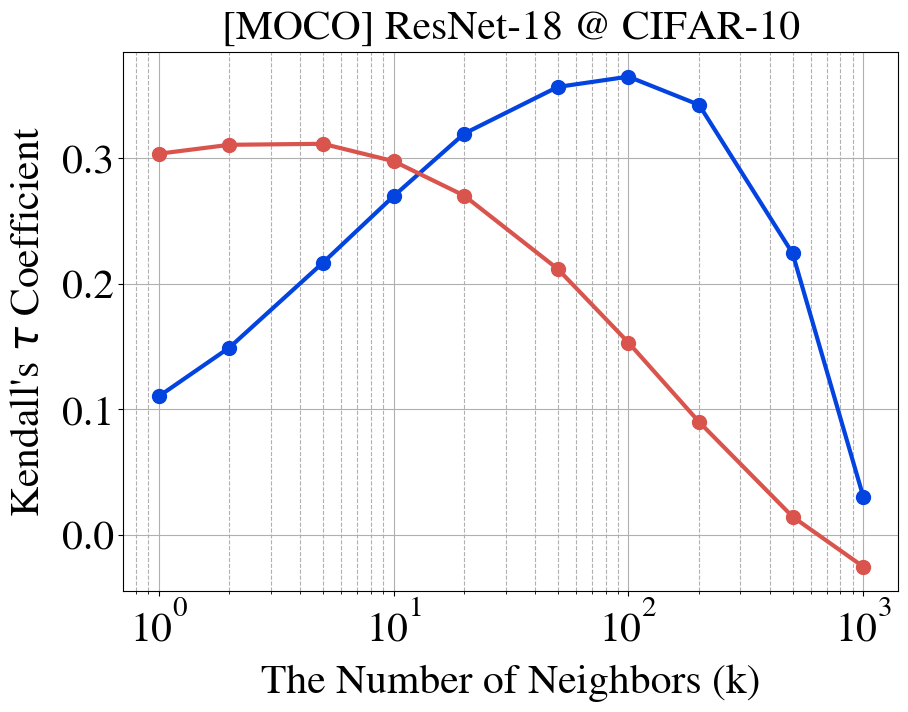}
    }
    \subfigure[]
    {
        \includegraphics[width=0.3\columnwidth]{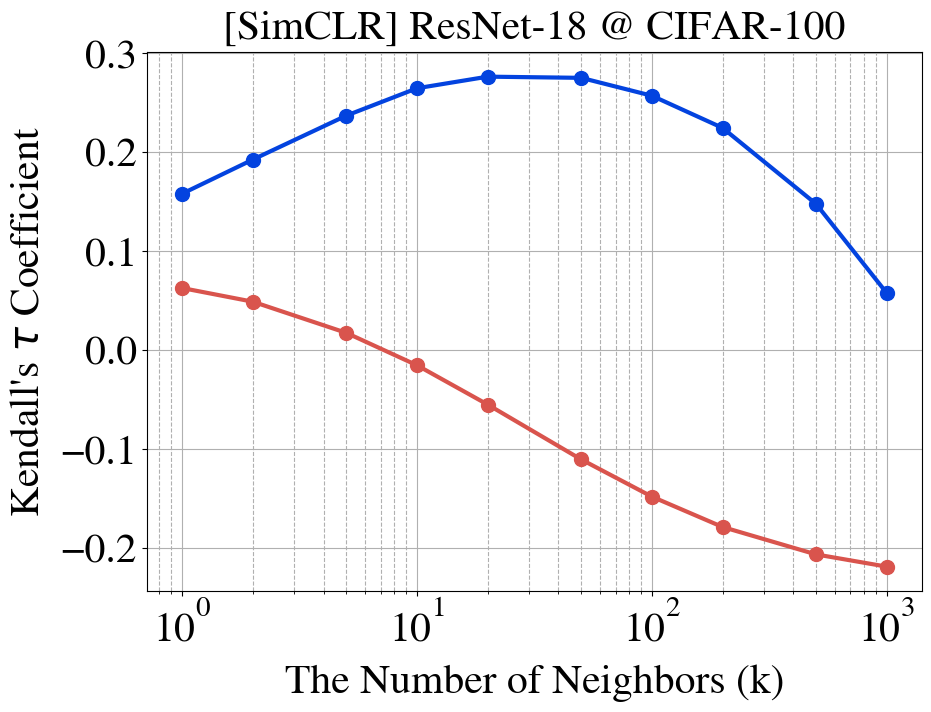}
    }
    \subfigure[]
    {   
        \includegraphics[width=0.3\columnwidth]{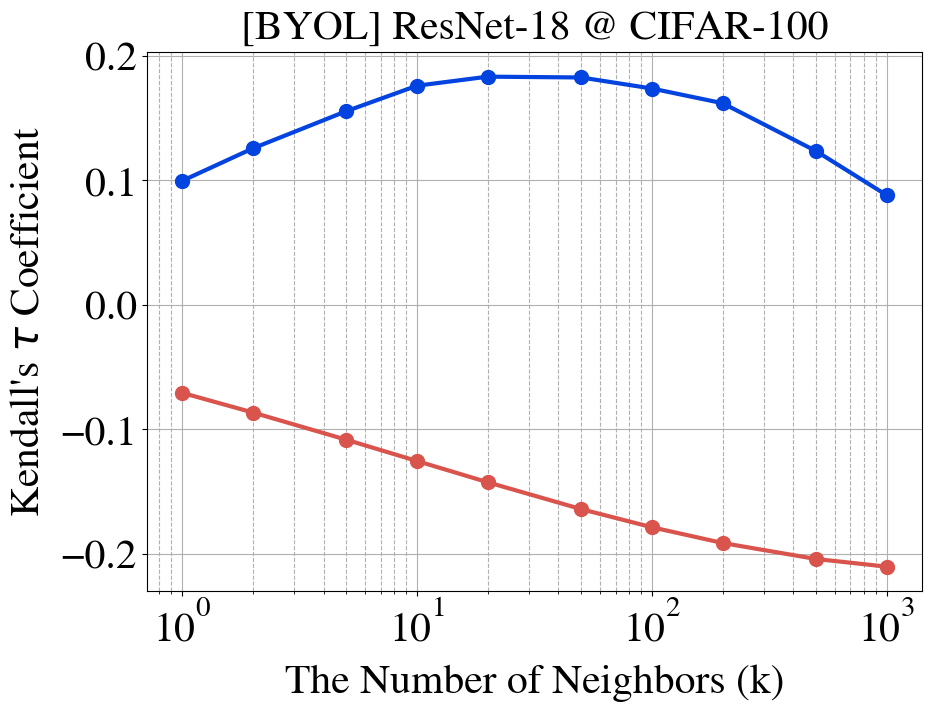}
    }
    \subfigure[]
    {   
        \includegraphics[width=0.3\columnwidth]{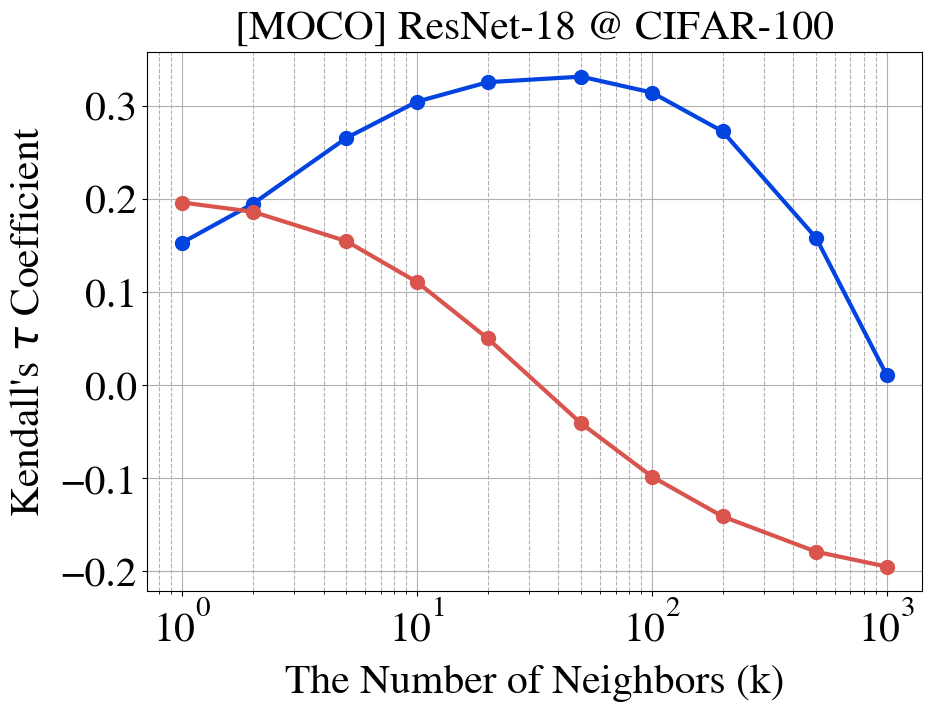}
    }
    \subfigure[]
    {
        \includegraphics[width=0.3\columnwidth]{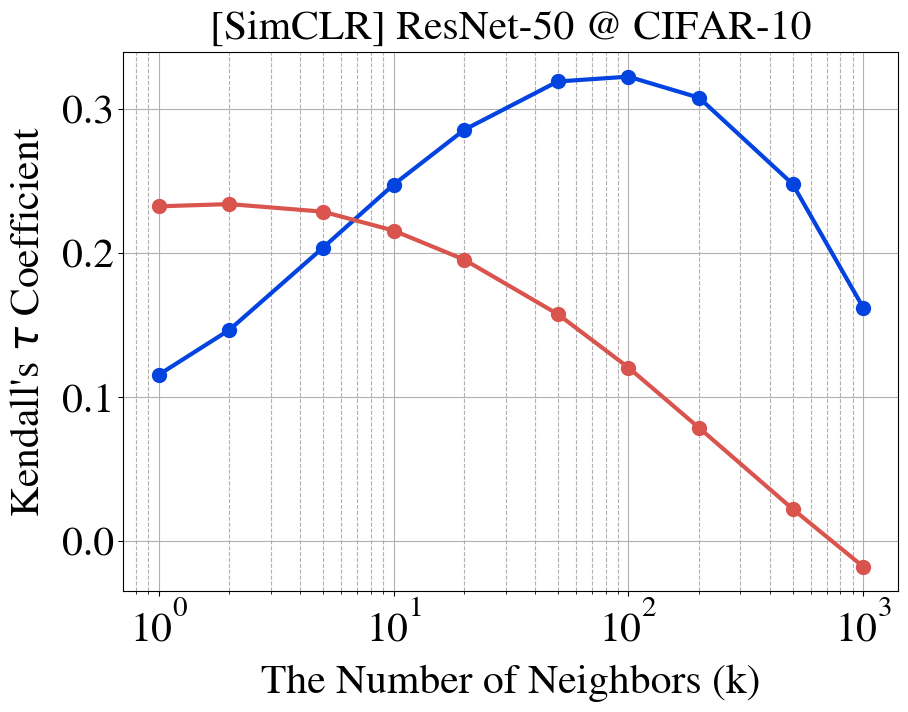}
    }
    \subfigure[]
    {   
        \includegraphics[width=0.3\columnwidth]{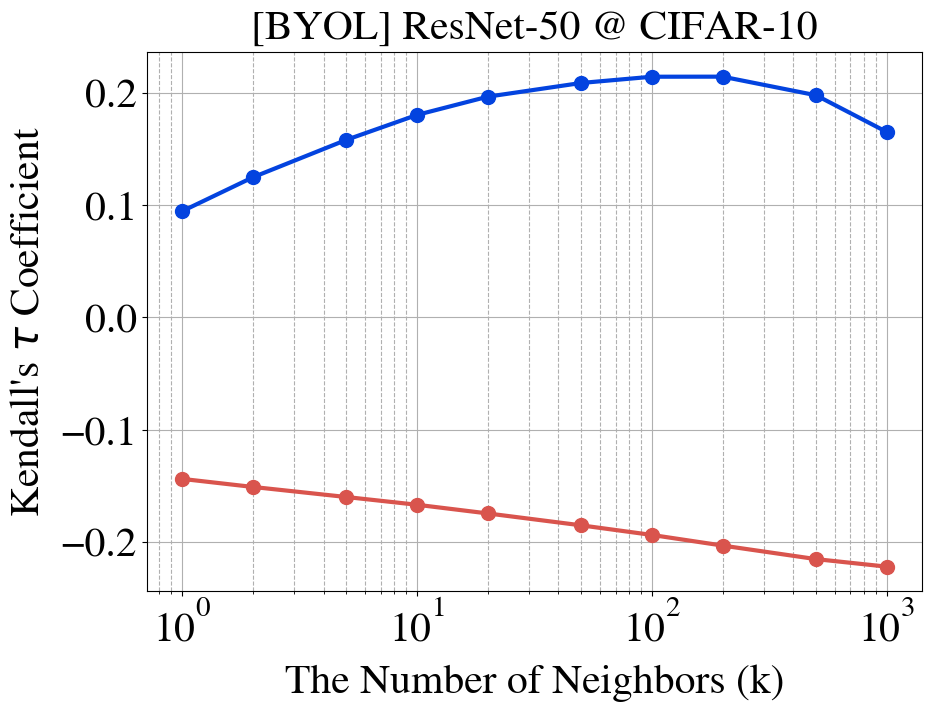}
    }
    \subfigure[]
    {   
        \includegraphics[width=0.3\columnwidth]{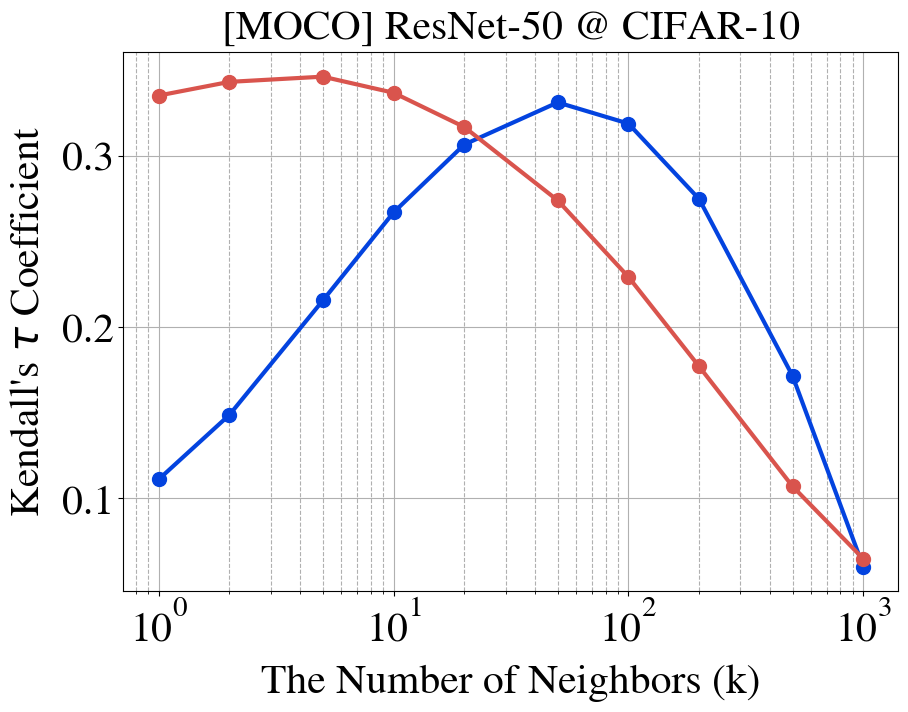}
    }
    \subfigure[]
    {
        \includegraphics[width=0.3\columnwidth]{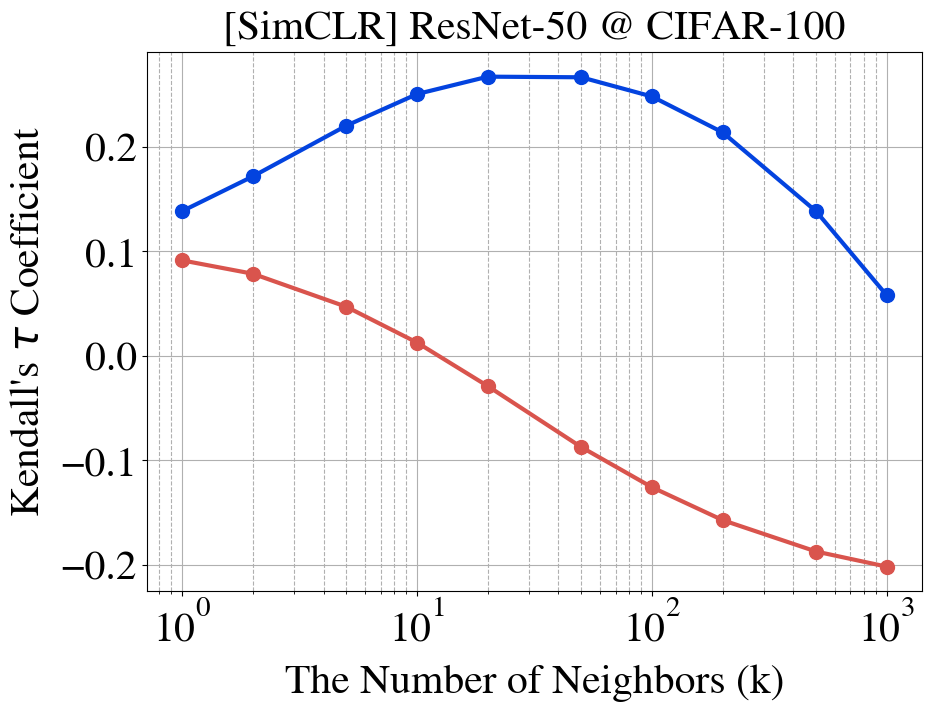}
    }
    \subfigure[]
    {   
        \includegraphics[width=0.3\columnwidth]{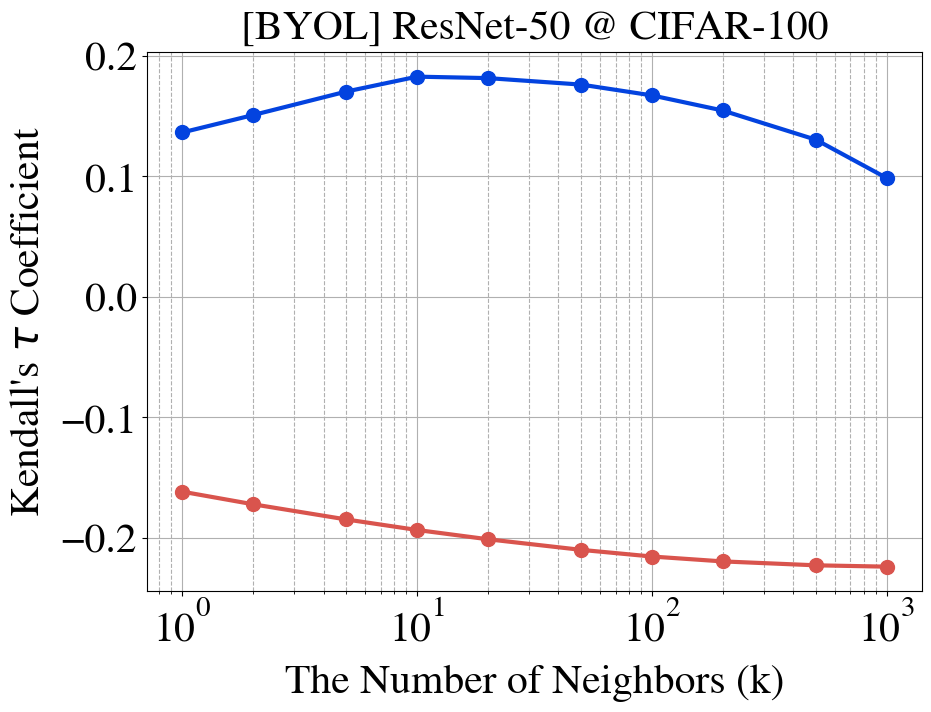}
    }
    \subfigure[]
    {   
        \includegraphics[width=0.3\columnwidth]{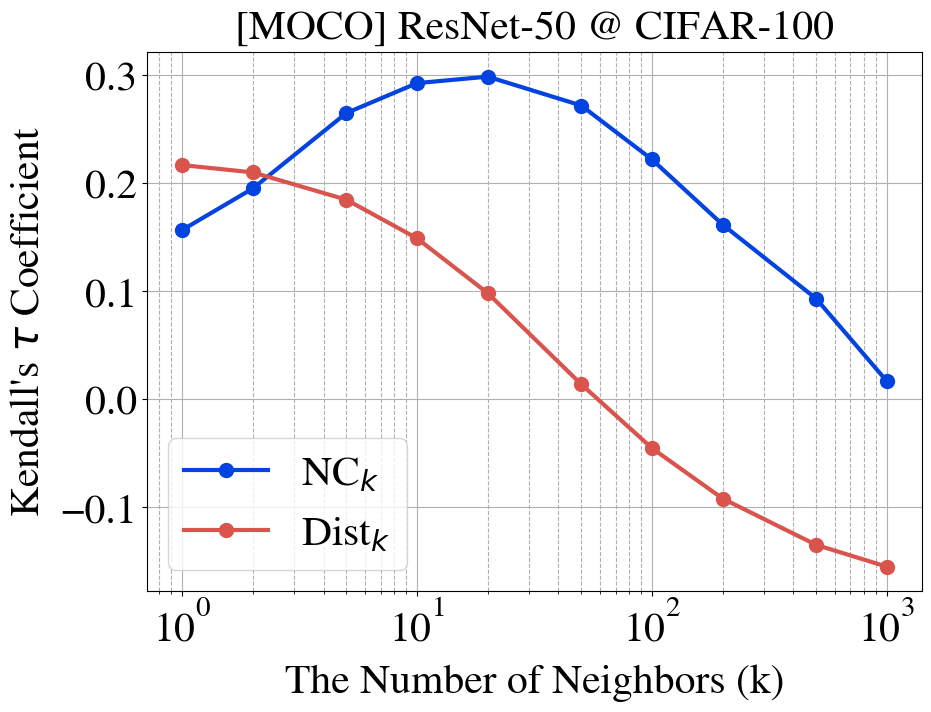}
    }
    \caption{Ablation over the number of neighbors ($k$) for $\mathsf{NC}_k$ (ours) and baselines. Brier score is used for the downstream performance metric.}
    \label{fig:abl_k_full}
\end{figure}
\clearpage

\begin{table*}[h] 
\centering
\fontsize{9pt}{9pt}\selectfont
\caption{Quantifying the reliability of individual embedding functions without retraining. \kendall correlation between each method and the downstream performance of individual models on in-distribution multi-class classification tasks. The baselines ($\amone$ and $\mathsf{Norm}$) are computed from the single model without ensemble averaging. The highest and second-highest scores are highlighted in bold and underlined, respectively.}
\renewcommand{\arraystretch}{1.25}
\begin{tabular}{cc|cccc}
\toprule
\multirow{3}{*}{\makecell[c]{Pretraining \\ Algorithms}} & \multirow{3}{*}{Method} & \multicolumn{4}{c}{ResNet-18}  \\
\cline{3-6} 
 & & \multicolumn{2}{c}{\texttt{CIFAR-10}}  & \multicolumn{2}{c}{\texttt{CIFAR-100}} \\
 & & Entropy  & Brier & Entropy & Brier \\
\midrule
\multirow{3}{*}{SimCLR} & $\mathsf{PNC}_{100}^s$ & \textbf{0.3300} \tiny{$\pm$ 0.0325} & \textbf{0.2716} \tiny{$\pm$ 0.0218} & \textbf{0.2183} \tiny{$\pm$ 0.0278} & \textbf{0.1427} \tiny{$\pm$ 0.0159}\\
& $\mathsf{Dist}_1$ & 0.2407 \tiny{$\pm$ 0.0251} & 0.2053 \tiny{$\pm$ 0.0176} & 0.0520 \tiny{$\pm$ 0.0315} & \underline{0.0640} \tiny{$\pm$ 0.0134}\\
& $\mathsf{Norm}$ & \underline{0.2903} \tiny{$\pm$ 0.0421} & \underline{0.2146} \tiny{$\pm$ 0.0319} & \underline{0.1808} \tiny{$\pm$ 0.0308} & 0.0363 \tiny{$\pm$ 0.0260}\\
\midrule
\multirow{3}{*}{BYOL} & $\mathsf{PNC}_{100}^s$ & \textbf{0.3198} \tiny{$\pm$ 0.0740} & \textbf{0.0863} \tiny{$\pm$ 0.0161} & \textbf{0.2793} \tiny{$\pm$ 0.0884} & \textbf{0.0494} \tiny{$\pm$ 0.0161}\\
& $\mathsf{Dist}_1$ & -0.1698 \tiny{$\pm$ 0.0411} & 0.0124 \tiny{$\pm$ 0.0224} & -0.2326 \tiny{$\pm$ 0.0231} & 0.0014 \tiny{$\pm$ 0.0128}\\
& $\mathsf{Norm}$ & \underline{0.0388} \tiny{$\pm$ 0.1060} & \underline{0.0381} \tiny{$\pm$ 0.0299} & \underline{0.1090} \tiny{$\pm$ 0.0494} & \underline{0.0438} \tiny{$\pm$ 0.0194}\\
\midrule
\multirow{3}{*}{MoCo} & $\mathsf{PNC}_{100}^s$ & \textbf{0.3404} \tiny{$\pm$ 0.0230} & \textbf{0.2955} \tiny{$\pm$ 0.0185} & \textbf{0.2872} \tiny{$\pm$ 0.0097} & \textbf{0.2097} \tiny{$\pm$ 0.0056}\\
& $\mathsf{Dist}_1$ & 0.2911 \tiny{$\pm$ 0.0343} & \underline{0.2592} \tiny{$\pm$ 0.0275} & 0.1791 \tiny{$\pm$ 0.0140} & \underline{0.1524} \tiny{$\pm$ 0.0110}\\
& $\mathsf{Norm}$ & \underline{0.3052} \tiny{$\pm$ 0.0403} & 0.2442 \tiny{$\pm$ 0.0342} & \underline{0.2316} \tiny{$\pm$ 0.0194} & 0.0937 \tiny{$\pm$ 0.0142}\\
\bottomrule
\end{tabular}

~
\vspace{1em}
~

\begin{tabular}{cc|cccccc}
\toprule
\multirow{3}{*}{\makecell[c]{Pretraining \\ Algorithms}} & \multirow{3}{*}{Method} & \multicolumn{4}{c}{ResNet-50}  \\
\cline{3-6} 
 & & \multicolumn{2}{c}{\texttt{CIFAR-10}}  & \multicolumn{2}{c}{\texttt{CIFAR-100}} \\
 & & Entropy  & Brier & Entropy & Brier \\
\midrule
\multirow{3}{*}{SimCLR} & $\mathsf{PNC}_{100}^s$ & \textbf{0.2946} \tiny{$\pm$ 0.0358} & \textbf{0.2474} \tiny{$\pm$ 0.0297} & \textbf{0.2318} \tiny{$\pm$ 0.0364} & \textbf{0.1605} \tiny{$\pm$ 0.0181}\\
& $\mathsf{Dist}_1$ & 0.2314 \tiny{$\pm$ 0.0281} & 0.2035 \tiny{$\pm$ 0.0223} & 0.0857 \tiny{$\pm$ 0.0321} & \underline{0.0801} \tiny{$\pm$ 0.0198}\\
& $\mathsf{Norm}$ & \underline{0.2782} \tiny{$\pm$ 0.0308} & \underline{0.2125} \tiny{$\pm$ 0.0269} & \underline{0.1942} \tiny{$\pm$ 0.0417} & 0.0406 \tiny{$\pm$ 0.0271}\\
\midrule
\multirow{3}{*}{BYOL} & $\mathsf{PNC}_{100}^s$ & \textbf{0.4408} \tiny{$\pm$ 0.0558} & \textbf{0.0932} \tiny{$\pm$ 0.0137} & \textbf{0.5472} \tiny{$\pm$ 0.0505} & \underline{0.0172} \tiny{$\pm$ 0.0206}\\
& $\mathsf{Dist}_1$ & -0.1230 \tiny{$\pm$ 0.0583} & \underline{0.0373} \tiny{$\pm$ 0.0235} & -0.2400 \tiny{$\pm$ 0.0325} & -0.0103 \tiny{$\pm$ 0.0140}\\
& $\mathsf{Norm}$ & \underline{-0.0576} \tiny{$\pm$ 0.1021} & -0.0055 \tiny{$\pm$ 0.0354} & \underline{-0.0132} \tiny{$\pm$ 0.0990} & \textbf{0.0301} \tiny{$\pm$ 0.0237}\\
\midrule
\multirow{3}{*}{MoCo} & $\mathsf{PNC}_{100}^s$ & 0.2964 \tiny{$\pm$ 0.0408} & 0.2508 \tiny{$\pm$ 0.0345} & \underline{0.2283} \tiny{$\pm$ 0.0446} & \underline{0.1588} \tiny{$\pm$ 0.0234}\\
& $\mathsf{Dist}_1$ & \underline{0.3208} \tiny{$\pm$ 0.0218} & \textbf{0.2847} \tiny{$\pm$ 0.0183} & 0.2102 \tiny{$\pm$ 0.0267} & \textbf{0.1645} \tiny{$\pm$ 0.0280}\\
& $\mathsf{Norm}$ & \textbf{0.3466} \tiny{$\pm$ 0.0356} & \underline{0.2815} \tiny{$\pm$ 0.0290} & \textbf{0.2782} \tiny{$\pm$ 0.0376} & 0.1196 \tiny{$\pm$ 0.0358}\\
\bottomrule
\end{tabular}
\label{tab:indist_single_multi}
\end{table*}

\begin{table*}[h] 
\centering
\fontsize{9pt}{9pt}\selectfont
\caption{Quantifying the reliability of individual embedding functions without retraining. \kendall correlation between each method and the downstream performance of individual models on transfer learning multi-class classification tasks. The baselines ($\amone$ and $\mathsf{Norm}$) are computed from the single model without ensemble averaging. Embedding
functions are pre-trained on \texttt{TinyImagenet} and fine-tuned on \texttt{CIFAR-10}, \texttt{CIFAR-100}, and
\texttt{STL-10}. The highest and second-highest scores are highlighted in bold and underlined, respectively.}
\renewcommand{\arraystretch}{1.25}
\resizebox{\textwidth}{!}{
\begin{tabular}{cc|cccccc}
\toprule
\multirow{4}{*}{\makecell[c]{Pretraining \\ Algorithms}} & \multirow{4}{*}{Method} & \multicolumn{6}{c}{ResNet-18}  \\
\cline{3-8} 
 & & \multicolumn{6}{c}{$\texttt{TinyImagenet} \rightarrow$} \\
 & & \multicolumn{2}{c}{\texttt{CIFAR-10}}  & \multicolumn{2}{c}{\texttt{CIFAR-100}} & \multicolumn{2}{c}{\texttt{STL-10}} \\
 & & Entropy  & Brier & Entropy & Brier & Entropy & Brier \\
\midrule
\multirow{3}{*}{SimCLR} & $\mathsf{PNC}_{100}^s$ & \underline{0.1769} \tiny{$\pm$ 0.0154} & \underline{0.1301} \tiny{$\pm$ 0.0117} & \textbf{0.1747} \tiny{$\pm$ 0.0080} & \textbf{0.0947} \tiny{$\pm$ 0.0048} & \textbf{0.2251} \tiny{$\pm$ 0.0198} & \textbf{0.1479} \tiny{$\pm$ 0.0149}\\
& $\mathsf{Dist}_1$ & 0.1210 \tiny{$\pm$ 0.0159} & 0.0882 \tiny{$\pm$ 0.0127} & -0.0228 \tiny{$\pm$ 0.0100} & -0.0188 \tiny{$\pm$ 0.0058} & 0.0589 \tiny{$\pm$ 0.0191} & 0.0357 \tiny{$\pm$ 0.0150}\\
& $\mathsf{Norm}$ & \textbf{0.2214} \tiny{$\pm$ 0.0121} & \textbf{0.1534} \tiny{$\pm$ 0.0091} & \underline{0.1363} \tiny{$\pm$ 0.0101} & \underline{0.0166} \tiny{$\pm$ 0.0071} & \underline{0.1435} \tiny{$\pm$ 0.0214} & \underline{0.0617} \tiny{$\pm$ 0.0168}\\
\midrule
\multirow{3}{*}{BYOL} & $\mathsf{PNC}_{100}^s$ & \textbf{0.2630} \tiny{$\pm$ 0.0692} & \textbf{0.0809} \tiny{$\pm$ 0.0161} & \textbf{0.3305} \tiny{$\pm$ 0.0873} & \underline{0.0170} \tiny{$\pm$ 0.0130} & \underline{0.0876} \tiny{$\pm$ 0.0625} & \underline{0.0207} \tiny{$\pm$ 0.0121}\\
& $\mathsf{Dist}_1$ & -0.1907 \tiny{$\pm$ 0.0371} & -0.0156 \tiny{$\pm$ 0.0106} & -0.3526 \tiny{$\pm$ 0.0297} & 0.0007 \tiny{$\pm$ 0.0089} & -0.3170 \tiny{$\pm$ 0.0262} & -0.1113 \tiny{$\pm$ 0.0145}\\
& $\mathsf{Norm}$ & \underline{0.1914} \tiny{$\pm$ 0.0302} & \underline{0.0588} \tiny{$\pm$ 0.0163} & \underline{0.2407} \tiny{$\pm$ 0.0561} & \textbf{0.0292} \tiny{$\pm$ 0.0162} & \textbf{0.1316} \tiny{$\pm$ 0.0724} & \textbf{0.0499} \tiny{$\pm$ 0.0270}\\
\midrule
\multirow{3}{*}{MoCo} & $\mathsf{PNC}_{100}^s$ & \underline{0.1775} \tiny{$\pm$ 0.0172} & \underline{0.1253} \tiny{$\pm$ 0.0134} & \textbf{0.2096} \tiny{$\pm$ 0.0123} & \textbf{0.0988} \tiny{$\pm$ 0.0080} & \textbf{0.2263} \tiny{$\pm$ 0.0252} & \textbf{0.1310} \tiny{$\pm$ 0.0187}\\
& $\mathsf{Dist}_1$ & 0.0858 \tiny{$\pm$ 0.0134} & 0.0601 \tiny{$\pm$ 0.0121} & -0.0068 \tiny{$\pm$ 0.0100} & -0.0247 \tiny{$\pm$ 0.0055} & 0.0441 \tiny{$\pm$ 0.0216} & 0.0045 \tiny{$\pm$ 0.0163}\\
& $\mathsf{Norm}$ & \textbf{0.1914} \tiny{$\pm$ 0.0177} & \textbf{0.1254} \tiny{$\pm$ 0.0132} & \underline{0.1600} \tiny{$\pm$ 0.0126} & \underline{0.0154} \tiny{$\pm$ 0.0073} & \underline{0.1237} \tiny{$\pm$ 0.0216} & \underline{0.0264} \tiny{$\pm$ 0.0171}\\
\bottomrule
\end{tabular}
}
~
\vspace{1em}
~

\resizebox{\textwidth}{!}{
\begin{tabular}{cc|cccccc}
\toprule
\multirow{4}{*}{\makecell[c]{Pretraining \\ Algorithms}} & \multirow{4}{*}{Method} & \multicolumn{6}{c}{ResNet-50}  \\
\cline{3-8} 
 & & \multicolumn{6}{c}{$\texttt{TinyImagenet} \rightarrow$} \\
 & & \multicolumn{2}{c}{\texttt{CIFAR-10}}  & \multicolumn{2}{c}{\texttt{CIFAR-100}} & \multicolumn{2}{c}{\texttt{STL-10}} \\
 & & Entropy  & Brier & Entropy & Brier & Entropy & Brier \\
\midrule
\multirow{3}{*}{SimCLR} & $\mathsf{PNC}_{100}^s$ & \underline{0.1282} \tiny{$\pm$ 0.0254} & \underline{0.0839} \tiny{$\pm$ 0.0205} & \textbf{0.1475} \tiny{$\pm$ 0.0207} & \textbf{0.0800} \tiny{$\pm$ 0.0077} & \textbf{0.1939} \tiny{$\pm$ 0.0332} & \textbf{0.1130} \tiny{$\pm$ 0.0244}\\
& $\mathsf{Dist}_1$ & 0.0765 \tiny{$\pm$ 0.0210} & 0.0546 \tiny{$\pm$ 0.0154} & -0.0644 \tiny{$\pm$ 0.0156} & -0.0476 \tiny{$\pm$ 0.0082} & -0.0155 \tiny{$\pm$ 0.0296} & -0.0223 \tiny{$\pm$ 0.0220}\\
& $\mathsf{Norm}$ & \textbf{0.1610} \tiny{$\pm$ 0.0245} & \textbf{0.1035} \tiny{$\pm$ 0.0172} & \underline{0.1099} \tiny{$\pm$ 0.0222} & \underline{-0.0047} \tiny{$\pm$ 0.0105} & \underline{0.0470} \tiny{$\pm$ 0.0385} & \underline{-0.0131} \tiny{$\pm$ 0.0269}\\
\midrule
\multirow{3}{*}{BYOL} & $\mathsf{PNC}_{100}^s$ & \textbf{0.5011} \tiny{$\pm$ 0.0526} & \textbf{0.1046} \tiny{$\pm$ 0.0210} & \textbf{0.5559} \tiny{$\pm$ 0.0348} & \underline{0.0094} \tiny{$\pm$ 0.0101} & \textbf{0.4008} \tiny{$\pm$ 0.0451} & \textbf{0.0244} \tiny{$\pm$ 0.0436}\\
& $\mathsf{Dist}_1$ & -0.1994 \tiny{$\pm$ 0.0608} & -0.0091 \tiny{$\pm$ 0.0229} & -0.2729 \tiny{$\pm$ 0.0548} & -0.0225 \tiny{$\pm$ 0.0158} & -0.2276 \tiny{$\pm$ 0.0823} & -0.0408 \tiny{$\pm$ 0.0304}\\
& $\mathsf{Norm}$ & \underline{0.0891} \tiny{$\pm$ 0.1019} & \underline{0.0533} \tiny{$\pm$ 0.0208} & \underline{0.0893} \tiny{$\pm$ 0.1204} & \textbf{0.0182} \tiny{$\pm$ 0.0175} & \underline{0.0568} \tiny{$\pm$ 0.0897} & \underline{0.0186} \tiny{$\pm$ 0.0264}\\
\midrule
\multirow{3}{*}{MoCo} & $\mathsf{PNC}_{100}^s$ & \underline{0.1381} \tiny{$\pm$ 0.0218} & 0.0871 \tiny{$\pm$ 0.0182} & \underline{0.1952} \tiny{$\pm$ 0.0205} & \textbf{0.0876} \tiny{$\pm$ 0.0115} & \underline{0.1971} \tiny{$\pm$ 0.0554} & \underline{0.1080} \tiny{$\pm$ 0.0349}\\
& $\mathsf{Dist}_1$ & 0.1309 \tiny{$\pm$ 0.0142} & \underline{0.0921} \tiny{$\pm$ 0.0115} & 0.0590 \tiny{$\pm$ 0.0113} & -0.0031 \tiny{$\pm$ 0.0056} & 0.1060 \tiny{$\pm$ 0.0170} & 0.0340 \tiny{$\pm$ 0.0105}\\
& $\mathsf{Norm}$ & \textbf{0.2503} \tiny{$\pm$ 0.0101} & \textbf{0.1631} \tiny{$\pm$ 0.0093} & \textbf{0.2341} \tiny{$\pm$ 0.0165} & \underline{0.0511} \tiny{$\pm$ 0.0080} & \textbf{0.2422} \tiny{$\pm$ 0.0163} & \textbf{0.1094} \tiny{$\pm$ 0.0130}\\
\bottomrule
\end{tabular}
}
\label{tab:transfer_single_multi}
\end{table*}

\begin{table*}[t!]
\centering
\fontsize{9pt}{9pt}\selectfont
\caption{Oracle ($\mathsf{PNC}_{100}^o$) vs.\ spread-tuned ($\mathsf{PNC}_{100}^s$) \kendall correlation with downstream performance on in-distribution binary classification tasks.}

\resizebox{\textwidth}{!}{
\renewcommand{\arraystretch}{1.1}
\begin{tabular}{cc|cccc|cccc}
\toprule
\multirow{3}{*}{\makecell[c]{Pretraining \\ Algorithms}} & \multirow{3}{*}{Method} & \multicolumn{4}{c|}{ResNet-18} & \multicolumn{4}{c}{ResNet-50}  \\
\cline{3-10} 
 & & \multicolumn{2}{c}{\texttt{CIFAR-10}}  & \multicolumn{2}{c|}{\texttt{CIFAR-100}} & \multicolumn{2}{c}{\texttt{CIFAR-10}}  & \multicolumn{2}{c}{\texttt{CIFAR-100}} \\
 & & Entropy  & Brier & Entropy & Brier & Entropy  & Brier & Entropy & Brier \\
\midrule
\multirow{2}{*}{SimCLR} & $\mathsf{PNC}_{100}^o$ & 0.3361 & 0.2863 & 0.2405 & 0.1929 & 0.3116 & 0.2658 & 0.2475 & 0.2004\\
& $\mathsf{PNC}_{100}^s$ & 0.3262 & 0.2779 & 0.2175 & 0.1724 & 0.2961 & 0.2525 & 0.2396 & 0.1954\\
\midrule
\multirow{2}{*}{BYOL} & $\mathsf{PNC}_{100}^o$ & 0.2458 & 0.1509 & 0.2624 & 0.1656 & 0.3530 & 0.1911 & 0.3157 & 0.1405\\
& $\mathsf{PNC}_{100}^s$ & 0.2306 & 0.1361 & 0.2184 & 0.1364 & 0.3473 & 0.1764 & 0.3130 & 0.1231\\
\midrule
\multirow{2}{*}{MoCo} & $\mathsf{PNC}_{100}^o$ & 0.3810 & 0.3333 & 0.3223 & 0.2744 & 0.3622 & 0.3124 & 0.2968 & 0.2499\\
& $\mathsf{PNC}_{100}^s$ & 0.3619 & 0.3156 & 0.3000 & 0.2558 & 0.3214 & 0.2749 & 0.2348 & 0.1995\\
\bottomrule
\end{tabular}
}
\label{tab:pnc_oracle_indist}
\end{table*}

\begin{table*}[t!]
\centering
\fontsize{9pt}{9pt}\selectfont
\caption{Oracle ($\mathsf{PNC}_{100}^o$) vs.\ spread-tuned ($\mathsf{PNC}_{100}^s$) \kendall correlation with downstream performance on transfer learning binary classification tasks.}
\label{tab:pnc_oracle_transfer}

\subtable[ResNet-18]{
\renewcommand{\arraystretch}{0.9}
\begin{tabular}{cc|cccccc}
\toprule
\multirow{3}{*}{\makecell[c]{Pretraining \\ Algorithms}} & \multirow{3}{*}{Method} & \multicolumn{6}{c}{$\texttt{TinyImagenet} \rightarrow$}  \\
 & & \multicolumn{2}{c}{\texttt{CIFAR-10}}  & \multicolumn{2}{c}{\texttt{CIFAR-100}} & \multicolumn{2}{c}{\texttt{STL-10}} \\
 & & Entropy  & Brier & Entropy & Brier & Entropy  & Brier \\
\midrule
\multirow{2}{*}{SimCLR} & $\mathsf{PNC}_{100}^o$ & 0.1867 & 0.1441 & 0.1964 & 0.1475 & 0.2503 & 0.1780\\
& $\mathsf{PNC}_{100}^s$ & 0.1810 & 0.1395 & 0.1958 & 0.1472 & 0.2472 & 0.1745\\
\midrule
\multirow{2}{*}{BYOL} & $\mathsf{PNC}_{100}^o$ & 0.2172 & 0.1254 & 0.2360 & 0.1584 & 0.1183 & 0.0901\\
& $\mathsf{PNC}_{100}^s$ & 0.1998 & 0.1155 & 0.2092 & 0.1415 & 0.0465 & 0.0380\\
\midrule
\multirow{2}{*}{MoCo} & $\mathsf{PNC}_{100}^o$ & 0.1958 & 0.1403 & 0.2169 & 0.1754 & 0.2517 & 0.1725\\
& $\mathsf{PNC}_{100}^s$ & 0.1954 & 0.1399 & 0.2169 & 0.1753 & 0.2491 & 0.1668\\
\bottomrule
\end{tabular}
\label{tab:pnc_oracle_transfer_r18}
}
\subtable[ResNet-50]{
\renewcommand{\arraystretch}{0.9}
\begin{tabular}{cc|cccccc}
\toprule
\multirow{3}{*}{\makecell[c]{Pretraining \\ Algorithms}} & \multirow{3}{*}{Method} & \multicolumn{6}{c}{$\texttt{TinyImagenet} \rightarrow$}  \\
 & & \multicolumn{2}{c}{\texttt{CIFAR-10}}  & \multicolumn{2}{c}{\texttt{CIFAR-100}} & \multicolumn{2}{c}{\texttt{STL-10}} \\
 & & Entropy  & Brier & Entropy & Brier & Entropy  & Brier \\
\midrule
\multirow{2}{*}{SimCLR} & $\mathsf{PNC}_{100}^o$ & 0.1684 & 0.1249 & 0.1858 & 0.1381 & 0.2336 & 0.1477\\
& $\mathsf{PNC}_{100}^s$ & 0.1262 & 0.0884 & 0.1672 & 0.1252 & 0.2254 & 0.1358\\
\midrule
\multirow{2}{*}{BYOL} & $\mathsf{PNC}_{100}^o$ & 0.3091 & 0.1633 & 0.3404 & 0.1692 & 0.2464 & 0.1083\\
& $\mathsf{PNC}_{100}^s$ & 0.3044 & 0.1461 & 0.3367 & 0.1609 & 0.2378 & 0.0891\\
\midrule
\multirow{2}{*}{MoCo} & $\mathsf{PNC}_{100}^o$ & 0.1956 & 0.1338 & 0.2223 & 0.1766 & 0.2483 & 0.1670\\
& $\mathsf{PNC}_{100}^s$ & 0.1619 & 0.1078 & 0.2056 & 0.1664 & 0.2245 & 0.1500\\
\bottomrule
\end{tabular}
\label{tab:pnc_oracle_transfer_r50}
}

\end{table*}

\begin{table*}[h] 
\centering
\fontsize{9pt}{9pt}\selectfont
\caption{Oracle ($\mathsf{PNC}_{100}^o$) vs.\ spread-tuned ($\mathsf{PNC}_{100}^s$) \kendall correlation with downstream performance on in-distribution multi-class classification tasks.}
\resizebox{\textwidth}{!}{
\renewcommand{\arraystretch}{1.1}
\begin{tabular}{cc|cccc|cccc}
\toprule
\multirow{3}{*}{\makecell[c]{Pretraining \\ Algorithms}} & \multirow{3}{*}{Method} & \multicolumn{4}{c|}{ResNet-18} & \multicolumn{4}{c}{ResNet-50}  \\
\cline{3-10} 
 & & \multicolumn{2}{c}{\texttt{CIFAR-10}}  & \multicolumn{2}{c|}{\texttt{CIFAR-100}} & \multicolumn{2}{c}{\texttt{CIFAR-10}}  & \multicolumn{2}{c}{\texttt{CIFAR-100}} \\
 & & Entropy  & Brier & Entropy & Brier & Entropy  & Brier & Entropy & Brier \\
\midrule
\multirow{2}{*}{SimCLR} & $\mathsf{PNC}_{100}^o$ & 0.3397 & 0.2806 & 0.2367 & 0.1594 & 0.3107 & 0.2613 & 0.2412 & 0.1641\\
& $\mathsf{PNC}_{100}^s$ & 0.3300 & 0.2716 & 0.2183 & 0.1427 & 0.2946 & 0.2474 & 0.2318 & 0.1605\\
\midrule
\multirow{2}{*}{BYOL} & $\mathsf{PNC}_{100}^o$ & 0.3365 & 0.1013 & 0.3468 & 0.0736 & 0.4493 & 0.1082 & 0.5505 & 0.0516\\
& $\mathsf{PNC}_{100}^s$ & 0.3198 & 0.0863 & 0.2793 & 0.0494 & 0.4408 & 0.0932 & 0.5472 & 0.0172\\
\midrule
\multirow{2}{*}{MoCo} & $\mathsf{PNC}_{100}^o$ & 0.3620 & 0.3141 & 0.3117 & 0.2259 & 0.3391 & 0.2886 & 0.2954 & 0.1992\\
& $\mathsf{PNC}_{100}^s$ & 0.3404 & 0.2955 & 0.2872 & 0.2097 & 0.2964 & 0.2508 & 0.2283 & 0.1588\\
\bottomrule
\end{tabular}
}
\label{tab:pnc_oracle_indist_multi}
\end{table*}

\begin{table*}[h] 
\centering
\fontsize{9pt}{9pt}\selectfont
\caption{Oracle ($\mathsf{PNC}_{100}^o$) vs.\ spread-tuned ($\mathsf{PNC}_{100}^s$) \kendall correlation with downstream performance on transfer learning multi-class classification tasks.}
\subtable[ResNet-18]{
\renewcommand{\arraystretch}{0.9}
\begin{tabular}{cc|cccccc}
\toprule
\multirow{3}{*}{\makecell[c]{Pretraining \\ Algorithms}} & \multirow{3}{*}{Method} & \multicolumn{6}{c}{$\texttt{TinyImagenet} \rightarrow$}  \\
 & & \multicolumn{2}{c}{\texttt{CIFAR-10}}  & \multicolumn{2}{c}{\texttt{CIFAR-100}} & \multicolumn{2}{c}{\texttt{STL-10}} \\
 & & Entropy  & Brier & Entropy & Brier & Entropy  & Brier \\
\midrule
\multirow{2}{*}{SimCLR} & $\mathsf{PNC}_{100}^o$ & 0.1825 & 0.1350 & 0.1756 & 0.0950 & 0.2295 & 0.1516\\
& $\mathsf{PNC}_{100}^s$ & 0.1769 & 0.1301 & 0.1747 & 0.0947 & 0.2251 & 0.1479\\
\midrule
\multirow{2}{*}{BYOL} & $\mathsf{PNC}_{100}^o$ & 0.2973 & 0.0907 & 0.3745 & 0.0401 & 0.1616 & 0.0600\\
& $\mathsf{PNC}_{100}^s$ & 0.2630 & 0.0809 & 0.3305 & 0.0170 & 0.0876 & 0.0207\\
\midrule
\multirow{2}{*}{MoCo} & $\mathsf{PNC}_{100}^o$ & 0.1782 & 0.1261 & 0.2096 & 0.1031 & 0.2289 & 0.1360\\
& $\mathsf{PNC}_{100}^s$ & 0.1775 & 0.1253 & 0.2096 & 0.0988 & 0.2263 & 0.1310\\
\bottomrule
\end{tabular}
\label{tab:pnc_oracle_transfer_multi_r18}
}
\subtable[ResNet-50]{
\renewcommand{\arraystretch}{0.9}
\begin{tabular}{cc|cccccc}
\toprule
\multirow{3}{*}{\makecell[c]{Pretraining \\ Algorithms}} & \multirow{3}{*}{Method} & \multicolumn{6}{c}{$\texttt{TinyImagenet} \rightarrow$}  \\
 & & \multicolumn{2}{c}{\texttt{CIFAR-10}}  & \multicolumn{2}{c}{\texttt{CIFAR-100}} & \multicolumn{2}{c}{\texttt{STL-10}} \\
 & & Entropy  & Brier & Entropy & Brier & Entropy  & Brier \\
\midrule
\multirow{2}{*}{SimCLR} & $\mathsf{PNC}_{100}^o$ & 0.1659 & 0.1187 & 0.1740 & 0.0862 & 0.2094 & 0.1252\\
& $\mathsf{PNC}_{100}^s$ & 0.1282 & 0.0839 & 0.1475 & 0.0800 & 0.1939 & 0.1130\\
\midrule
\multirow{2}{*}{BYOL} & $\mathsf{PNC}_{100}^o$ & 0.5068 & 0.1110 & 0.5624 & 0.0540 & 0.4091 & 0.0507\\
& $\mathsf{PNC}_{100}^s$ & 0.5011 & 0.1046 & 0.5559 & 0.0094 & 0.4008 & 0.0244\\
\midrule
\multirow{2}{*}{MoCo} & $\mathsf{PNC}_{100}^o$ & 0.1767 & 0.1166 & 0.2210 & 0.0921 & 0.2274 & 0.1276\\
& $\mathsf{PNC}_{100}^s$ & 0.1381 & 0.0871 & 0.1952 & 0.0876 & 0.1971 & 0.1080\\
\bottomrule
\end{tabular}
\label{tab:pnc_oracle_transfer_multi_r50}
}
\label{tab:pnc_oracle_transfer_multi}
\end{table*}

\begin{figure}[t] 
    \centering
    \subfigure{
        \includegraphics[width=0.4\columnwidth]{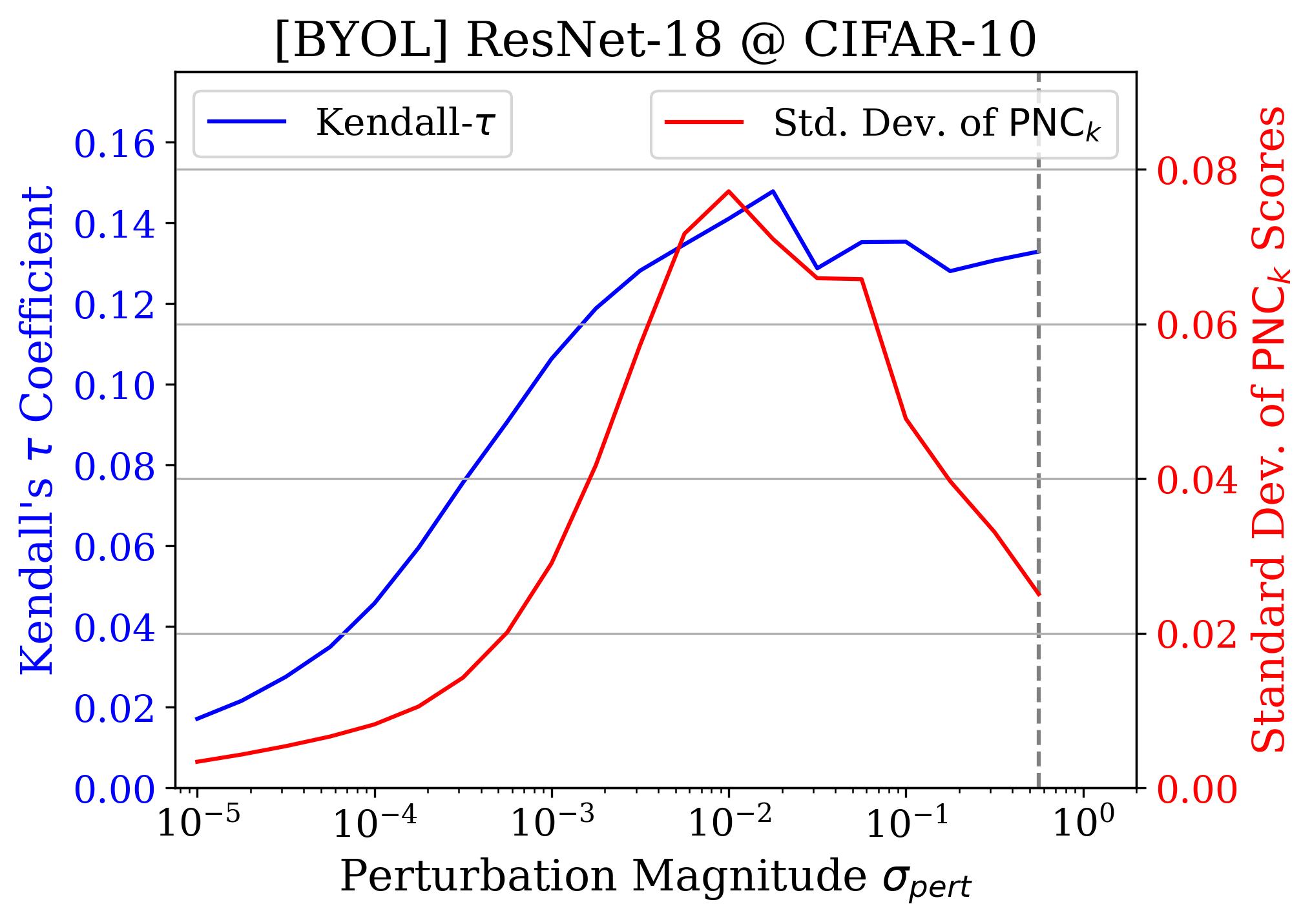}
    }
    \subfigure{   
        \includegraphics[width=0.4\columnwidth]{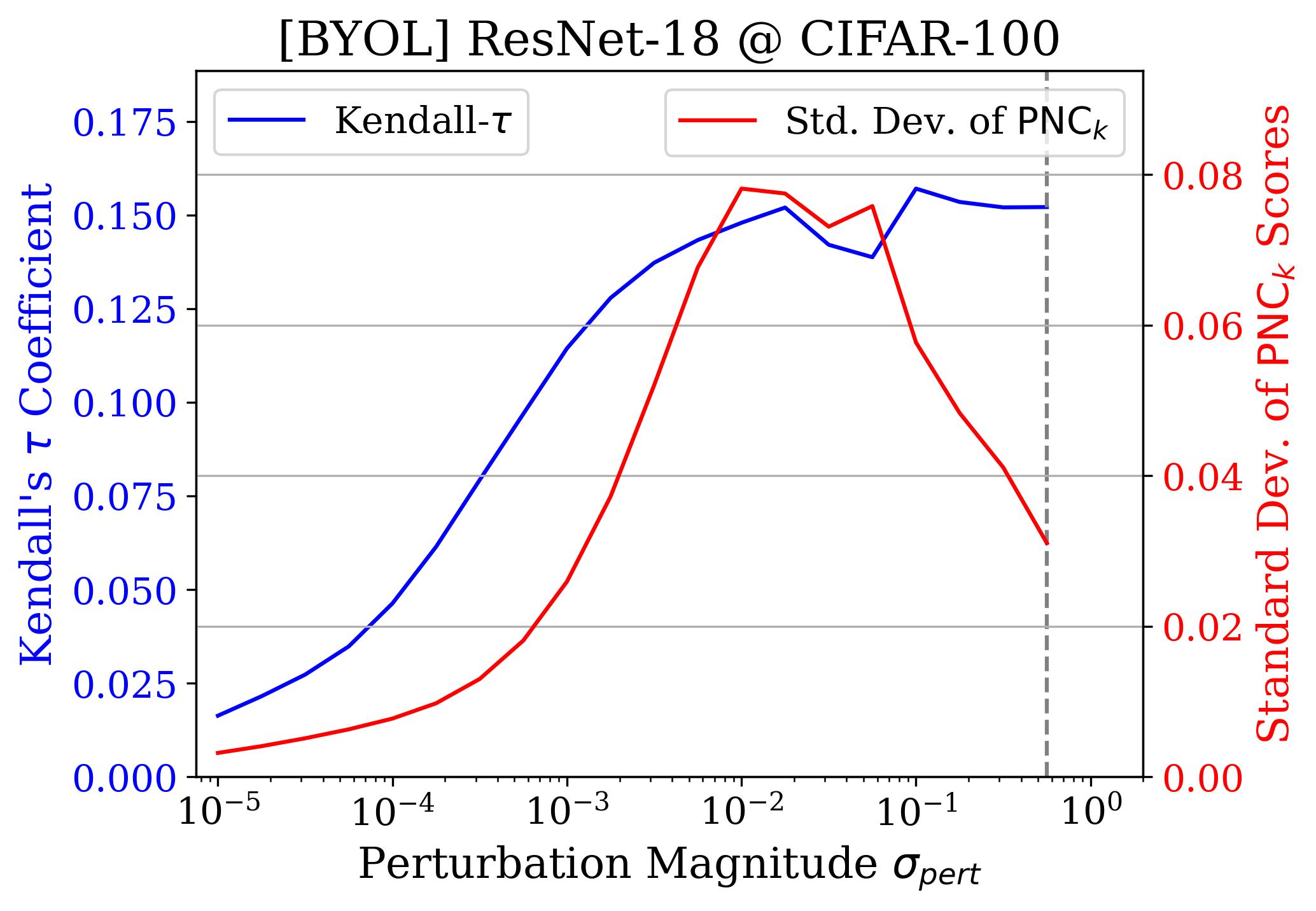}
    }
    \subfigure{   
        \includegraphics[width=0.4\columnwidth]{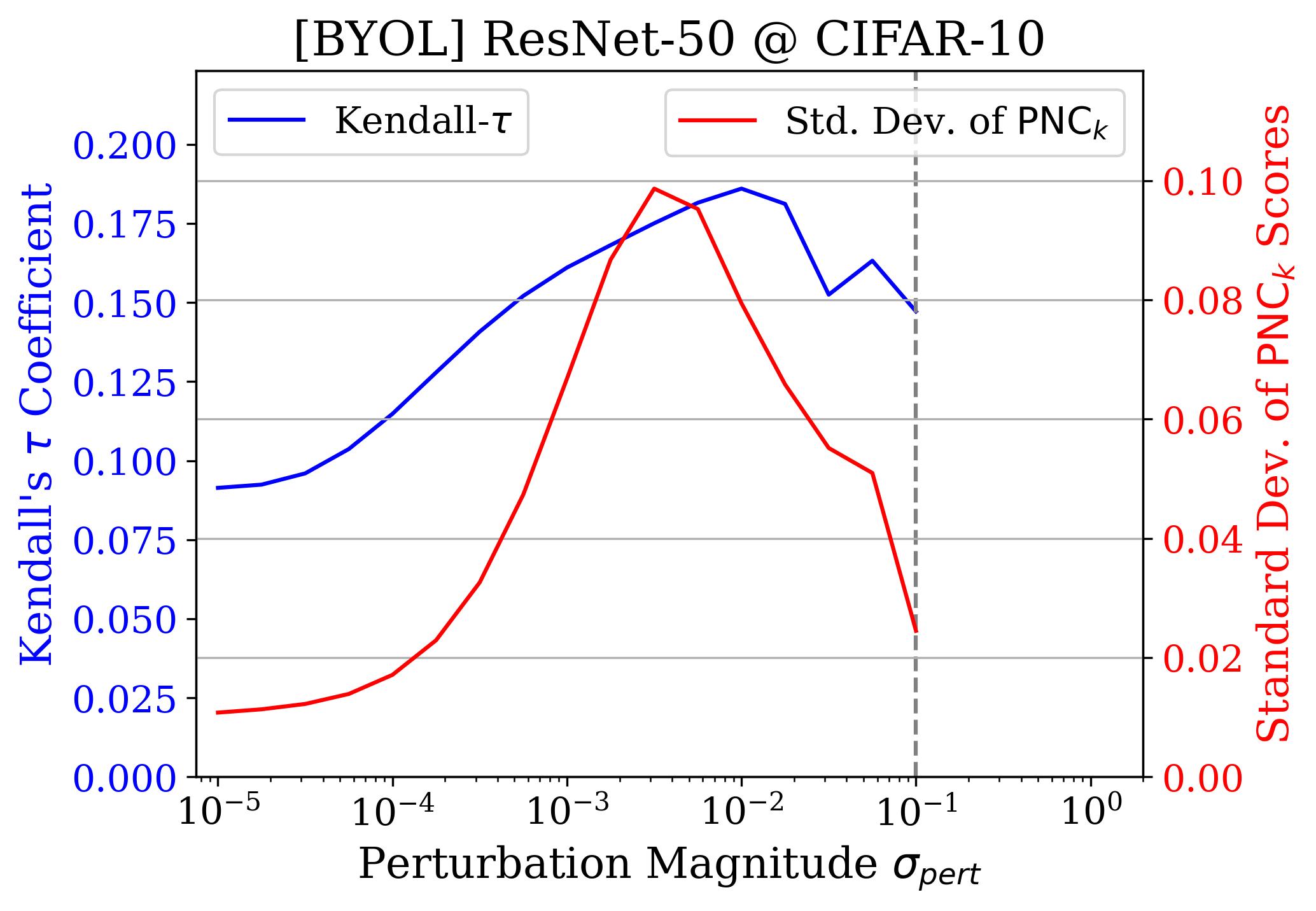}
    }
    \subfigure{
        \includegraphics[width=0.4\columnwidth]{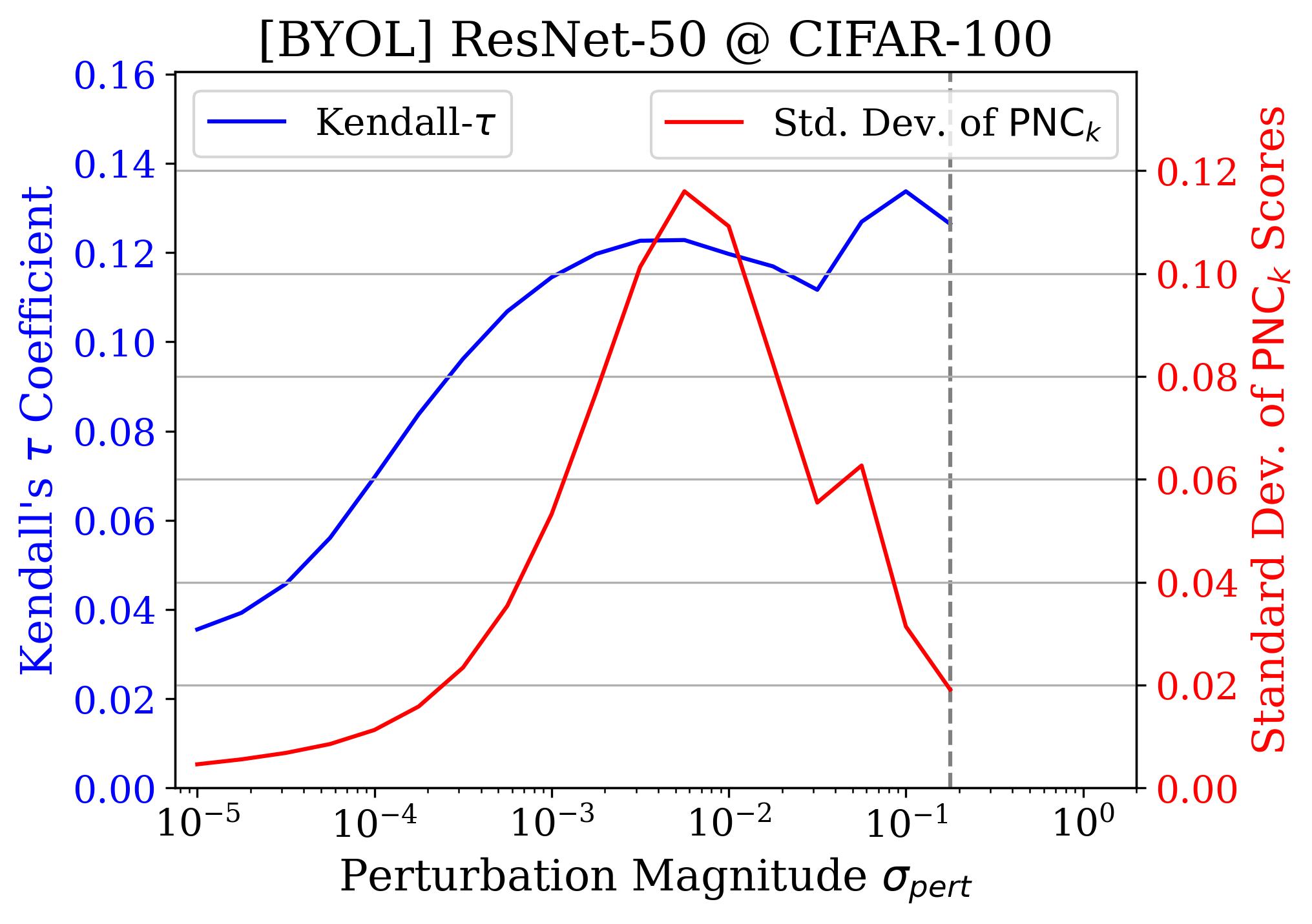}
    }
    \subfigure{
        \includegraphics[width=0.4\columnwidth]{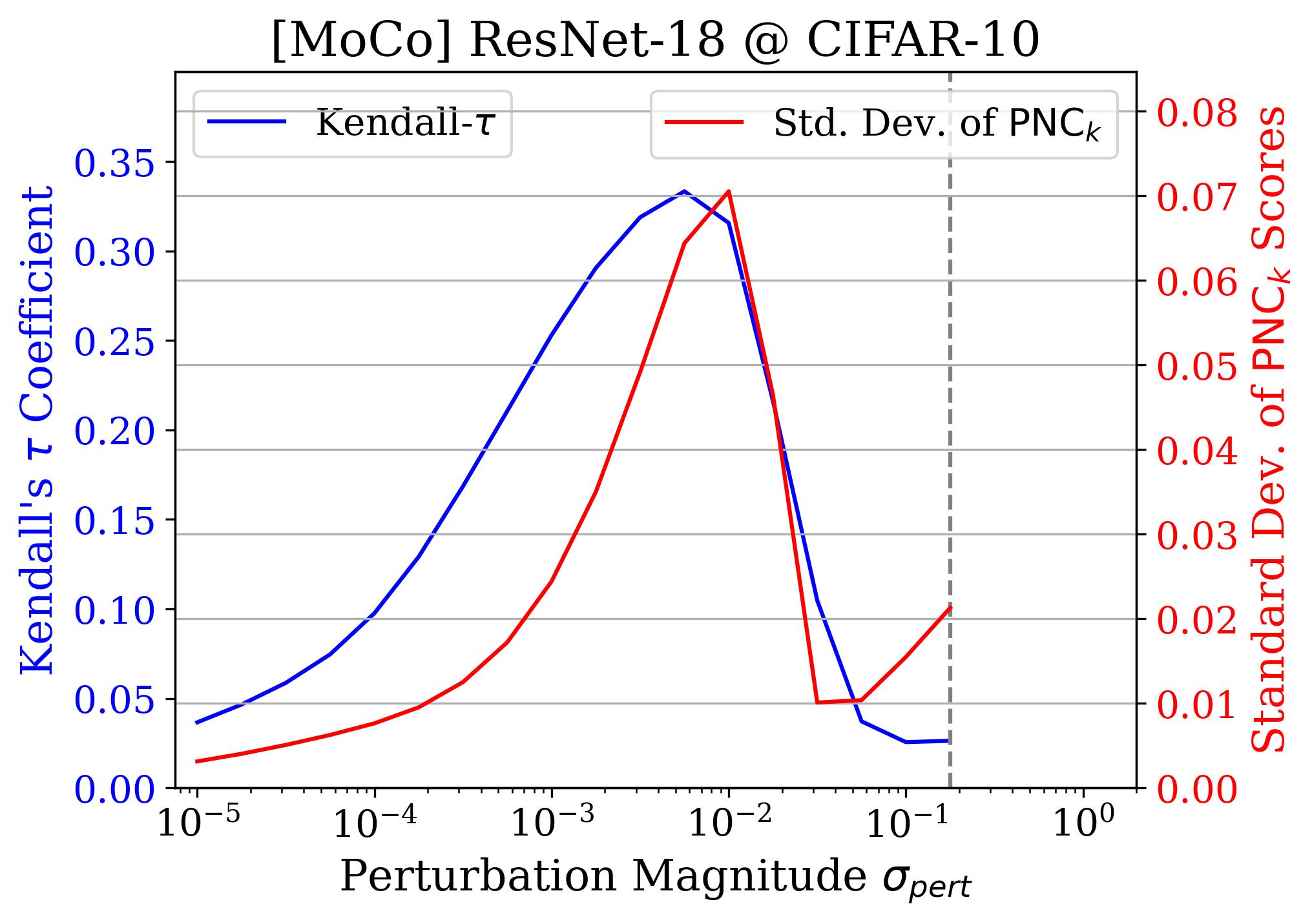}
    }
    \subfigure{   
        \includegraphics[width=0.4\columnwidth]{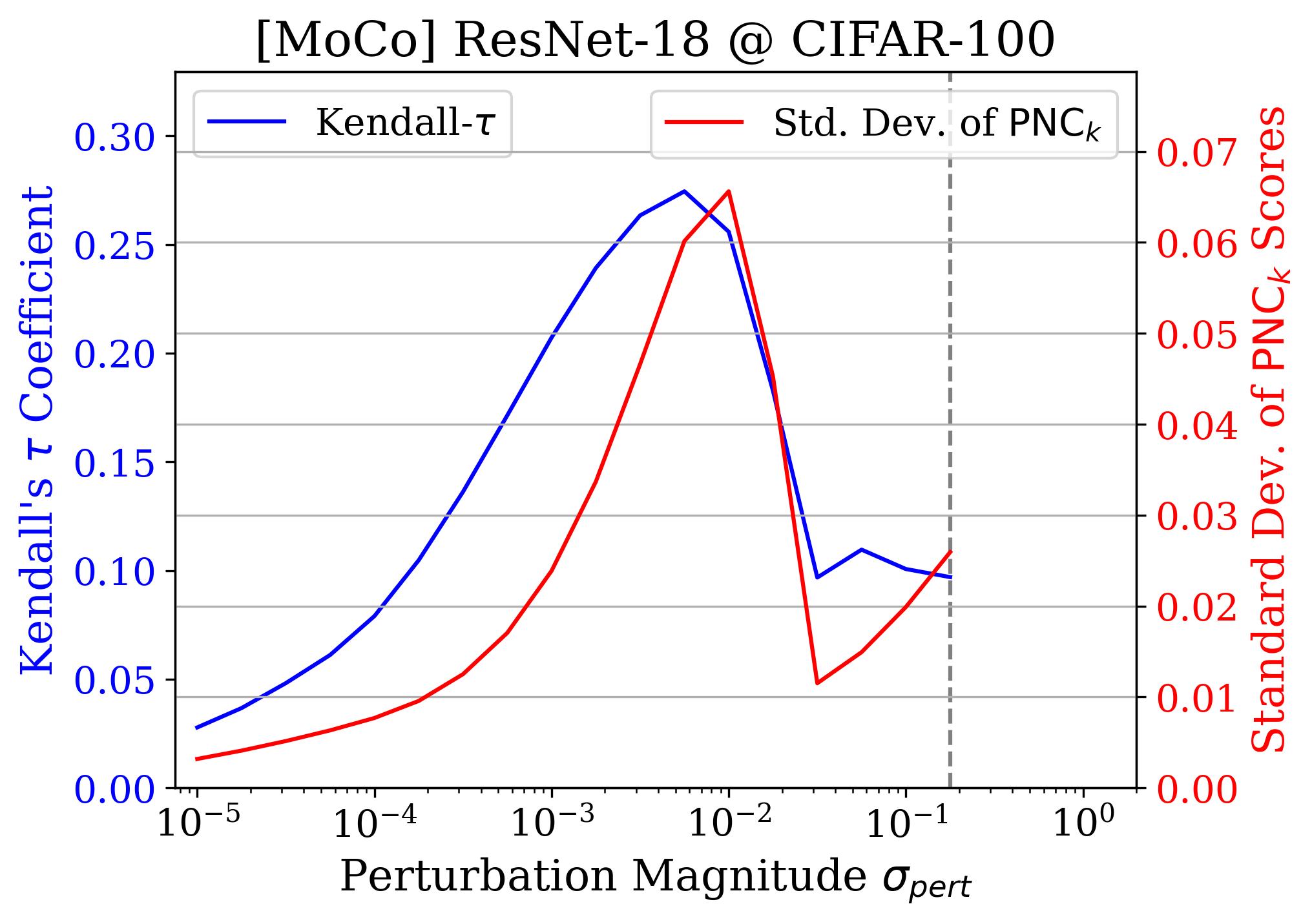}
    }
    \subfigure{   
        \includegraphics[width=0.4\columnwidth]{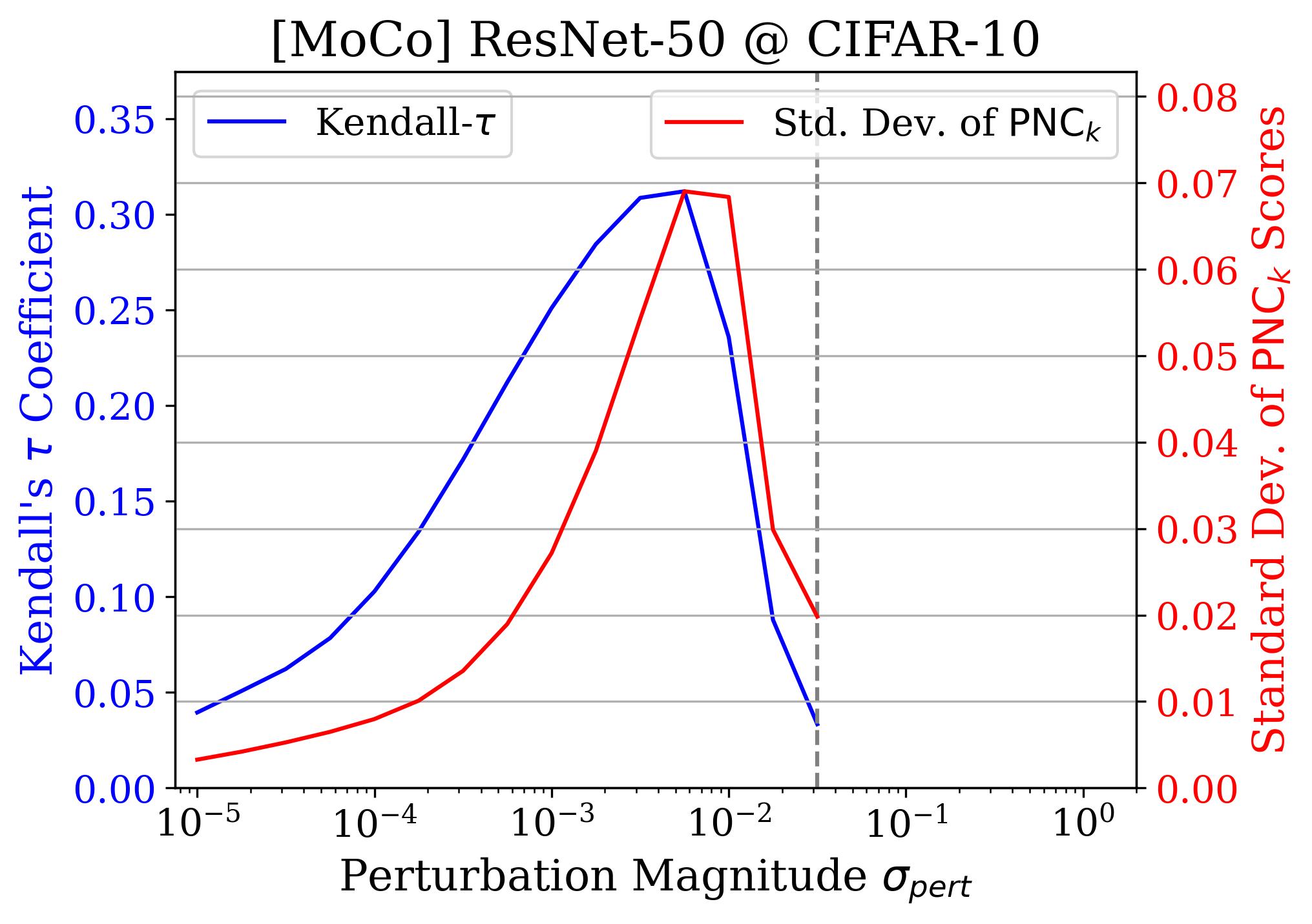}
    }
    \subfigure{
        \includegraphics[width=0.4\columnwidth]{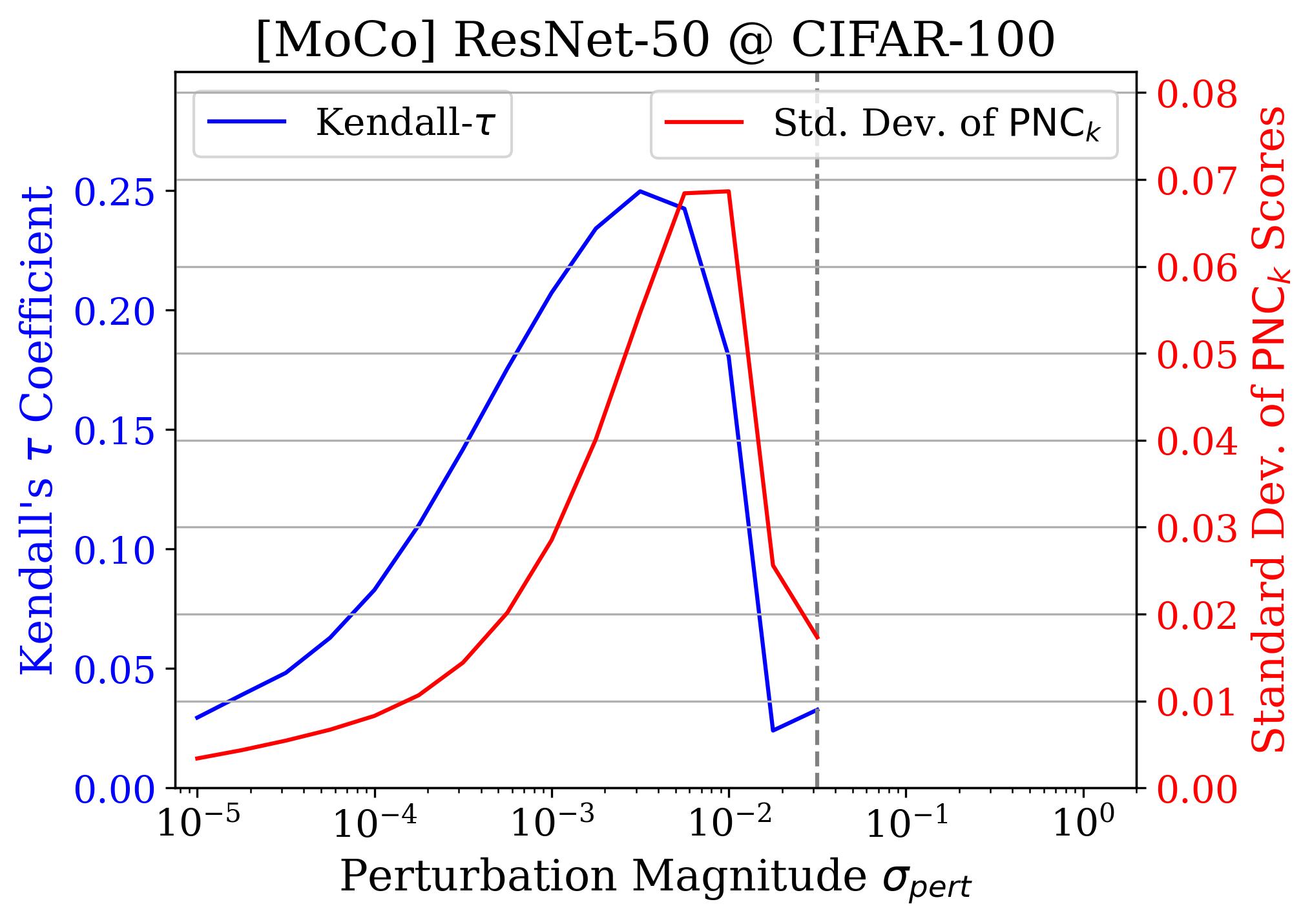}
    }
    \caption{Validation of $\mathsf{PNC}$-spread tuning for BYOL and MoCo on in-distribution binary classification tasks. Blue: \kendall between $\mathsf{PNC}_{100}$ scores and downstream Brier scores as a function of $\sigma_{\text{pert}}$. Red: standard deviation of $\mathsf{PNC}_{100}$ scores on $\bX_{\text{ref}}$. The peak correlation aligns with the peak spread, supporting the spread-based tuning criterion. Dashed gray line: largest considered $\sigma_{pert}$ that did not cause a floating-point overflow.}
    \label{fig:stddev_tuning_indist_binary_brier}
\end{figure}

\begin{table*}[h] 
\centering
\fontsize{9pt}{9pt}\selectfont
\caption{Quantifying the reliability of individual embedding functions without retraining for different choices of $\mathcal{P}$. The following are \kendall correlation between each method and the downstream performance of individual models on in-distribution classification tasks.}
\renewcommand{\arraystretch}{1.1}
\begin{tabular}{cc|cccc|cccc}
\toprule
\multirow{3}{*}{\makecell[c]{Downstream \\ Task Type}} & \multirow{3}{*}{Method} & \multicolumn{4}{c|}{ResNet-18} & \multicolumn{4}{c}{ResNet-50}  \\
\cline{3-10} 
 & & \multicolumn{2}{c}{\texttt{CIFAR-10}}  & \multicolumn{2}{c|}{\texttt{CIFAR-100}} & \multicolumn{2}{c}{\texttt{CIFAR-10}}  & \multicolumn{2}{c}{\texttt{CIFAR-100}} \\
 & & Entropy  & Brier & Entropy & Brier & Entropy  & Brier & Entropy  & Brier \\
\midrule
\multirow{6}{*}{Binary}& $\mathsf{PNC}_{100}^{\mathbin{\|},o}$ & 0.3622 & 0.3062 & 0.2490 & 0.1968 & 0.3636 & 0.3107 & 0.2546 & 0.2016\\
& $\mathsf{PNC}_{100}^{\mathbin{\|},s}$ & 0.3574 & 0.3025 & 0.2479 & 0.1961 & 0.3618 & 0.3087 & 0.2546 & 0.2016\\
& $\mathsf{PNC}_{100}^{\perp,o}$ & 0.3267 & 0.2787 & 0.2349 & 0.1901 & 0.2873 & 0.2446 & 0.2458 & 0.2018\\
& $\mathsf{PNC}_{100}^{\perp,s}$ & 0.3193 & 0.2719 & 0.2119 & 0.1677 & 0.2637 & 0.2246 & 0.2113 & 0.1759\\
& $\mathsf{PNC}_{100}^{g,o}$ & 0.3313 & 0.2786 & 0.2173 & 0.1825 & 0.2874 & 0.2439 & 0.2020 & 0.1667\\
& $\mathsf{PNC}_{100}^{g,s}$ & 0.2791 & 0.2373 & 0.1969 & 0.1647 & 0.2530 & 0.2138 & 0.1820 & 0.1444\\
\midrule
\multirow{6}{*}{Multi-Class}& $\mathsf{PNC}_{100}^{\mathbin{\|},o}$ & 0.3648 & 0.2987 & 0.2479 & 0.1603 & 0.3618 & 0.3031 & 0.2525 & 0.1648\\
& $\mathsf{PNC}_{100}^{\mathbin{\|},s}$ & 0.3603 & 0.2951 & 0.2465 & 0.1600 & 0.3598 & 0.3009 & 0.2525 & 0.1648\\
& $\mathsf{PNC}_{100}^{\perp,o}$ & 0.3310 & 0.2734 & 0.2296 & 0.1574 & 0.2868 & 0.2413 & 0.2373 & 0.1656\\
& $\mathsf{PNC}_{100}^{\perp,s}$ & 0.3234 & 0.2661 & 0.2128 & 0.1390 & 0.2631 & 0.2204 & 0.1984 & 0.1463\\
& $\mathsf{PNC}_{100}^{g,o}$ & 0.3303 & 0.2719 & 0.2033 & 0.1562 & 0.2881 & 0.2380 & 0.1971 & 0.1404\\
& $\mathsf{PNC}_{100}^{g,s}$ & 0.2840 & 0.2355 & 0.1847 & 0.1406 & 0.2608 & 0.2141 & 0.1810 & 0.1191\\
\bottomrule
\end{tabular}
\label{tab::pnc_indist_others}
\end{table*}

\begin{table*}[h] 
\centering
\fontsize{9pt}{9pt}\selectfont
\caption{Quantifying the reliability of individual embedding functions without retraining for different choices of $\mathcal{P}$. The following are \kendall correlation between each method and the downstream performance of individual models on transfer learning classification tasks. Embedding
functions are pre-trained on \texttt{TinyImagenet} and fine-tuned on \texttt{CIFAR-10}, \texttt{CIFAR-100}, and
\texttt{STL-10}.}
\subtable[ResNet-18]{
\renewcommand{\arraystretch}{1.1}
\begin{tabular}{cc|cccccc}
\toprule
\multirow{2}{*}{\makecell[c]{Downstream \\ Task Type}} & \multirow{2}{*}{Method} & \multicolumn{6}{c}{$\texttt{TinyImagenet} \rightarrow$}  \\
\cline{3-8} 
 & & \multicolumn{2}{c}{\texttt{CIFAR-10}}  & \multicolumn{2}{c}{\texttt{CIFAR-100}} & \multicolumn{2}{c}{\texttt{STL-10}} \\
 & & Entropy  & Brier & Entropy & Brier & Entropy & Brier \\
\midrule
\multirow{6}{*}{Binary}& $\mathsf{PNC}_{100}^{\mathbin{\|},o}$ & 0.2022 & 0.1595 & 0.1931 & 0.1412 & 0.2757 & 0.2003\\
& $\mathsf{PNC}_{100}^{\mathbin{\|},s}$ & 0.1735 & 0.1363 & 0.1680 & 0.1202 & 0.2329 & 0.1671\\
& $\mathsf{PNC}_{100}^{\perp,o}$ & 0.1753 & 0.1337 & 0.1957 & 0.1483 & 0.2354 & 0.1649\\
& $\mathsf{PNC}_{100}^{\perp,s}$ & 0.1729 & 0.1317 & 0.1957 & 0.1483 & 0.2352 & 0.1640\\
& $\mathsf{PNC}_{100}^{g,o}$ & 0.1750 & 0.1309 & 0.2036 & 0.1501 & 0.2373 & 0.1647\\
& $\mathsf{PNC}_{100}^{g,s}$ & 0.1741 & 0.1302 & 0.1909 & 0.1399 & 0.2370 & 0.1638\\
\midrule
\multirow{6}{*}{Multi-Class}& $\mathsf{PNC}_{100}^{\mathbin{\|},o}$ & 0.1996 & 0.1496 & 0.1771 & 0.0894 & 0.2543 & 0.1727\\
& $\mathsf{PNC}_{100}^{\mathbin{\|},s}$ & 0.1709 & 0.1279 & 0.1515 & 0.0751 & 0.2174 & 0.1437\\
& $\mathsf{PNC}_{100}^{\perp,o}$ & 0.1705 & 0.1253 & 0.1731 & 0.0974 & 0.2138 & 0.1390\\
& $\mathsf{PNC}_{100}^{\perp,s}$ & 0.1683 & 0.1231 & 0.1727 & 0.0958 & 0.2124 & 0.1378\\
& $\mathsf{PNC}_{100}^{g,o}$ & 0.1714 & 0.1237 & 0.1766 & 0.0970 & 0.2115 & 0.1396\\
& $\mathsf{PNC}_{100}^{g,s}$ & 0.1706 & 0.1232 & 0.1645 & 0.0873 & 0.2110 & 0.1376\\
\bottomrule
\end{tabular}
\label{tab:pnc_transfer_others_r18}
}
\subtable[ResNet-50]{
\renewcommand{\arraystretch}{1.1}
\begin{tabular}{cc|cccccc}
\toprule
\multirow{2}{*}{\makecell[c]{Downstream \\ Task Type}} & \multirow{2}{*}{Method} & \multicolumn{6}{c}{$\texttt{TinyImagenet} \rightarrow$}  \\
\cline{3-8} 
 & & \multicolumn{2}{c}{\texttt{CIFAR-10}}  & \multicolumn{2}{c}{\texttt{CIFAR-100}} & \multicolumn{2}{c}{\texttt{STL-10}} \\
 & & Entropy  & Brier & Entropy & Brier & Entropy & Brier \\
\midrule
\multirow{6}{*}{Binary}& $\mathsf{PNC}_{100}^{\mathbin{\|},o}$ & 0.1828 & 0.1402 & 0.1837 & 0.1347 & 0.2547 & 0.1713\\
& $\mathsf{PNC}_{100}^{\mathbin{\|},s}$ & 0.1748 & 0.1340 & 0.1806 & 0.1309 & 0.2386 & 0.1589\\
& $\mathsf{PNC}_{100}^{\perp,o}$ & 0.1530 & 0.1105 & 0.1797 & 0.1346 & 0.2310 & 0.1430\\
& $\mathsf{PNC}_{100}^{\perp,s}$ & 0.1292 & 0.0896 & 0.1705 & 0.1263 & 0.2290 & 0.1362\\
& $\mathsf{PNC}_{100}^{g,o}$ & 0.1261 & 0.0898 & 0.1520 & 0.1141 & 0.1823 & 0.1163\\
& $\mathsf{PNC}_{100}^{g,s}$ & 0.1038 & 0.0708 & 0.1219 & 0.0775 & 0.1222 & 0.0604\\
\midrule
\multirow{6}{*}{Multi-Class}& $\mathsf{PNC}_{100}^{\mathbin{\|},o}$ & 0.1820 & 0.1315 & 0.1784 & 0.0855 & 0.2314 & 0.1469\\
& $\mathsf{PNC}_{100}^{\mathbin{\|},s}$ & 0.1752 & 0.1263 & 0.1716 & 0.0807 & 0.2176 & 0.1350\\
& $\mathsf{PNC}_{100}^{\perp,o}$ & 0.1512 & 0.1055 & 0.1639 & 0.0843 & 0.2034 & 0.1204\\
& $\mathsf{PNC}_{100}^{\perp,s}$ & 0.1317 & 0.0859 & 0.1491 & 0.0803 & 0.1960 & 0.1129\\
& $\mathsf{PNC}_{100}^{g,o}$ & 0.1287 & 0.0854 & 0.1408 & 0.0734 & 0.1603 & 0.0997\\
& $\mathsf{PNC}_{100}^{g,s}$ & 0.1069 & 0.0705 & 0.1062 & 0.0394 & 0.0951 & 0.0457\\
\bottomrule
\end{tabular}
\label{tab:pnc_transfer_others_r50}
}
\label{tab::pnc_transfer_others}
\end{table*}

\end{document}